\documentclass{article}

\pdfoutput=1

\usepackage[english]{babel}
\usepackage[letterpaper,top=2cm,bottom=2cm,left=3cm,right=3cm,marginparwidth=1.75cm]{geometry}
\usepackage{amsmath,mathtools,amsfonts,amsthm} 
\usepackage{amssymb}
\usepackage{graphicx}
\usepackage[colorlinks=true]{hyperref}
\usepackage{caption,subcaption}
\usepackage{comment}
\usepackage{tikz}
\usetikzlibrary{automata,positioning,arrows}
\tikzset{node distance=3.5cm, 
every state/.style={ 
semithick,
fill=gray!10},
initial text={}, 
double distance=2pt, 
every edge/.style={ 
                draw,
->,>=stealth, 
auto,
semithick}}

\usepackage{algorithm}
\usepackage{algpseudocode}

\newtheorem{problem}{Problem}
\newtheorem{corollary}{Corollary}
\newtheorem{lemma}{Lemma}

\newtheorem{definition}{Definition}
\newtheorem{example}{Example}
\newtheorem{remark}{Remark}
\newtheorem{theorem}{Theorem}

\title{Joint Learning of Reward Machines and Policies in Environments with Partially Known Semantics}

\author{Christos K. Verginis, Cevahir Koprulu, Sandeep Chinchali, Ufuk Topcu
}

\date{}

\begin{document}
	\maketitle

		\begin{abstract}
			We study the problem of reinforcement learning 
			for a task encoded by a reward machine.  
			The task 
			is defined over a set of properties in the environment, called atomic propositions, and represented by Boolean variables. 
			One unrealistic assumption commonly used in the literature is that the truth values of these propositions are accurately known. In real situations, however, these truth values are uncertain since they come from sensors that suffer from imperfections.  
			At the same time, reward machines can be difficult to model explicitly, especially when they encode complicated tasks.  
			We develop a reinforcement-learning algorithm that infers a reward machine that encodes the 		 	underlying task while learning how to execute it, despite the uncertainties
			of the propositions' truth values. In order to address such uncertainties, the algorithm maintains a probabilistic estimate about the truth value of the atomic propositions; it updates this estimate according to new sensory measurements that arrive from \textcolor{blue}{exploration} of the environment. 
			Additionally, the algorithm maintains a hypothesis reward machine, which acts as an estimate of the reward machine that encodes the task to be learned. As the agent explores the environment, the algorithm updates the hypothesis reward machine according to the obtained rewards and the estimate of the atomic propositions' truth value. Finally, the algorithm uses a Q-learning procedure for the states of the hypothesis reward machine to determine \textcolor{blue}{an optimal policy that accomplishes the task.} 			
			We prove that the algorithm successfully infers the reward machine and asymptotically learns a policy that accomplishes the respective task.			
		\end{abstract}
		
		
	
	
	\section{Introduction}
	Reinforcement learning (RL) studies the problem of learning an optimal behavior for agents with unknown dynamics in potentially unknown environments. 
	\textcolor{blue}{When the task assigned to the agent admits a certain form of compositionality, such as decomposition to temporally dependent sub-tasks, a variety of methods incorporate high-level knowledge that can help the agent explore the environment more efficiently \cite{taylor2007cross}. }
	This high-level knowledge is usually expressed through abstractions of the environment and sub-task sequences 
\cite{singh1992reinforcement,kulkarni2016hierarchical,abel2018state} or in linear temporal logic \cite{li2017reinforcement,aksaray2016q}.
\textcolor{blue}{Intuitively, such methodologies inform the agent about the compositionality of the task, including its decomposition to sub-tasks as well as the sub-task execution sequences that lead to its completion. 
Consequently, 
the agent learns sub-task policies to execute such sequences and 
accomplish the overall task.}
%
	Recently, the authors in \cite{icarte2018using} proposed the concept of \textit{reward machines} in order to provide high-level information to the agent in the form of rewards. 	
	\textcolor{blue}{Reward machines are finite-state structures that encode a possibly non-Markovian reward function. They decompose a task into several temporally related subtasks, such as ``get coffee and bring it to the office without encountering obstacles".} Reward machines allow the composition of such subtasks in flexible ways, including concatenations, loops, and conditional rules, providing the agent access to high-level temporal 
relationships among the subtasks. These subtasks are defined over a set of properties in the environment, called atomic propositions, and represented by Boolean variables, such as ``office'' and ``obstacle'' in the aforementioned example. Intuitively, as the agent navigates in the environment, it transits between the states within a reward machine; such transitions are triggered by the truth values of the atomic propositions. 
 	After every such transition, the reward machine outputs the reward the agent should obtain according to the encoded reward function. Furthermore, \cite{icarte2018using} develops a q-learning method that learns the optimal policy associated with the reward function that is encoded by the reward machine. 

In many practical situations, the reward machines encode complex temporal relationships among the subtasks and are hence too difficulty to construct. Therefore, the need emerges for the learning agent to learn how to accomplish the underlying task without having a priori access to the respective reward machine.   The work \cite{xu2020joint} develops an algorithm that infers both the reward machine and an optimal RL policy from the trajectories of the system. 

	\textcolor{blue}{The procedure of inferring and using a reward machine is tightly connected to the values of the atomic propositions; such values trigger the reward machine's transitions, which output the reward the agent should obtain. 
	In related literature, there are two main methods to encode the learning agent's access to the truth values of the atomic propositions. First, the agent has access to a so-called \textit{labelling function}, a map that provides which atomic propositions are true in which areas of the environment, e.g., which room of an indoor environment has an office. Second, the agent accurately detects the truth values of the atomic propositions from sensors as it navigates in the environment.}
	
\textcolor{blue}{The aforementioned methods rely on assumptions that are unrealistic in practical scenarios. 	     
Firstly, many situations involve agents operating in a priori unknown environments, and hence accurate knowledge of a labelling function cannot be obtained a priori. Secondly, autonomous agents are endowed with sensors that naturally suffer from imperfections, such as misclassifications or missed object detections. Consequently, the truth values of the atomic propositions are uncertain. When the reward machine that encodes the task is a priori unknown,  uncertainties in such values complicate significantly the RL problem since they interfere with the estimation of the reward machine and, consequently, the learning of a suitable policy. Therefore, it is necessary to develop RL algorithms that are robust to such uncertainties.}
	 	
	\textcolor{blue}{In this paper, we investigate an RL problem for a task encoded by a reward machine over a set of atomic propositions.
	The reward machine and the labelling function associated with the atomic propositions are a priori unknown. 
	The agent obtains the truth values of the atomic propositions via measurements from sensors that suffer from imperfections, leading to uncertainties in the atomic propositions' truth values. 
	Our main contribution lies in the development of an RL algorithm that learns  a policy that achieves the given task despite the a priori unknown reward machine and the atomic proposition uncertainties. 
In order to account for such uncertainties, the algorithm 
 holds a probabilistic belief over the truth values of the atomic propositions and it updates this belief according to new sensory measurements that arrive from the exploration of the environment. 
 Furthermore, the algorithm maintains a so-called hypothesis reward machine to estimate the unknown reward machine that encodes the task. It uses the rewards the agent obtains as it explores the environment and the aforementioned belief on the atomic propositions to update this hypothesis reward machine.  
 Finally, the algorithm deploys a q-learning procedure that is based on the hypothesis reward machine and the probabilistic belief to learn \textcolor{blue}{an optimal policy} that accomplishes the underlying task. An illustrative diagram of the proposed algorithm is depicted in Fig. \ref{fig:algorithm}. }

We establish theoretical guarantees on the algorithm's convergence to a policy that accomplishes the underlying task. Our guarantees rely on the following sufficient conditions. First, the belief updates lead to better estimation of the truth values of the atomic propositions. That is, there exists a finite time instant after which the probability that the estimated propositions' values match the actual ones is more than $0.5$. 
Second, the learning agent explores the environment sufficiently well in each episode of the proposed RL algorithm. More specifically, the length of each episode is larger than a pre-defined constant associated with the size of the operating environment and of the reward machine that encodes the task to be accomplished.
Based on the aforementioned assumptions, we prove that the proposed algorithm 1) asymptotically infers a reward machine that is equivalent to the one encoding the underlying task, and 2) asymptotically learns a policy that accomplishes this task. 
 
\textcolor{blue}{ Finally, we carry out experiments comparing the proposed algorithm with a baseline version that does not update the beliefs of the propositions' truth values as well as with nominal q-learning and double deep q-network (DDQN) algorithms. The experiments show that the proposed algorithm outperforms the baseline version and the nominal algorithm in terms of convergence.}
 The experiments further show the robustness of the proposed algorithm to inaccurate sensor measurements.
 
	
	
	\begin{figure}
		\centering
		\includegraphics[width=\textwidth, trim={0cm 0cm 0cm 0cm}]{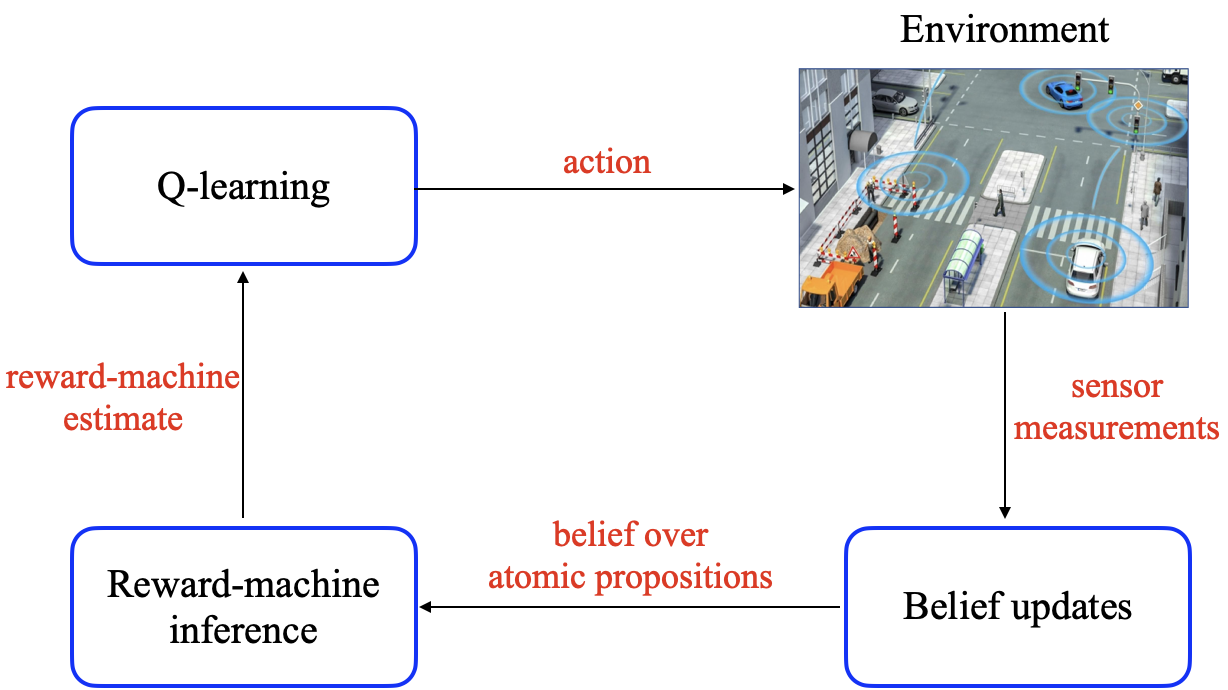}
		\caption{Overview of the proposed reinforcement-learning algorithm, consisting of perception updates, inference of the reward machine, and q-learning. }
		\label{fig:algorithm}
	\end{figure}

	\section{Related Work}
	Previous RL works on providing higher-level structure and knowledge about the reward function focus mostly on  abstractions and hierarchical RL; in such cases, a meta-controller decides which subtasks to perform, and a controller decides which actions to take within a subtask (e.g., \cite{konidaris2019necessity,kulkarni2016hierarchical,andreas2017modular}). Other works use temporal logic languages to express a specification, and then generate the corresponding reward functions, \textcolor{blue}{typically using deterministic automata (DA)} \cite{wen2015correct,li2017reinforcement,
sadigh2014learning,aksaray2016q,hasanbeig2018logically,
toro2018teaching,wang2021reinforcement,bozkurt2020control,
hahn2019omega,carr2020verifiable}. Inspired by \cite{icarte2018using}, we use in this work the concept of reward machines, which encode more compactly and expressively high-level information on the reward function of the agent. \textcolor{blue}{In particular and unlike DA, reward machines offer the ability of assigning separate state-action reward functions to each of their states; the learning agent obtains rewards depending on the transitions of the reward machine and hence each transition can be viewed as a separate sub-task. Therefore, the reward machine gives  a certain structure to the agent in the sense that it informs it 
about the decomposition of the task into sub-tasks and the sequence these sub-tasks must be executed. Consequently,  the agent learns how and in what sequence to execute the different sub-tasks towards the completion of the overall task.}	
	Additionally, unlike the works in the related literature, we consider that the reward machine that encodes the task at hand is a priori \textit{unknown}.

	There exists a large variety of works dealing with perception uncertainty \cite{ghasemi2020task,da2019active,agha2014firm,guo2013revising,lahijanian2016iterative,livingston2012backtracking,montana2017sampling,ayala2013temporal,camacho2017non,ramasubramanian2019secure}.
	However, most of these works consider the \textit{planning} problem, which consists of computing an optimal policy by explicitly employing the underlying agent model. Furthermore, the considered uncertainty usually arises from probabilistic dynamics, which is resolved either by active-perception procedures, belief-state propagation, or analysis with partially observable Markov decision processes (POMDPs).  Besides, the works that consider high-level reward structure (like temporal logic \cite{ghasemi2020task,da2019active,guo2013revising,livingston2012backtracking}), assume full availability of the respective reward models, unlike the problem setting considered in this  paper. 
	
	\textcolor{blue}{Several recent works consider tasks encoded by Non-Markovian reward functions that are \textit{a priori unknown}, using both automata \cite{zhang2015learning,furelos2020induction,hasanbeig2021deepsynth,abate2022learning} as well as reward machines \cite{toro2019learning,rens2020online,xu2020joint}. In such works, the agents infers the reward function from the executed trajectories in the environment. Similarly, \cite{dohmen2022inferring} develops an algorithm that infers probabilistic reward machines. 
	The aforementioned works, however, do not consider uncertainty in the truth values of the atomic propositions;} \cite{zhang2015learning} and \cite{toro2019learning} assume partial environment state observability by employing POMDP models, while considering the  truth values of atomic propositions in the environment known. In contrast, in this work we consider uncertainty in the semantic representation of the environment, i.e., the truth value of the environment properties that define the agent's task is a priori unknown.

	\section{Problem Formulation}
	
	This section formulates the considered problem.
	 We first describe the model of the learning agent and the environment, the reward machine that encodes the task to be accomplished, and the observation model of the sensor the agent uses to detect the values of atomic propositions.  

	\subsection{Agent Model}
	
	We model the interaction between the agent and the environment by a Markov decision process (MDP) \cite{bellman1957markovian,puterman2014markov}, formally defined below.

	\begin{definition} \label{def:MDP}
		A \textcolor{blue}{labelled MDP (l-MDP)} is a tuple $\mathcal{M} = (S, s_I,A, \mathcal{T}$, $R$, $\mathcal{AP}, {L}_G, \hat{L}, \gamma)$ consisting of the following:
	 $S$ is a finite state space, resembling the operating environment;  $s_I\in S$ is an initial state\footnote{\textcolor{blue}{$s_I$ can also be a distribution of initial states.}}; $A$ is a finite set of actions;  $\mathcal{T}:S \times A \times S \to [0,1]$ is a probabilistic transition function, modelling the agent dynamics in the environment; $R: S^\ast \times A \times S \to \mathbb{R}$ is a reward function, specifying payoffs to the agent;  $\gamma \in[0,1)$  is a discount factor; $\mathcal{AP}$ is a finite set of Boolean atomic propositions, each one taking values in $\{\mathsf{True},\mathsf{False}\}$; ${L}_G: S  \to 2^{\mathcal{AP}}$ is the ground-truth labelling function, assigning to each state $s \in S$ the atomic propositions that are true in that state;  $\hat{L}: S  \to 2^{\mathcal{AP}}$ is an estimate of the ground-truth labelling function ${L}_G$.	 
		We define the size of $\mathcal{M}$, denoted as $|\mathcal{M}|$, to be $|S|$ (i.e., the cardinality of the set $S$).
	\end{definition}

	The atomic propositions $\mathcal{AP}$ are Boolean-valued properties of the environment; we denote by $s\models p$ if an atomic proposition $p\in\mathcal{AP}$ is true at state $s\in S$. 
	The ground-truth labelling function $L_G$ provides which atomic propositions are true in the states of the environment,
	i.e., $L_G(s) = P \subseteq \mathcal{AP}$ is equivalent to $s \models p$, for all $p\in P$. \textcolor{blue}{We consider the setting where $L_G$ is \textit{unknown} to the agent}, which detects the truth values of  $ \mathcal{AP}$ through sensor units such as cameras or range sensors. \textcolor{blue}{Since such sensors suffer from imperfections (e.g., noise), the agent maintains} a time-varying probabilistic belief $\hat{\mathcal{L}}_i: S \times 2^{\mathcal{AP}} \to [0,1]$ about the truth
	values of $ \mathcal{AP}$, where $i$ denotes a time index. More specifically, for a state $s\in S$ and a subset of atomic propositions $P\subseteq \mathcal{AP}$, $\hat{\mathcal{L}}_i(s,P)$ represents the probability of the event that $P$ is true at $s$, i.e., $\hat{\mathcal{L}}_i(s,P) = \mathsf{Pr}( \bigcap_{p\in P} s \models p)$. Notice that at each time index $i$ and for every state $s \in S$, it holds that $\sum_{P \subset \mathcal{AP} } \hat{\mathcal{L}}_i(s,P) = 1$. 
	The prior belief of the agent might be an uninformative prior distribution. 
	We assume that the truth values of the propositions are mutually
	independent in each state, i.e., 
	\begin{align*}
		Pr\bigg(  \bigcap_{p\in P} s \models P  \bigg) = \prod_{p\in P} Pr( s \models p), \ \ \forall s \in S, P \subset \mathcal{AP}
	\end{align*}
	and also between every pair,
	\begin{align*}
		&Pr\big(  s \models P \land s' \models P'\big) = Pr( s \models P ) Pr( s' \models P' ),\forall s, s' \in S, P, P' \subset \mathcal{AP}  
	\end{align*}	
	
	The agent's belief $\hat{\mathcal{L}}_i$, for some time index $i \geq 0$, determines the agent's inference of the environment configuration, and hence the estimate $\hat{L}$ of the labelling function. More specifically, we define $\hat{L}: ( S \times 2^{\mathcal{AP}} \to [0,1]) \times S \to 2^{\mathcal{AP}},$
	with
	\begin{align} \label{eq:L_hat}
	\hat{L}(\hat{\mathcal{L}}_i,s) = \{p\in\mathcal{AP}: \hat{\mathcal{L}}_i(s,p) \geq 0.5\},
	\end{align}
	i.e, the most probable outcomes. The following example illustrates the relation between $L_G$ and $\hat{L}$.
	
	\begin{example} \label{ex:office}
		Let the 10-state office workspace \textcolor{blue}{be} shown in Fig. \ref{fig:example grid workspace}. The set of atomic propositions is $\mathcal{AP} = \{\textup{c,o,X}\}$, corresponding to ``coffee", ``office", and ``obstacle", respectively. The ground-truth labelling function, shown at the top of Fig. \ref{fig:example grid workspace},  is defined as $L_G(s_{12}) = \{\textup{c}\}$, $L_G(s_{22}) = L_G(s_{23}) = \{\textup{X} \}$, $L_G(s_{24}) = \{\textup{o}\}$, and $\emptyset$ in the rest of the states.
		Imperfect sensory measurements might lead to a distribution $\hat{\mathcal{L}}$ and an estimate $\hat{L}(\hat{\mathcal{L}},\cdot)$ as shown at the bottom of Fig. \ref{fig:example grid workspace}, i.e., $\hat{L}(s_{12}) = \{\textup{c,X} \}$, $\hat{L}(s_{22}) = \{\textup{c} \}$, $\hat{L}(s_{14}) = \hat{L}(s_{24}) = \{\textup{o}\}$,
		$\hat{L}(s_{23}) = \{\textup{o}, X\}$, and $\emptyset$ in the rest of the states.
		\begin{figure}
			\centering
			\includegraphics[width=0.65\textwidth]{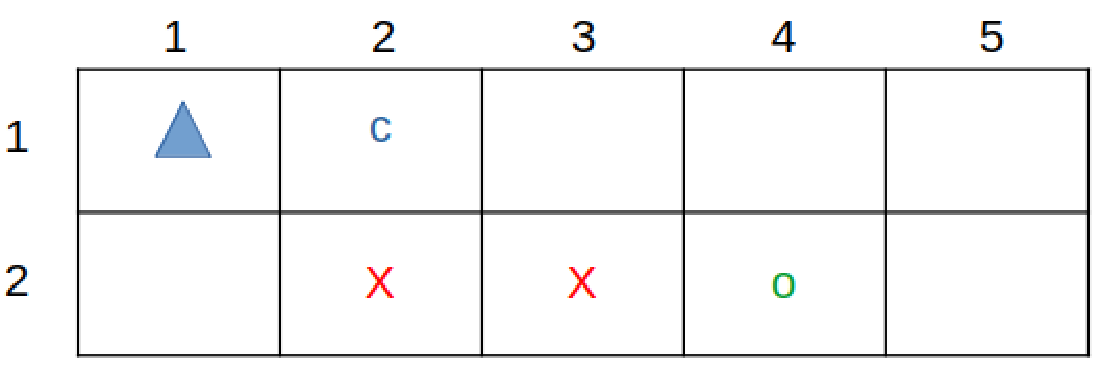} 
			
			\vspace{15mm}
			
			\includegraphics[width=0.65\textwidth]{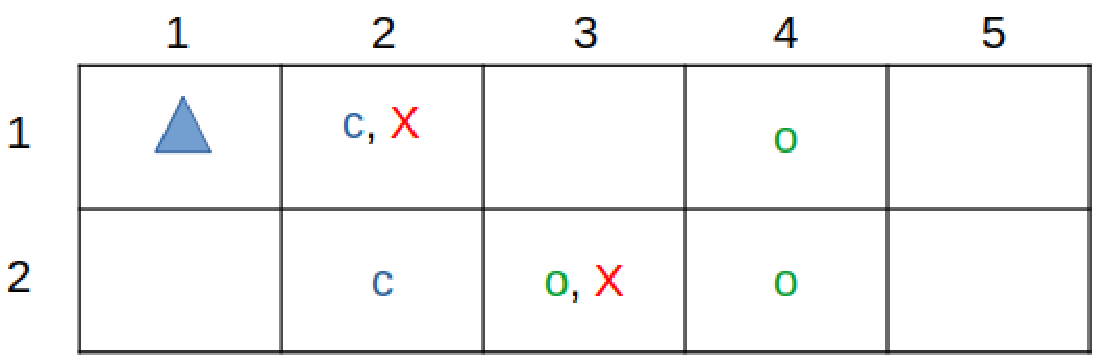}
			\caption{Top: A workspace with 10 states $s_{ij}$ illustrating the ground-truth labelling function ${L}_G:S\to 2^{\{\textup{c,o,X}\}}$, were $\{\textup{c,o,X}\}$ stand for \{``coffee", ``office", ``obstacle"\}, respectively; Bottom: Perception uncertainty leads to 
			the estimate $\hat{L}(\hat{\mathcal{L}},s):S\to 2^{\{\textup{c,o,X}\}}$, according to which the truth values of several properties are incorrect (e.g. $\textup{X}$ in $s_{12}$ or the absence of $\textup{X}$ in $s_{22}$). }
			\label{fig:example grid workspace}
		\end{figure}		
	\end{example}
	
A policy is a function that maps states in $S$ to a probability distribution over actions in $A$. At state $s\in S$, an agent, using policy $\pi$, picks an action $a$ with probability $\pi(s, a)$, and the new state $s'$ is chosen with probability $\mathcal{T}(s,a, s')$, defined in Def. \ref{def:MDP}. A policy $\pi$ and the initial state $s_I$ together determine a stochastic process; we write $S_0 A_0 S_1$...for the random trajectory of states and actions.
	
	A trajectory is a realization of the agent's stochastic process $S_0 A_0 S_1$: a sequence of states and actions $s_0 a_0 s_1 \dots s_k a_k s_{k+1}$ with $s_0 = s_I$. Its corresponding label sequence is $\ell_0 \ell_1 \dots \ell_k$ with $\ell_j = L_G(s_j)$, for all $j \leq k$.  Similarly, the reward sequence is $r_0r_1
	\dots r_k$, where $r_j = R(s_0 \dots s_j a_j s_{j+1})$, for all $j
	\leq k$.  
	A trajectory $s_0 a_0 s_1 \dots s_k a_k s_{k+1}$ achieves a reward $\sum_{j=0}^k \gamma^j R(s_0 \dots s_j a_js_{j+1})$, where $\gamma$ is the discount factor defined in Def. \ref{def:MDP}. 
	Note that the definition of the reward function assumes that the reward is a function of the whole trajectory; this allows the reward function to be non-Markovian  \cite{icarte2018using}.
	Given a trajectory $s_0 a_0 s_1 \dots s_{k+1}$ and a belief $\hat{\mathcal{L}}_i$, we naturally define then 
	the \textit{observed} label sequence by $\hat{\ell}_0 \hat{\ell}_1 \dots \hat{\ell}_k$, with $\hat{\ell}_j = 
	\hat{L}(\hat{\mathcal{L}}_i, s_j)$ for all $j\leq k$ and some $i \geq 0$. 
	
	\begin{example}[Continued] 
		Let the 10-state office workspace \textcolor{blue}{be} shown in Fig. \ref{fig:example grid workspace}. Assume that the agent receives award 1 if it brings coffee to the office without encountering states with obstacles, and zero otherwise. 
	\textcolor{blue}{Let us now consider} a state sequence $s_{11} s_{12} s_{13} s_{14} s_{24}$ of a trajectory.
		According to the ground-truth labelling function (see top of Fig. \ref{fig:example grid workspace}), such a trajectory would
		 would produce the label and reward sequence $(\emptyset, 0)(c,0)(\emptyset,0)(\emptyset,0)(o,1)$. 
		 However, 		 
		the \textit{observed} label and reward sequence, based on the estimate $\hat{L}$ (see bottom of Fig. \ref{fig:example grid workspace}), is $(\emptyset,0)$ $((\textup{c,X}),0)$ $(\emptyset,0)(\textup{o},0)(\textup{o},1)$.		
		
	\end{example}

	\subsection{Reward Machines}
	
	Encoding a (non-Markovian) reward in a type of finite state-machine is achieved by reward machines, introduced in \cite{icarte2018using} and formally defined below.
	
	\begin{definition}
		A reward machine is a tuple $\mathcal{A} = (V, v_I, 2^\mathcal{AP}, \mathcal{R}, \delta, \sigma)$ that consists of the following: $V$ is a finite set of states; $v_I \in V$ is an initial state;  $\mathcal{R}$ is an input alphabet; $\delta:V \times 2^{\mathcal{AP}} \to V$ is a deterministic transition function;  $\sigma: V \times 2^\mathcal{AP} \to \mathcal{R}$ is an output function. We define the size of $\mathcal{A}$, denoted by $|\mathcal{A}|$, to be $|V|$ (i.e., the cardinality of the set $V$).
	\end{definition}
	
	The run of a reward machine $\mathcal{A}$ on a sequence of labels $\ell_0 \dots \ell_k \in (2^{\mathcal{AP}})^\ast $ is a sequence $v_0(\ell_0, r_0)v_1(\ell_1,v_1)$ $\dots$ $v_k$ $(\ell_k,v_k)$ of states and label-reward pairs such that $v_0 = v_I$ and for all $j\in\{0,\dots,k\}$, we have $\delta(v_j,\ell_j) = v_{j+1}$ and $\sigma(v_j,\ell_j) = r_j$. We write $\mathcal{A}(\ell_0 \dots \ell_k ) = r_0 \dots r_k$ to connect the input label sequence to the sequence of rewards produced by the reward machine $\mathcal{A}$. We say that a reward machine $\mathcal{A}$ encodes the reward function $R$ of an \textcolor{blue}{l-MDP} \textit{on the ground truth} if, for every trajectory $s_0 a_0 \dots s_k a_k s_{k+1}$ and the corresponding label sequence $\ell_0 \dots \ell_k$, the sequence of rewards equals $\mathcal{A}(\ell_0 \dots \ell_k)$.  
	Moreover, given a fixed $i\geq 0$, we say that a reward machine $\mathcal{A}$ \textit{encodes the reward function $R$ on $\hat{\mathcal{L}}_i$} if for every trajectory $s_0 a_0 \dots s_k a_k s_{k+1}$ and the corresponding \textit{observed} label sequence $\hat{\ell}_0 \dots \hat{\ell}_k$ the sequence of rewards equals $\mathcal{A}(\hat{\ell}_0 \dots \hat{\ell}_k)$.  
	Note, however, that there might not exist a reward machine that encodes the reward function on $\hat{\mathcal{L}}_i$, as the next example shows. 
	
	\begin{example}[Continued] 
		Let the 10-state office workspace \textcolor{blue}{be} shown in Fig. \ref{fig:example grid workspace}. 
		Assume that the agent receives award 1 if it brings coffee to the office without encountering states with obstacles, and zero otherwise. Such a task is encoded by the reward machine shown in Fig. \ref{fig:automata coffee}. Let also the state sequence of a given trajectory \textcolor{blue}{be} $s_{11} s_{12} s_{13} s_{14} s_{24}$. 
		As stated before, the trajectory produces, via the ground-truth labelling function, the label and reward sequence $(\emptyset, 0)(c,0) (\emptyset,0)(\emptyset,0)$ $(o,1)$. Such a label sequence produces the run of the reward machine $$v_0(\emptyset,0) v_0 (c, 0) v_1 (\emptyset,0) v_1 (\emptyset,0) v_1 (o,1) v_3.$$
		
		The \textit{observed} label and reward sequence, based on the estimate $\hat{L}$, is $(\emptyset,0)$ $((\textup{c,X}),0)$ $(\emptyset,0)(\textup{o},0)(\textup{o},1)$, which produces the reward-machine run $$v_0(\emptyset,0)v_0 ((\textup{c,X}),0) v_2 (\emptyset,0) v_2 (\textup{o},0)v_2 (\textup{o},0) v_2,$$ which clearly does not comply with the observed reward sequence.  Hence, we conclude that the reward machine does \textit{not} encode the reward function on $\hat{\mathcal{L}}$. In fact, note that we cannot construct a reward machine that encodes the reward function on $\hat{\mathcal{L}}$.  
		The actual reward function of bringing coffee from $s_{12}$ to the office in $s_{24}$ cannot be expressed via the set $\mathcal{AP}$ with this specific $\hat{\mathcal{L}}$, because of the ambiguity in the states satisfying ``\textup{o}" and ``\textup{c}"; a single accepting trajectory does not correspond to a unique observed label sequence.

		\begin{figure}
			\centering
			\begin{tikzpicture}[scale=1, transform shape]
				\node[state, initial] (q0) {$v_0$};
				\node[state, below left of=q0] (q1) {$v_1$};
				\node[state, below right of=q0] (q2) {$v_2$};
				\node[state, accepting, below right of=q1] (q3) {$v_3$};
				\draw (q0) edge[loop above] node {\tt $(\neg \textup{c} \land \neg \textup{X},0)$} (q0);
				\draw (q0) edge node {\tt $(\textup{c} \land \neg \textup{X},0)$} (q1);
				\draw (q0) edge node {\tt $(\textup{X},0)$} (q2);
				\draw (q2) edge[loop below] node {\tt $(\textup{True}, 0)$} (q2);
				\draw (q1) edge[loop below] node {\tt $(\neg \textup{o} \land \neg \textup{X}, 0)$} (q1);
				\draw (q1) edge node {\tt $(\neg \textup{o} \land \textup{X},0)$} (q2);
				\draw (q1) edge node {\tt $(\textup{o},1)$} (q3);
				\draw (q3) edge[loop below] node {\tt $(\textup{True}, 0)$} (q3);
				
			\end{tikzpicture}
			\caption{Reward machine for the task of Fig. \ref{fig:example grid workspace}.}
			\label{fig:automata coffee}
		\end{figure}
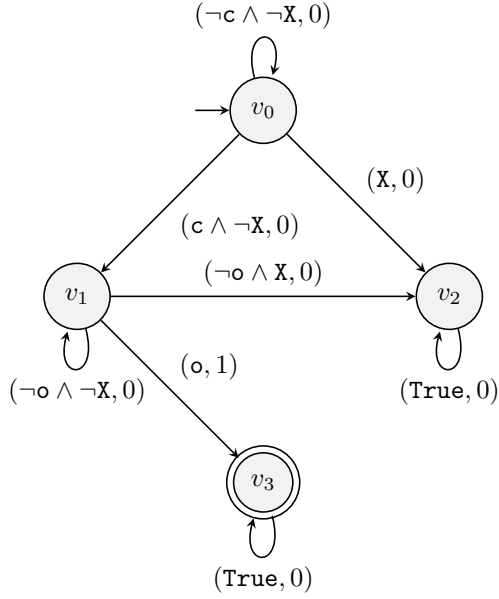
		
	\end{example}

		\subsection{Observation Model}

	The probabilistic belief is updated based on an observation model $\mathcal{O}$.  
	More specifically, at each time step, the agent’s perception module processes a set of sensory measurements regarding the atomic propositions $\mathcal{AP}$. In the Bayesian framework, the observation model is used
	for the update of the agent’s belief, as specified in the following definition. 	
	
	\begin{definition} \label{def:ob_model}
		Let $\mathcal{Z}(s_1,s_2,p) \in \{\mathsf{True},\mathsf{False}\}$ denote the perception output of the agent, when the agent is at state $s_1$, for the atomic proposition $p\in\mathcal{AP}$ at $s_2$. The joint observation model of the agent is $\mathcal{O}: S\times S \times \mathcal{AP} \times \{\mathsf{True},\mathsf{False}\} \to [0,1]$, where $\mathcal{O}(s_1,s_2,p,b)$ represents the probability that $\mathcal{Z}(s_1,s_2,p) = \mathsf{True}$ if the truth value of $p$ is identical to $b$.
	\end{definition}
	
\textcolor{blue}{More specifically, the joint observation describes the likelihoods $\mathcal{O}(s_1,s_2,p,\mathsf{True}) = Pr\bigg( \mathcal{Z}(s_1,s_2,p) = \mathsf{True} \ \big| \ s_2 \models p\bigg) $ and $\mathcal{O}(s_1,s_2,p,\mathsf{False}) = Pr\bigg( \mathcal{Z}(s_1,s_2,p) = \mathsf{True} \ \big| \ s_2 \not\models p \bigg) $.}
	An accurate observation model is the one for
	which $\mathcal{O}(s_1, s_2, p, b)$ $=$ $1$ for
	$b = \mathsf{True}$ and $\mathcal{O}(s_1, s_2, p, b) = 0$ for $b = \mathsf{False}$.	
	In the Bayesian framework, the observation model is used
	for the update of the agent’s belief. Nevertheless, in the
	absence of such observation model, one can perform the
	update in a frequentist way.

	\subsection{Problem Statement}
	
	We consider an \textcolor{blue}{l-MDP}-modeled autonomous agent, whose task is encoded by a reward machine that is unknown to the agent. The ground-truth labelling function $L_G$, providing the truth values of the atomic propositions, is also unknown to the agent. 
	The formal definition of the problem statement is as follows.
	
	\begin{problem}
		Let an \textcolor{blue}{l-MDP} agent model \textcolor{blue}{be} $\mathcal{M} = (S, s_I,A, \mathcal{T}, R, \mathcal{AP}, L_G, \hat{L}, \gamma)$ with unknown ground-truth labelling function $L_G$ and transition function $\mathcal{T}$. Let also an unknown reward machine  that encodes the reward function $R$ on the ground truth \textcolor{blue}{be} $\mathcal{A}$, representing a task to be accomplished. Develop an algorithm that learns a policy $\pi$ that maximizes $Z(s_I) =  \mathbb{E}_\pi[\sum_{i=0}^\infty \gamma^i R(S_0, A_0\dots S_{i+1})]$.
	\end{problem}
	
\textcolor{blue}{
\begin{remark}
We note that the problem of unknown ground-truth labelling function $L_G$ and uncertain sensor observations does not apply only to cases with reward machines, but also to more general RL frameworks that employ labelling-function mappings to identify properties of the environment (e.g., \cite{icarte2018using,jin2022creativity,hasanbeig2021deepsynth,dohmen2022inferring}). However, reward machines \cite{icarte2018using} constitute state-of-the-art tools when it comes to symbolic RL with task decomposition. Additionally, reward machines naturally depend on the atomic propositions of the environment and hence on the respective labelling function. Therefore, \textit{inference} of an unknown reward machine poses a unique challenge when such a labelling function is uncertain. For the aforementioned reasons, we focus on the problem of RL with unknown reward machines. Nevertheless, we illustrate how the proposed algorithm, given in the next section, applies to more general RL frameworks.
\end{remark}
}

\textcolor{blue}{
	\section{Joint Inference of Reward machines and Policies for RL} \label{sec:JIRP}}

	\textcolor{blue}{
	The main results of the paper consist of the development of a methodology that accommodates the uncertainty in the ground-truth labelling function $L_G$ and the environmental observations. One can incorporate such a methodology to standard RL algorithms
to account for the aforementioned uncertainty. For the sake of presentation, we illustrate the incorporation of our methodology to the JIRP algorithm \cite{xu2020joint},
which considers the joint inference of reward machines and policies with perfect knowledge of the environment semantics. Consequently, we provide in this section an overview of JIRP, before proceeding to the main results of the paper.  }
		\algdef{SE}[DOWHILE]{Do}{doWhile}{\algorithmicdo}[1]{\algorithmicwhile\ #1}%
	\begin{algorithm}[!t]
		\caption{JIRP Algorithm} 
	\label{algo:JIRP}
	\begin{algorithmic}[1]    
		\State $\mathcal{H} \leftarrow \mathsf{Initialize}(V)$ \State $Q = \{q^v | v\in V \} \leftarrow \mathsf{Initialize}()$  
		\State $X \leftarrow \emptyset$
		\For{episode $n=1,2,\dots$}
		\State $(\lambda,\rho,Q) \leftarrow \text{QRM}\_\text{episode}(\mathcal{H},Q,\mathcal{L})$
		\If{$\mathcal{H}(\lambda) \neq \rho$}
	\hspace{12mm}	\small \textcolor{blue}{// \texttt{The trace $(\lambda,\rho)$ is inconsistent with $\mathcal{H}$}} \normalsize 
		\State $\mathsf{Add}(X,(\lambda,\rho))$ 
		\hspace{12mm} \small \textcolor{blue}{// \texttt{Append the set $X$ with the trace $(\lambda,\rho)$}} \normalsize
		\State $\mathcal{H} \leftarrow \mathsf{Infer}(X)$
		\hspace{12mm} \small \textcolor{blue}{// \texttt{Infer a new reward machine from  $X$}} \normalsize
		\State $Q \leftarrow \mathsf{Initialize}()$
		\EndIf
		
		\EndFor
		
	\end{algorithmic}
\end{algorithm}

\algdef{SE}[DOWHILE]{Do}{doWhile}{\algorithmicdo}[1]{\algorithmicwhile\ #1}%
\begin{algorithm}[!t]
	\caption{QRM$\_$episode($\mathcal{H}$, $Q$, $\mathcal{L}$)} 
\label{algo:QRM}
\begin{algorithmic}[1]    
	\Require{A reward machine $\mathcal{H}=(V,v_I,2^\mathcal{P},\mathbb{R},\delta,\sigma)$, a set of $q$-functions $Q = \{q^v | v\in V\}$, a labelling function $\mathcal{L}$.}
\State $s \leftarrow \mathsf{InitialState}()$; $v \leftarrow v_I$, $\lambda \leftarrow []$; $\rho \leftarrow []$
\For{$0 \leq t < eplength$} 
\State $a \leftarrow \mathsf{GetEpsilonGreedyAction}(q^v,s)$
\State $s' \leftarrow \mathsf{ExecuteAction}(s,a)$
\hspace{3mm} \small \textcolor{blue}{// \texttt{Transition in \textcolor{blue}{l-MDP} $\mathcal{M}$ based on $\mathcal{T}(s',a,s)$}} \normalsize
\State $v' \leftarrow \delta(v, \mathcal{L}(s,a,s'))$
\hspace{11mm} \small \textcolor{blue}{// \texttt{Transition in $\mathcal{H}$}} \normalsize
\State $r \leftarrow \mathsf{ObserveReward}()$
\State $q^v(s,a) \leftarrow \mathsf{Update}(r,s')$
\hspace{6mm} \small \textcolor{blue}{// \texttt{Update of $q$-function for state $v$ of $\mathcal{H}$}} \normalsize
\For{$\hat{v} \in V \backslash{v}$}
\State $\hat{v}' \leftarrow \delta(\hat{v},\mathcal{L}(s,a,s'))$  
\hspace{7mm} \small \textcolor{blue}{// \texttt{Transition in $\mathcal{H}$}} \normalsize
\State $\hat{r} \leftarrow \sigma(\hat{v},\mathcal{L}(s,a,s'))$
\hspace{7mm} \small \textcolor{blue}{// \texttt{Reward calculation from $\mathcal{H}$}} \normalsize
\State $q^{\hat{v}}(s,a) \leftarrow \mathsf{Update}(\hat{r},s')$
\hspace{0mm} \small \textcolor{blue}{// \texttt{Update of $q$-function for state $\hat{v}$ of $\mathcal{H}$}} \normalsize
\EndFor
\State $\mathsf{Append}(\lambda, \mathcal{L}(s,a,s'))$; $\mathsf{Append}(\rho, r)$
\State $s \leftarrow s'$; $v\leftarrow v'$
\EndFor
\State \textbf{return} $(\lambda,\rho,Q)$
\end{algorithmic}
\end{algorithm}	
		\textcolor{blue}{
The JIRP algorithm \cite{xu2020joint} (see Algorithm \ref{algo:JIRP}) aims at learning the optimal policy for maximizing $Z(s_I)$ by maintaining a hypothesis reward machine (RM) $\mathcal{H}$ (initialized with a set of states $V$ in line 1). Its main component is the QRM$\_$episode() (shown in Algorithm \ref{algo:QRM}), originally proposed in \cite{icarte2018using} for accurately known reward machines. \textcolor{blue}{The QRM episode outputs a set $(\lambda,\rho)$, called \textit{trace}, and a set $Q$ of q-functions $q^v$  (initialized in line 2), one for each state $v\in V$ of the hypothesis reward machine; recall that a reward machine decomposes the overall task into sub-tasks, one for each one of its states}. 
The current state $v$ of the hypothesis reward machine guides the exploration by determining which $q$-function is used to choose
	the next action (line 3 of Alg. \ref{algo:QRM}). However, in each single exploration step, the algorithm updates the $q$-functions that correspond to all reward-machine
	states (lines 7 and 11 of Alg. \ref{algo:QRM}). 
\textcolor{blue}{Furthermore, the algorithm uses the sets $\lambda$ and $\rho$ to collect the labelling and reward sequences $\mathcal{L}(s,a,s')$ and $r$, respectively, from the agent trajectory (line 13 of Alg. \ref{algo:QRM})). }
	Note that the returned rewards in $\rho$ are \textit{observed} (line 6 of Alg. \ref{algo:QRM}), since the reward machine encoding the actual reward function is not known. Instead, JIRP operates with a hypothesis reward machine, which is updated using the traces of QRM$\_$episode(). In particular, the episodes of QRM are used to collect traces and update $q$-functions. As long
	as the traces are consistent with the current hypothesis reward machine, QRM explores more of the environment using
	the reward machine to guide the search. However, if a trace $(\lambda,\rho)$ is detected that is inconsistent with the hypothesis
	reward machine (i.e., $\mathcal{H}(\lambda)\neq \rho$, line 6 of Alg. \ref{algo:JIRP}), JIRP stores it in a set $X$ (line 7 of Alg. \ref{algo:JIRP}) - the trace $(\lambda,\rho)$ is called a
	counterexample and the set $X$ a sample. Once the sample is updated, the algorithm re-learns a new hypothesis reward
	machine (line 8 of Alg. \ref{algo:JIRP}) and proceeds. More information can be found in \cite{xu2020joint,icarte2018using}. 
We further detail the various functions used in Alg.  \ref{algo:QRM}: The function $\mathsf{GetEpsilonGreedyAction}(q^v,s)$ chooses an action based on an $\varepsilon_a$-greedy method, i.e., it chooses the action $\textup{argmax}_{a\in A} (q^v(s,a))$ with probability $1-\varepsilon_a$ and a random action $a \in A$ with probability $\varepsilon_a$ for a small positive constant $\varepsilon_a$. The function $\mathsf{Update}(r,s')$ updates $q^v(s,a)$ based on 
\begin{align*}
q^v(s,a) \leftarrow q^v(s,a) + \alpha_{lr} \bigg(
r + \gamma \max_{a' \in A} q^v(s',a') - q^v(s,a)  \bigg)
\end{align*}
for a constant $\alpha_{lr}$ representing the learning rate.
	}

\textcolor{blue}{
	\section{Joint Learning of Reward Machines and Policies  with Partially Known Semantics}}

	This section gives the main results of this paper, which is joint perception and inference of policies and reward machines.
In particular, we extend the JIRP algorithm, presented in Section \ref{sec:JIRP}, to take into account the uncertainty in the labelling function. More specifically, our algorithm tries to infer the reward machine $\mathcal{A}$ that encodes the reward function on the most probable outcomes, $\hat{L}(\hat{\mathcal{L}}_j, \cdot)$, i.e., $\mathcal{A}(\hat{\ell}_0,\dots,\hat{\ell}_k) = r_0 \dots r_k$ for some $j\geq 0$. 
At the same time, the agent's belief is updated at every time step based on the observation model of Definition \ref{def:ob_model}. Similarly to \cite{ghasemi2020task}, if the agent's knowledge about the environment has changed significantly, the agent then updates its belief $\hat{\mathcal{L}}$ and aims to infer a new reward machine. 

\textcolor{blue}{
Algorithms \ref{algo:JIRP ext} and \ref{algo:QRM_mod}  illustrate the aforementioned procedure. The algorithms are similar to \ref{algo:JIRP} and \ref{algo:QRM}, respectively, with three main differences: first, Alg. \ref{algo:JIRP ext} aims to approximate a reward machine $\mathcal{A}$ that encodes the reward functions on the ``current" belief  $\hat{L}(\hat{\mathcal{L}}_h,\cdot)$; second, Alg. \ref{algo:QRM_mod} updates this belief using the observation obtained by the agent in its state (line 14 of Alg. \ref{algo:QRM_mod}); third, Alg. \ref{algo:JIRP ext} updates the belief $\hat{L}(\hat{\mathcal{L}}_h,\cdot)$, and consequently the reward machine $\mathcal{H}$, if it is different enough from what is dictated by the obtained observations (lines 13-18 of Alg. \ref{algo:JIRP ext}). We present next the reasoning of Alg. \ref{algo:JIRP ext} and \ref{algo:QRM_mod}  in more detail. }

Algorithm \ref{algo:JIRP ext}  uses a hypothesis reward machine $\mathcal{H}$ (line 1 of Alg. \ref{algo:JIRP ext}) to guide the learning process. Similarly to \cite{xu2020joint}, it runs multiple QRM episodes to update the Q functions, and uses the collected traces to update the counterexamples in the set $X$ and $\mathcal{H}$ (lines 7-12 of Alg. \ref{algo:JIRP ext}). \textcolor{blue}{Nevertheless, in our case, $\mathcal{H}$  aims to 
approximate the reward machine $\mathcal{A}$ that encodes the reward function on the most probable outcomes, $\hat{L}(\hat{\mathcal{L}}_h, \cdot)$, since that is the available information to the agent; $\hat{\mathcal{L}}_h$ is the ``current" belief that the agent uses in its policy learning and reward-machine inference.}  
This is illustrated via a modified version of the QRM episode algorithm (QRM\_episode\_mod invoked in line 7 of Alg. \ref{algo:JIRP ext}), shown in Algorithm \ref{algo:QRM_mod}. In particular, QRM\_episode\_mod 
uses the probabilistic belief $\hat{\mathcal{L}}_h$ in order to navigate in the hypothesis reward machine $\mathcal{H}$ (lines 5, 9 of Alg. \ref{algo:QRM_mod}). 
Moreover,  $\hat{\mathcal{L}}_h$ is used to update the counterexample set $X$ with the inconsistent traces (line 15 of Alg. \ref{algo:QRM_mod} and lines 8-12 of Alg. \ref{algo:JIRP ext}). Hence, Alg. \ref{algo:JIRP ext} aims to infer the reward machine that encodes the reward function on the belief $\hat{\mathcal{L}}_h$. 
Of course, as illustrated in Example \ref{ex:office}, there might not exist a reward machine that encodes the reward function on $\hat{\mathcal{L}}_h$. In that case, the algorithm cannot infer the ``correct" reward machine.  
\textcolor{blue}{For this reason, the modified QRM-episode algorithm updates a recurring probabilistic belief $\hat{\mathcal{L}}_j$ (line 14) in a Bayesian manner and based on the observation model of Def. \ref{def:ob_model}; we provide the function $\mathsf{BayesUpdate}()$ in Section \ref{sec:information proc}. 
The  environment probabilistic belief $\hat{\mathcal{L}}_h$ used in the policy- and reward-machine-learning is updated with the current belief $\hat{\mathcal{L}}_j$ only if the two are significantly different (lines 13-18 of Algorithm \ref{algo:JIRP ext}), which is evaluated by the function $\mathsf{SignifChange}()$.
More specifically, if $\hat{\mathcal{L}}_j$ is significantly changed with respect to the estimate $\hat{\mathcal{L}}_h$, the algorithm updates $\hat{\mathcal{L}}_h$ accordingly, re-initializes the $Q$ functions, the hypothesis reward machine $\mathcal{H}_h$, and the counterexample set $X$.
We provide the function $\mathsf{SignifChange}()$, which evaluates the difference among $\hat{\mathcal{L}}_h$ and $\hat{\mathcal{L}}_j$ using a divergence test, in Section \ref{sec:divergence test}.  }


\algdef{SE}[DOWHILE]{Do}{doWhile}{\algorithmicdo}[1]{\algorithmicwhile\ #1}%
\begin{algorithm}[!t]
\caption{Joint Perception and Learning Algorithm} 
\label{algo:JIRP ext}
\begin{algorithmic}[1]    
\State $\mathcal{H} \leftarrow \mathsf{Initialize}(V)$ 
\State $Q = \{q^v | v \in V\} \leftarrow \mathsf{Initialize}(V)$
\State $X \leftarrow \emptyset$ 
\State $j\leftarrow 0$
\State $\hat{\mathcal{L}}_h \leftarrow \hat{\mathcal{L}}_0$;	 
\For{episode $n=1,2,\dots$} 
\State $(\lambda,\rho,Q, \hat{\mathcal{L}}_j) \leftarrow $ QRM\_episode\_mod($\mathcal{H},Q,\hat{\mathcal{L}}_h,\hat{\mathcal{L}}_j$)
\If{$\mathcal{H}(\lambda) \neq \rho$}
	\hspace{12mm}	\small \textcolor{blue}{// \texttt{The trace $(\lambda,\rho)$ is inconsistent with $\mathcal{H}$}} \normalsize 
\State add $(\lambda,\rho)$ to $X$
		\hspace{12mm} \small \textcolor{blue}{// \texttt{Append the set $X$ with the trace $(\lambda,\rho)$}} \normalsize
\State $\mathcal{H} \leftarrow \mathsf{Infer}(X)$ 
		\hspace{12mm} \small \textcolor{blue}{// \texttt{Infer a new reward machine from  $X$}} \normalsize
\State $Q\leftarrow \mathsf{Initialize}()$ 
\EndIf
\If{ $\mathsf{SignifChange}(\hat{\mathcal{L}}_h, \hat{\mathcal{L}}_j) $ }
		\hspace{10mm} \small \textcolor{blue}{// \texttt{$\hat{\mathcal{L}}_j$ has changed significantly with respect to $\hat{\mathcal{L}}_h$   }} \normalsize
\State $\mathcal{H} \leftarrow \mathsf{Initialize}(V)$ 
\State $Q\leftarrow \mathsf{Initialize}(V)$ 
\State $X \leftarrow \emptyset$  
\State $\hat{\mathcal{L}}_h \leftarrow \hat{\mathcal{L}}_j$
\EndIf
\EndFor
\end{algorithmic}
\end{algorithm}

\algdef{SE}[DOWHILE]{Do}{doWhile}{\algorithmicdo}[1]{\algorithmicwhile\ #1}%
\begin{algorithm}[!t]
\caption{QRM\_episode\_mod($\mathcal{H}$, $Q$, $\hat{\mathcal{L}}_h,\hat{\mathcal{L}}_j$)} 
\label{algo:QRM_mod}
\begin{algorithmic}[1]    
\Require{A reward machine $\mathcal{H}=(V,v_I,2^\mathcal{P},\mathbb{R},\delta,\sigma)$, a set of $q$-functions $Q = \{q^v | v\in V\}$, two probabilistic beliefs $\hat{\mathcal{L}}_h,\hat{\mathcal{L}}_j$.}
\State $s \leftarrow \mathsf{InitialState}()$; $v \leftarrow v_I$, $\lambda \leftarrow []$; $\rho \leftarrow []$
\For{$1 \leq t < eplength$} 
\State $a \leftarrow \mathsf{GetEpsilonGreedyAction}(q^v,s)$
\State $s' \leftarrow \mathsf{ExecuteAction}(s,a)$
\hspace{3mm} \small \textcolor{blue}{// \texttt{Transition in \textcolor{blue}{l-MDP} $\mathcal{M}$ based on $\mathcal{T}(s',a,s)$}} \normalsize
\State $v' \leftarrow \delta(v, \hat{L}(\hat{\mathcal{L}}_h,s'))$
\hspace{11mm} \small \textcolor{blue}{// \texttt{Transition in $\mathcal{H}$ based on $\hat{\mathcal{L}}_h$}} \normalsize
\State $r \leftarrow \mathsf{ObserveReward}()$
\State $q^v(s,a) \leftarrow \mathsf{Update}(r,s')$
\hspace{6mm} \small \textcolor{blue}{// \texttt{Update of $q$-function for state $v$ of $\mathcal{H}$}} \normalsize
\For{$\hat{v} \in V \backslash{v}$}
\State $\hat{v}' \leftarrow \delta(\hat{v}, \hat{L}(\hat{\mathcal{L}}_h, s'))$
\hspace{7mm} \small \textcolor{blue}{// \texttt{Transition in $\mathcal{H}$ based on $\hat{\mathcal{L}}_h$}} \normalsize
\State $\hat{r} \leftarrow \sigma(\hat{v},\hat{L}(\hat{\mathcal{L}}_h, s'))$
\hspace{7mm} \small \textcolor{blue}{// \texttt{Reward calculation from $\mathcal{H}$ based on $\hat{\mathcal{L}}_h$}} \normalsize
\State $q^{\hat{v}}(s,a) \leftarrow \mathsf{Update}(\hat{r},s')$
\hspace{0mm} \small \textcolor{blue}{// \texttt{Update of $q$-function for state $\hat{v}$ of $\mathcal{H}$}} \normalsize
\EndFor

\State $j\leftarrow (n-1)\cdot eplength + t$
\State $\hat{\mathcal{L}}_{j} \leftarrow \mathsf{BayesUpdate}(\hat{\mathcal{L}}_{j-1},s')$
\hspace{10mm} \small \textcolor{blue}{// \texttt{Update probabilistic belief $\mathcal{L}_j$}} \normalsize
\State $\mathsf{Append}(\lambda, \hat{L}(\hat{\mathcal{L}}_h,s'))$; $\mathsf{Append}(\rho, r)$
\State $s \leftarrow s'$; $v\leftarrow v'$
\EndFor
\State \textbf{return} $(\lambda,\rho,Q,\hat{\mathcal{L}}_j)$
\end{algorithmic}
\end{algorithm}

\subsection{Information Processing} \label{sec:information proc}

\textcolor{blue}{
This section provides the $\mathsf{BayesUpdate}(\hat{\mathcal{L}}_{j-1},s')$ function (line 14 of Alg. \ref{algo:QRM_mod}) that updates the
probability distribution $\hat{\mathcal{L}}_j$ based on the observation made by the agent at state $s'$.}

When at state $s'$, the agent will receive new perception outputs according to the observation model $\mathcal{O}(s',\dots)$ for all states and atomic propositions. 
The agent then employs the observations to update its learned model of the environment in a Bayesian approach. For ease of notation, let $\hat{\mathcal{L}}_j = \hat{\mathcal{L}}_j( s, p)$ and $\mathcal{O}(b) = \mathcal{O}( s', s, p, b)$. Given the prior belief of the agent $\hat{\mathcal{L}}_{j-1}$ and the received observations $\mathcal{Z}$, the posterior belief follows 
\begin{subequations} \label{eq:prob updates}
\begin{align}
Pr\big( s \models p | \mathcal{Z}(s',s,p) = \mathsf{True} \big) & =  \frac{ \hat{\mathcal{L}}_{j-1} \mathcal{O}( \mathsf{True} ) }{ \hat{\mathcal{L}}_{j-1} \mathcal{O}( \mathsf{True} ) +  (1 - \hat{\mathcal{L}}_{j-1}) \mathcal{O}(\mathsf{False}) } \label{eq:prob updates 1}\\
Pr\big( s \models p | \mathcal{Z}(s',s,p) = \mathsf{False} \big) &=  \frac{ \hat{\mathcal{L}}_{j-1} (1-\mathcal{O}(\mathsf{True})) }{ \hat{\mathcal{L}}_{j-1} (1-\mathcal{O}(\mathsf{True})) +  (1 - \hat{\mathcal{L}}_{j-1}) (1- \mathcal{O}(\mathsf{False})) } \label{eq:prob updates 2}
\end{align}
\end{subequations}
for all $s \in S$ and $p\in\mathcal{AP}$. \textcolor{blue}{Depending on the truth value observed for $p$, $\mathsf{BayesUpdate}(\hat{\mathcal{L}}_{j-1},s')$ updates $\hat{\mathcal{L}}_j$ according to one of the aforementioned expressions}. Besides, for any $P \subseteq \mathcal{AP}$, 
$\hat{\mathcal{L}}_j(s,P) = \prod_{p \in P}     \hat{\mathcal{L}}_j(s,p)$.

\subsection{Divergence Test on the Belief} \label{sec:divergence test}

\textcolor{blue}{
This section provides the $\mathsf{SignifChange}(\hat{\mathcal{L}}_{h},\hat{\mathcal{L}}_{j})$ function (line 13 of Alg. \ref{algo:JIRP ext}) that compares the distributions $\hat{\mathcal{L}}_{h}$ and $\hat{\mathcal{L}}_{j}$.}

As mentioned before, \textcolor{blue}{if the agent’s current estimate about the atomic propositions}, encoded by the current belief $\hat{\mathcal{L}}_j$, has not changed significantly from its previous estimate (encoded by $\hat{\mathcal{L}}_h$), then the agent will continue executing the algorithm using the current $\hat{\mathcal{L}}_h$. \textcolor{blue}{However}, if its estimate has changed significantly, the agent will update its belief and aim to infer a new reward machine (lines 13-18 of Alg. \ref{algo:JIRP ext}). 
\textcolor{blue}{The algorithm compares the two distributions, $\hat{\mathcal{L}}_{h}$ and $\hat{\mathcal{L}}_{j}$ via the function $\mathsf{SignifChange}()$. }
The function uses the
Jensen-Shannon divergence to quantify the change in the
belief distribution between two consecutive time steps. 
\textcolor{blue}{We note here that alternative divergence methods would also work without altering the correctness of the algorithm (proven in Sec. 5.4). We choose the Jensen-Shannon divergence due to its symmetry, finite value, and its efficiency for generic distributions over atomic propositions.}
The
cumulative Jensen-Shannon divergence over the states and
the propositions can be expressed as
\begin{align*}
D_{\mathcal{JSD}}(\hat{\mathcal{L}}_{h} & || \hat{\mathcal{L}}_j) =\frac{1}{2} D_{\mathcal{KL}}(\hat{\mathcal{L}}_{h} || \hat{\mathcal{L}}_{m}) + \frac{1}{2} D_{\mathcal{KL}}(\hat{\mathcal{L}}_{j} || \hat{\mathcal{L}}_m) = \\
&
\frac{1}{2} \sum_{s \in S} \sum_{p\in\mathcal{AP}} \hat{\mathcal{L}}_{h}(s,p) \log\frac{\hat{\mathcal{L}}_{h}(s,p)}{ \hat{\mathcal{L}}_{m}(s,p)} + (1 - \hat{\mathcal{L}}_h(s,p))  \log\frac{1-\hat{\mathcal{L}}_{c}(s,p)}{ 1-\hat{\mathcal{L}}_{m}(s,p)} \\ 
&+
\frac{1}{2} \sum_{ s \in S } \sum_{p\in\mathcal{AP}} \hat{\mathcal{L}}_{j}(s,p) \log\frac{\hat{\mathcal{L}}_{j}(s,p)}{ \hat{\mathcal{L}}_{m}(s,p)} + (1 - \hat{\mathcal{L}}_{j}(s,p))  \log\frac{1-\hat{\mathcal{L}}_{j}(s,p)}{ 1-\hat{\mathcal{L}}_{m}(s,p)} 
\end{align*}
where $\hat{\mathcal{L}}_{m} = \frac{1}{2}(\hat{\mathcal{L}}_{h} + \hat{\mathcal{L}}_{j})$ is the average distribution. 
One of the input parameters to the algorithm is a threshold
$\gamma_d$ on the above divergence; \textcolor{blue}{$\mathsf{SignifChange}(\hat{\mathcal{L}}_{h},\hat{\mathcal{L}}_{j})$ returns $\mathsf{False}$ if $D_{\mathcal{JSD}}(\hat{\mathcal{L}}_{h}  || \hat{\mathcal{L}}_j) < \gamma_d$ and $\mathsf{True}$ if $D_{\mathcal{JSD}}(\hat{\mathcal{L}}_{h}  || \hat{\mathcal{L}}_j) \geq \gamma_d$, determining  whether the
agent will use its previous estimate $\hat{L}(\hat{\mathcal{L}}_h,\cdot)$ or update $\hat{\mathcal{L}}_h$ 
to guide the learning and reward machine inference.}

\subsection{Inference of Reward Machines} \label{sec:inference}

The goal of reward-machine inference is to find a reward machine $\mathcal{H}$ that is consistent with all the counterexamples in the sample set $X$, i.e., such that $\mathcal{H}(\lambda) = \rho$ for all $(\lambda,\rho) \in X$, for the current belief $\hat{\mathcal{L}}_h$. Unfortunately, such a task is computationally hard in the sense that the corresponding decision problem ``given a sample X and a natural number $k>0$, does a consistent Mealy machine with at most $k$ states exist?" is NP-complete \cite{xu2020joint,gold1978complexity}. 
In this paper we follow the approach of \cite{xu2020joint}, which is \textcolor{blue}{to learn minimal} consistent reward machines with the help of highly-optimized SAT solvers \cite{neider2013regular,heule2010exact,neider2014applications}. The underlying idea is to generate a sequence of formulas $\phi^X_k$ in propositional logic for increasing values of $k \in \mathbb{N}$ (starting with $k=1$) that satisfy the following two properties:
\begin{itemize}
\item $\phi^X_k$ is satisfiable if and only if there exists a reward machine with k states that is consistent with X;
\item a satisfying assignment of the variables in $\phi^X_k$ contains sufficient information to derive such a reward machine. 
\end{itemize}

By increasing $k$ by one and stopping once $\phi^X_k$ becomes satisfiable (or by using a binary search), \textcolor{blue}{we construct an algorithm that
learns a minimal reward machine that is consistent with the given sample}.

\textcolor{blue}{As stressed in Example \ref{ex:office}, a reward machine that encodes $R$ on $\hat{\mathcal{L}}_h$ might \textit{not} exist and the inference procedure might end up creating an arbitrarily large Mealy machine and yielding very high computation times (trying to compute a Mealy machine that does not exist). 
To prevent such cases, one can limit the maximum allowable number of states of the Mealy machine and the maximum number of episodes used to generate the counterexamples, and use the last valid inferred Mealy machine (or the initial estimate if there is no valid one). 
We note that such a technique requires knowledge of an upper bound on the number of states of the reward machine. Practically, this bound could be a conservatively large number that could be derived from the total number of atomic propositions $\mathcal{AP}$. Nevertheless, large bounds could lead to high computational times in cases the algorithm is trying to infer a reward machine that does not exist due to inaccuracies in $\hat{\mathcal{L}}_h$. Such complexity stems from the inherent scalability issues of the employed SAT-based optimizers and can be alleviated by modifying the underlying algorithms to obtain polynomial-time heuristics \cite[Sec. 4.3]{xu2020joint} . 
In the experimental results, the proposed algorithm,  such cases did not prevent the successful
 convergence of the proposed algorithm.}

\textcolor{blue}{A sufficient condition for a reward machine that encodes $R$ on $\hat{\mathcal{L}}_h$ to exist is sufficient accuracy of $\hat{\mathcal{L}}_h$, i.e., $\hat{L}(\hat{\mathcal{L}}_h,s) = \{ p \in \mathcal{AP}: \hat{\mathcal{L}}_h(s,p) \geq 0.5 \} = L_G(s)$, for all $s\in S$  (see \eqref{eq:L_hat}). Intuitively, such a property implies that the most probably atomic-proposition estimates are the correct ones. We show later that such a property can be achieved with an unambiguous observation model and sufficient exploration of the environment. In case $\hat{L}(\hat{\mathcal{L}}_h,s) \neq L_G$, it is not straightforward to derive conditions for the existence of a reward machine that encodes $R$ on $\hat{\mathcal{L}}_h$. Based on intuition, one could argue that small inaccuracies in $\hat{\mathcal{L}}_h$, such as wrong labels in just one of the environment states, could lead to such a property.  Consider, however, the grid example of Fig. 2 and the case where $\hat{L}(\hat{\mathcal{L}}_h,s_{ij})=L_G(s_{ij})$ for all $i\in\{1,2\}$, $j\in\{1,\dots,5\}$, except for $\hat{L}(\hat{\mathcal{L}}_h,s_{14}) = \{``X"\} \neq L_G(s_{14}) = \{\}$, i.e., the sensor falsely detects an obstacle at state $s_{14}$. In that case, there is no reward machine that encodes $R$ on $\hat{\mathcal{L}}_h$: if such a reward machine $\mathcal{A}$ existed, its reward sequence $\mathcal{A}(\hat{\ell}_0,\dots,\hat{\ell}_k)$, given an input label sequence from any trajectory $s_0a_0\dots a_k s_{k+1}$ and based on $\hat{\mathcal{L}}_h$, would need to match the observed reward sequence. Consider, however, the trajectory $s_{11}s_{12}s_{13} s_{14} s_{24}$, which gives the \textit{observed} label sequence $\{``\text{c}"\},\{``\text{X}"\},\{``\text{o}"\}$. In order for $A(\{``\text{c}"\},\{``\text{X}"\},\{``\text{o}"\})$ to match the observed reward sequence, which is $0,0,0,1$, $\mathcal{A}$ would have to neglect the obstacle label $``\text{X}"$. In that case, however, the reward machine would output a reward sequence of $0,0,0,1$ also for the trajectory $s_{11}s_{12}s_{22}s_{23} s_{24}$, which gives the observed label sequence $\{``\text{c}"\},\{``\text{X}"\},\{``\text{X}"\},\{``\text{o}"\}$, i.e., $\mathcal{A}(\{``\text{c}"\},\{``\text{X}"\},\{``\text{X}"\},\{``\text{o}"\}) = 0,0,0,1$. The \textit{observed} reward sequence, however, is a sequence of zeros due to encounter of obstacles. Hence, there is no reward machine that encodes $R$ on $\hat{\mathcal{L}}_h$.  } 
 
\textcolor{blue}{Similarly, one might argue that if $\hat{L}(\hat{\mathcal{L}}_h,s)$ is not equal to $L_G$, then there is no reward machine that encodes $R$ on $\hat{\mathcal{L}}_h$. 
Consider, however, the scenario where the task is to bring coffee \textit{or} water to the office without encountering any obstacles and the extra atomic proposition ``water" is found in $s_{14}$, i.e., $L_G(s_{14}) = \{``\textit{w}"\}$. Further assume that  $\hat{L}(\hat{\mathcal{L}}_h,s_{ij}) = L_G$ for all $i\in\{1,2\}$, $j\in\{1,\dots,5\}$, except for $\hat{L}(\hat{\mathcal{L}},s_{14}) = \{ ``\textit{c}" \} \neq L_G(s_{14}) = \{``\textit{w}"\}$, i.e., the sensor falsely detects ``coffee" instead of ``water" in $s_{14}$. In that case, 
it can be verified that the observed reward sequence of any trajectory matches the sequence $\mathcal{A}(\hat{\ell}_0,\dots,\hat{\ell}_k)$, where $\mathcal{A}$ is a reward machine similar to the one Fig. \ref{fig:automata coffee}. Therefore, the existence of a reward machine that encodes $R$ on $\hat{\mathcal{L}}_h$ depends on the underlying task and its relation to the atomic propositions. }

\subsection{Convergence in the limit}

In this section, we establish  the correctness of Algorithm \ref{algo:JIRP ext}. \textcolor{blue}{More specifically, we provide theoretical guarantees on (i)
the convergence of the hypothesis reward machine $\mathcal{H}$ to the actual reward machine $\mathcal{A}$, and (ii) the asymptotic learning of the optimal policy that achieves accomplishment of the task encoded by $\mathcal{A}$}.

The aforementioned guarantees 
are based on two sufficient conditions. First, the estimate 
$\hat{L}(\hat{\mathcal{L}}_h, \cdot)$ becomes identical with the ground-truth labelling function $L_G$ after a finite number of episodes. Second, the length of each episode is larger than $2^{|\mathcal{M}|+1}(|\mathcal{A}| + 1)-1$, where $\mathcal{A}$ is the unknown reward machine that encodes the task to be accomplished. 
The first condition is tightly connected to the observation model $\mathcal{O}$ of the agent (see Def. \ref{def:ob_model}) and its effect on the Bayesian updates \eqref{eq:prob updates}. 
\textcolor{blue}{Intuitively, when there is unambiguity in the observation model, i.e., it is sufficiently accurate or inaccurate, the Bayesian updates \eqref{eq:prob updates}, along with consistent observations, lead to accurate estimation of $L_G$. 
In the next section, we show that such unambiguity implies $\mathcal{O}(s_1,s_2,p,\mathsf{True}) = Pr(\mathcal{Z}(s_1,s_2,p) = \mathsf{True} | s_2 \models p) > \mathcal{O}(s_1,s_2,p,\mathsf{False}) = Pr(\mathcal{Z}(s_1,s_2,p) = \mathsf{False} | s_2 \not \models p)$ or  $\mathcal{O}(s_1,s_2,p,\mathsf{True}) = Pr(\mathcal{Z}(s_1,s_2,p) = \mathsf{True} | s_2 \models p) < \mathcal{O}(s_1,s_2,p,\mathsf{False}) = Pr(\mathcal{Z}(s_1,s_2,p) = \mathsf{False} | s_2 \not \models p)$ and we illustrate that, in such cases, observations that are consistent with the truth values of $\mathcal{AP}$ and the observation model lead to convergence of $\hat{L}(\hat{\mathcal{L}}_i, \cdot)$ to $L_G(\cdot)$.    }
Nevertheless, such conditions are only sufficient for accurate estimation of $L_G$. The experiments of Section \ref{sec:exps} illustrate that the proposed algorithm learns how to accomplish the underlying task even with random observation models.


\textcolor{blue}{
The convergence results are given in the following Theorem, whose proof can be found in the Appendix:
\begin{theorem} \label{th:main}
Let $\mathcal{M} = (S, s_I,A, \mathcal{T}, R, \mathcal{AP}, \hat{L}, \gamma)$ be an \textcolor{blue}{l-MDP} and $\mathcal{A}$ be a reward machine that encodes the reward function $R$ on the ground truth. Further, assume that there exists a constant $n_f > 0$ such that $\hat{L}(\hat{\mathcal{L}}_h,\cdot) = L_G(\cdot)$, for all episodes $n > n_f$ of Algorithm \ref{algo:JIRP ext}\footnote{Note that this implies that $\{p \in \mathcal{AP}: \hat{L}_h(s,p) \geq 0.5\} = L_G(s)$ for all $s\in S$.}.
Then, Algorithm \ref{algo:JIRP ext}, with $eplength \geq 2^{|\mathcal{M}| + 1}(|\mathcal{A}| + 1) - 1$, converges almost surely to an optimal policy in the limit. 
\end{theorem}
}

\textcolor{blue}{
\begin{remark}
It can be concluded that the main differences of the proposed algorithm with respect to JIRP \cite{xu2020joint} consist of the updates of the labelling-function estimate (Sec. \ref{sec:information proc}) and the divergence test (Sec. \ref{sec:divergence test}). As mentioned before, however, these techniques can be applied to more general RL frameworks that rely on labelling functions \cite{jin2022creativity,hasanbeig2021deepsynth}. Such an application follows a procedure similar to the extension of the JIRP algorithm in the proposed methodology. 
More specifically, all the operations of the respective algorithms can be implemented based on a labelling-function estimate $\hat{\mathcal{L}}_h$. Since the algorithms in these works rely on exploration of the environment by the agents, as commonly done in RL, 
sensor observations can be used to update a running estimate $\hat{\mathcal{L}}_i$ based on an observation model, as done in Sec. \ref{sec:information proc}. Then, a divergence test can determine the update of $\hat{\mathcal{L}}_h$ with $\hat{\mathcal{L}}_i$, as done in Sec. \ref{sec:divergence test}.
\end{remark}
}

\textcolor{blue}{\subsection{Accurate estimation of $L_G$}} \label{sec:L estimate}

\textcolor{blue}{This section provides further insights on the relation of the Bayesian updates \eqref{eq:prob updates} and the observation model $\mathcal{O}$ towards the accurate estimation of $L_G$ by $\hat{L}$ given in \eqref{eq:L_hat}.} 
\textcolor{blue}{For ease of exposition, we assume that the agent receives observations about atomic proposition $p\in\mathcal{AP}$ in a state $s\in S$ only when it is located at that state. In the following, we use
$\mathcal{O}(s,s,p,\mathsf{True}) = \mathcal{O}_T$ and $\mathcal{O}(s,s,p,\mathsf{False}) = \mathcal{O}_F$.  }

\textcolor{blue}{We  derive a closed-form expression for the probabilistic belief $\hat{\mathcal{L}}_N(s,p)$ regarding an atomic proposition $p\in\mathcal{AP}$ in a state $s\in S$ after $N$ observations. From these $N$ observations, we consider that $N_T$ observations correspond to the detection $\mathcal{Z}(s,s,p) = \mathsf{True}$ and $N_F$ observations correspond to $\mathcal{Z}(s,s,p) = \mathsf{False}$, with $N = N_T + N_F$. We first show that $\hat{\mathcal{L}}_N(s,p)$ will be the same regardless of the order of the observations. Indeed, let $\hat{\mathcal{L}}_i(s,p)$ be the belief at the $i$th iteration, with $i \geq 0$. Assume that the agent then obtains first an observation $\mathcal{Z}(s,s,p) = \mathsf{True}$ followed by $\mathcal{Z}(s,s,p) = \mathsf{False}$. 
According to \eqref{eq:prob updates}, it holds that 
\begin{align} \label{eq:L i 1}
\hat{\mathcal{L}}_{i+1} &= \frac{\hat{\mathcal{L}}_{i} \mathcal{O}_T}{\hat{\mathcal{L}}_{i} \mathcal{O}_T + (1 - \hat{\mathcal{L}}_{i}) \mathcal{O}_F} \notag  \\
\hat{\mathcal{L}}_{i+2} &= \frac{\hat{\mathcal{L}}_{i+1}(1 - \mathcal{O}_T) }{ \hat{\mathcal{L}}_{i+1}(1 - \mathcal{O}_T)  + (1 - \hat{\mathcal{L}}_{i+1})(1 - \mathcal{O}_F)} \notag \\
&= \frac{ \hat{\mathcal{L}}_{i} \mathcal{O}_T (1 - \mathcal{O}_T)   }{ \hat{\mathcal{L}}_{i} \mathcal{O}_T (1 - \mathcal{O}_T)   +  \mathcal{O}_F ( 1- \hat{\mathcal{L}}_{i})(1 - \mathcal{O}_F) } 
\end{align}
where we omit the argument $(s,p)$ for brevity.
On the contrary, assume that the agent obtains first an observation $\mathcal{Z}(s,s,p) = \mathsf{False}$ followed by $\mathcal{Z}(s,s,p) = \mathsf{True}$. 
According to \eqref{eq:prob updates}, it holds that  
\begin{align} \label{eq:L i 2}
\hat{\mathcal{L}}_{i+1} &= \frac{\hat{\mathcal{L}}_{i}(1- \mathcal{O}_T)}{\hat{\mathcal{L}}_{i}(1 - \mathcal{O}_T) + (1 - \hat{\mathcal{L}}_{i})(1 - \mathcal{O}_F)} \notag \\
\hat{\mathcal{L}}_{i+2} &= \frac{\hat{\mathcal{L}}_{i+1} \mathcal{O}_T }{ \hat{\mathcal{L}}_{i+1}\mathcal{O}_T  + (1 - \hat{\mathcal{L}}_{i+1})\mathcal{O}_F} \notag \\
&= \frac{ \hat{\mathcal{L}}_{i} \mathcal{O}_T (1 - \mathcal{O}_T)   }{ \hat{\mathcal{L}}_{i} \mathcal{O}_T (1 - \mathcal{O}_T)   +  \mathcal{O}_F ( 1- \hat{\mathcal{L}}_{i})(1 - \mathcal{O}_F) },
\end{align}
from which we conclude that $\hat{\mathcal{L}}_{i+2}$ is the same in both cases. By induction, after $N = N_T + N_F$ observations, $\hat{\mathcal{L}}_N$ will be the same regardless of the sequence of $N_T$ and $N_F$ observations. 
}

\textcolor{blue}{
We now proceed to derive a general closed-form expression for $\hat{\mathcal{L}}_N$. Since the sequence of observations does not alter such an expression, we assume, without loss of generality, that the agent obtains first $N_T$ observations followed by $N_F$ observations. After $N_T$ observations  $\mathcal{Z}(s,s,p) = \mathsf{True}$, iterative application of \eqref{eq:prob updates 1} yields
\begin{align*}
\hat{\mathcal{L}}_{N_T} = \frac{\alpha^{N_T}}{\hat{\mathcal{L}}_0^{-1} + (\alpha - 1) \sum_{i=0}^{N_T - 1} \alpha^i} 
\end{align*}
where we denote $\alpha = \frac{\mathcal{O}_T}{\mathcal{O}_F}$. 
After another $N_F$ observations  $\mathcal{Z}(s,s,p) = \mathsf{False}$, iterative application of \eqref{eq:prob updates 2} yields 
\begin{align*}
\hat{\mathcal{L}}_{N} = \frac{( 1 - \alpha \mathcal{O}_F )^{N_F}}{D_F}
\end{align*}
where
\begin{align*}
D_F =& (1-\mathcal{O}_F)^{N_F} \hat{\mathcal{L}}_{N_T}^{-1} + \mathcal{O}_F (1 - \alpha) (1 - \alpha \mathcal{O}_F)^{N_F-1} \\
&+ \mathcal{O}_F (1 - \mathcal{O}_F)(1-\alpha) \sum_{i=0}^{N_F - 2}(1 - \alpha \mathcal{O}_F)^i (1 - \mathcal{O}_F)^{N_F - 2 - i} 
\end{align*}
Substituting $\hat{\mathcal{L}}_{N_T}$ in $\hat{\mathcal{L}}_{N}$ yields 
\begin{align} \label{eq:L_N}
\hat{\mathcal{L}}_N = \frac{\alpha^{N_T} ( 1 - \alpha \mathcal{O}_F )^{N_F}  }{D_N}
\end{align}
where
\begin{align*}
D_N =& \alpha^{N_T} \mathcal{O}_F (1 - \alpha) (1 - \alpha \mathcal{O}_F)^{N_F-1} + (1 - \mathcal{O}_F)^{N_F} \left( \hat{\mathcal{L}}_0^{-1} + (\alpha - 1) \sum_{i=0}^{N_T-1}\alpha^i \right) \\
&+ a^{N_T} \mathcal{O}_F (1 - \mathcal{O}_F)(1-\alpha) \sum_{i=0}^{N_F - 2}(1 - \alpha \mathcal{O}_F)^i (1 - \mathcal{O}_F)^{N_F - 2 - i}
\end{align*}
Note that, the case $\mathcal{O}_T = \mathcal{O}_F$, i.e., $\alpha = 1$, leads to $\hat{\mathcal{L}}_N = \hat{\mathcal{L}}_0$ for all $N >0$. Such cases entail an \textit{ambiguous} observation model and are equivalent to cases without the updates \eqref{eq:prob updates}; 
the probability of observing $\mathcal{Z}(s,s,p) = \mathsf{True}$ is the same regardless of whether $s\models p$ or $s\not \models p$ and the updates \eqref{eq:prob updates} fail to estimate $L_G$. In cases of an unambiguous observation model, however, satisfying $\mathcal{O}_T > \mathcal{O}_F$ or $\mathcal{O}_T < \mathcal{O}_F$, observations that are consistent with $\mathcal{O}$ and the truth value of $p$ in $s$ lead to accurate estimation of $L_G$. By consistent observations, we mean observations satisfying $N_T = \mathcal{O}_T N$ and $N_F = (1 - \mathcal{O}_T)N$ in case $s \models p$ and $N_T = \mathcal{O}_F N$ and $N_F = (1 - \mathcal{O}_F)N$ in case $s \not \models p$. } 

\textcolor{blue}{In order to formally prove the aforementioned conjecture, one would ideally need to show that $\hat{\mathcal{L}}_{N+1} > \hat{\mathcal{L}}_N$ in case $s \models p$ and $\hat{\mathcal{L}}_{N+1} < \hat{\mathcal{L}}_N$ in case $s \not \models p$. The complexity of the expression \eqref{eq:L_N}, however, renders such a procedure significantly difficult. Instead, we plot $\hat{\mathcal{L}}_N$ in \eqref{eq:L_N} as a function of $N$ for different values of the observation model satisfying $\mathcal{O}_T > \mathcal{O}_F$ and $\mathcal{O}_T < \mathcal{O}_F$. In particular, we consider four different cases formed by $\mathcal{O}_T > \mathcal{O}_F$,  $\mathcal{O}_T < \mathcal{O}_F$ and $s\models p$, $s \not \models p$. In each of these cases, we generate 500 different combinations of $\mathcal{O}_T$ and $\mathcal{O}_F$ randomly from a uniform distribution. In the cases where $\mathcal{O}_T > \mathcal{O}_F$, we choose $\mathcal{O}_T$ randomly in the interval $(0,1)$ and $\mathcal{O}_F$ in the interval $(0,\mathcal{O}_F)$. Reversely, in  the cases where $\mathcal{O}_T < \mathcal{O}_F$, we choose $\mathcal{O}_F$ randomly in the interval $(0,1)$ and $\mathcal{O}_T$ in the interval $(0,\mathcal{O}_T)$. In all cases we choose $\hat{\mathcal{L}}_0$ randomly from a uniform distribution in $(0,1)$. To obtain observations consistent with the observation model, we further choose $N_T = \mathcal{O}_T N$, $N_F = (1-\mathcal{O}_T)N$ when $s\models p$ and $N_T = \mathcal{O}_F N$, $N_F = (1-\mathcal{O}_F)N$ when $s \not \models p$. The results for the four different cases are depicted in Fig. 4. One can conclude from the figure that the $\hat{\mathcal{L}}_N(s,p)$ converges to $1$ and $0$ for when $s\models p$ and $s\not\models p$, respectively. 
}

\begin{figure}
\begin{subfigure}{.5\textwidth}
  \centering
  \includegraphics[width=\linewidth]{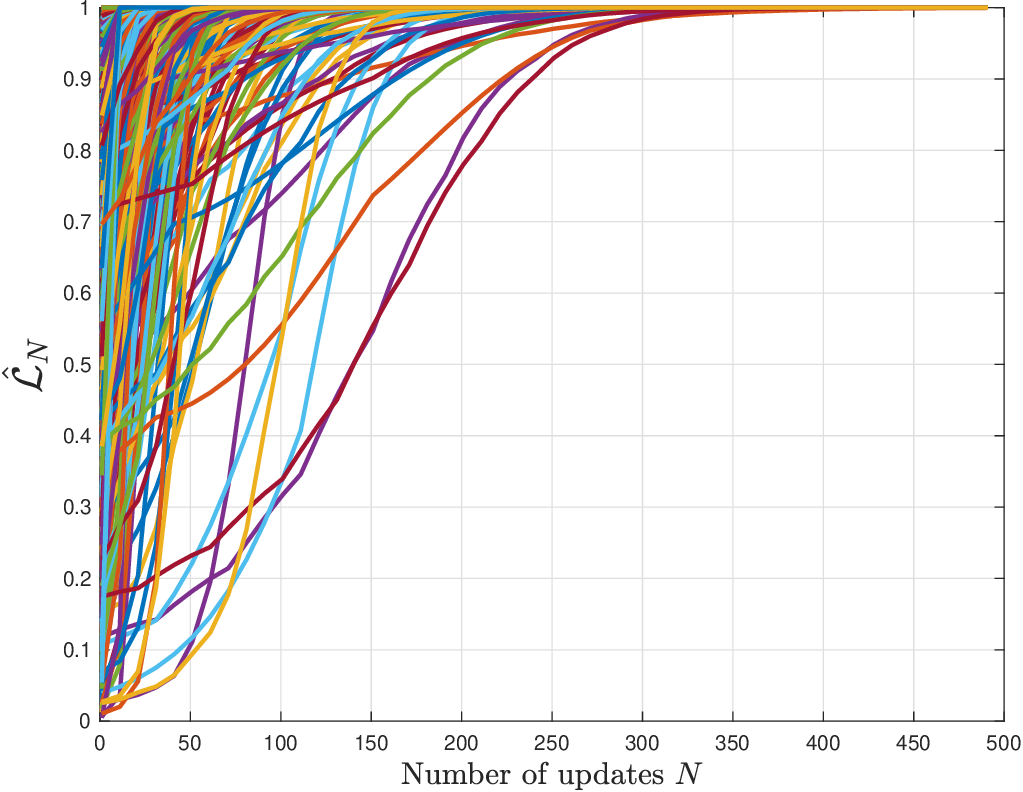}  
  \caption{Evolution of $\hat{\mathcal{L}}_N(s,p)$ when $s\models p$, $\mathcal{O}_T > \mathcal{O}_F$}
  \label{fig:L_N_s_accurate}
\end{subfigure}
\begin{subfigure}{.5\textwidth}
  \centering
  \includegraphics[width=\linewidth]{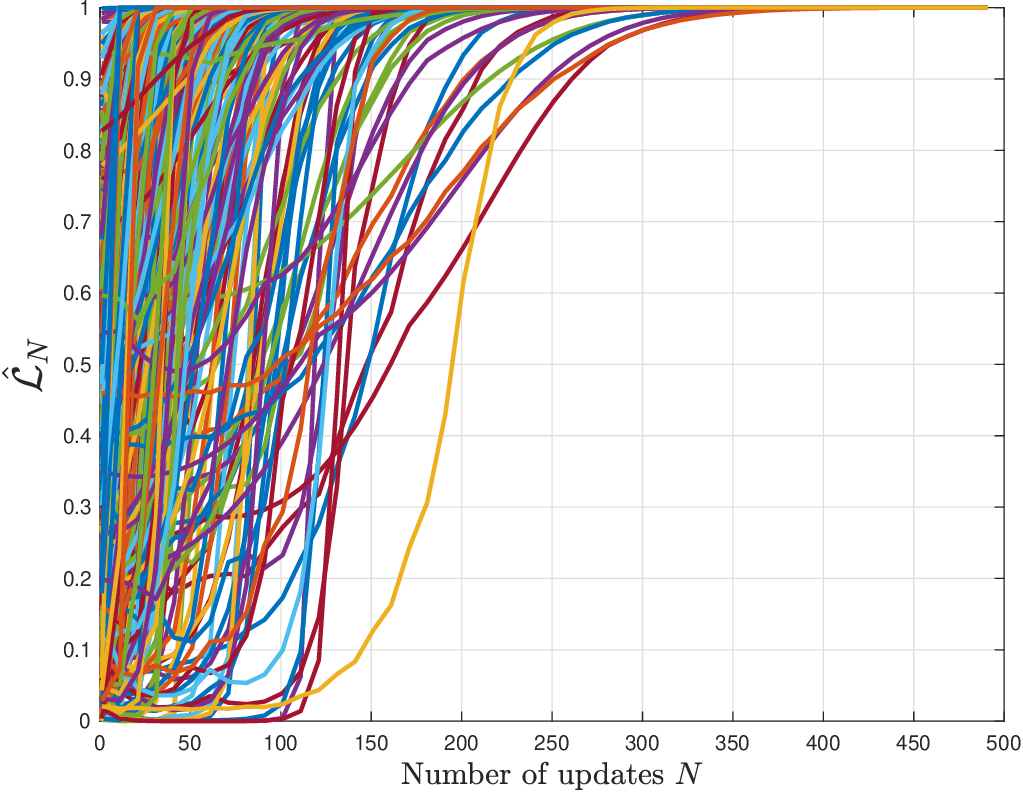}  
   \caption{Evolution of $\hat{\mathcal{L}}_N(s,p)$ when $s\models p$, $\mathcal{O}_T < \mathcal{O}_F$}
  \label{fig:L_N_s_not_accurate}
\end{subfigure}
\begin{subfigure}{.5\textwidth}
  \centering
  \includegraphics[width=\linewidth]{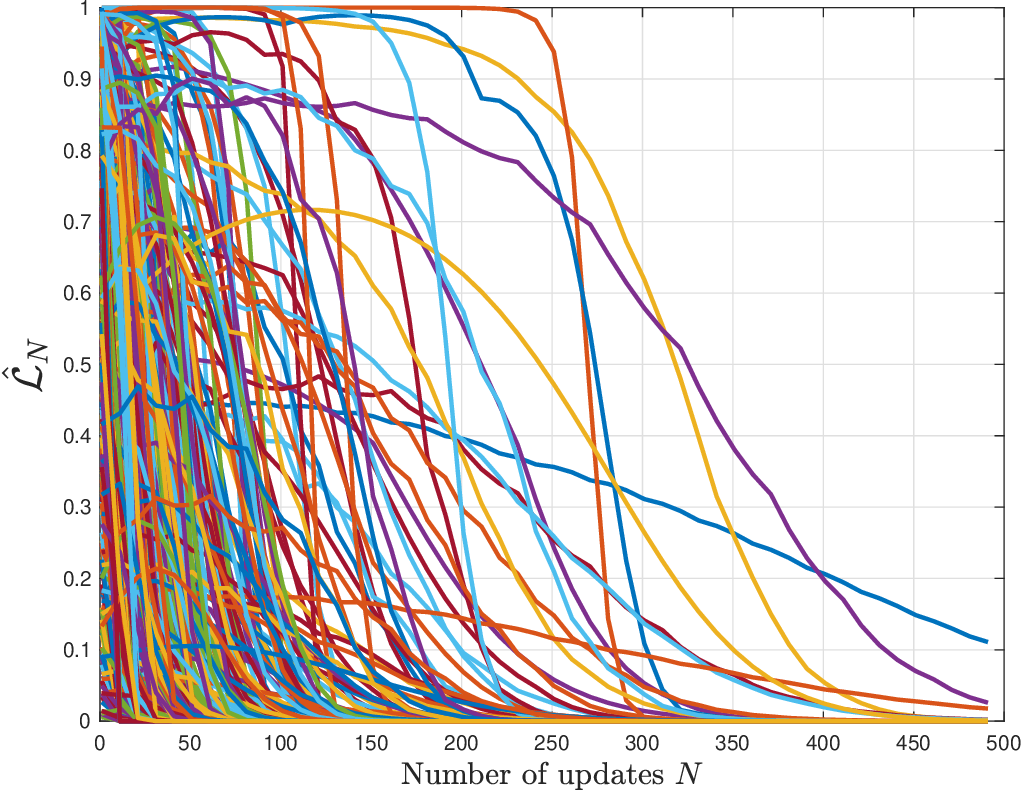}  
   \caption{Evolution of $\hat{\mathcal{L}}_N(s,p)$ when $s\not \models p$, $\mathcal{O}_T > \mathcal{O}_F$}
  \label{fig:L_N_not_s_accurate}
\end{subfigure}
\begin{subfigure}{.5\textwidth}
  \centering
  \includegraphics[width=\linewidth]{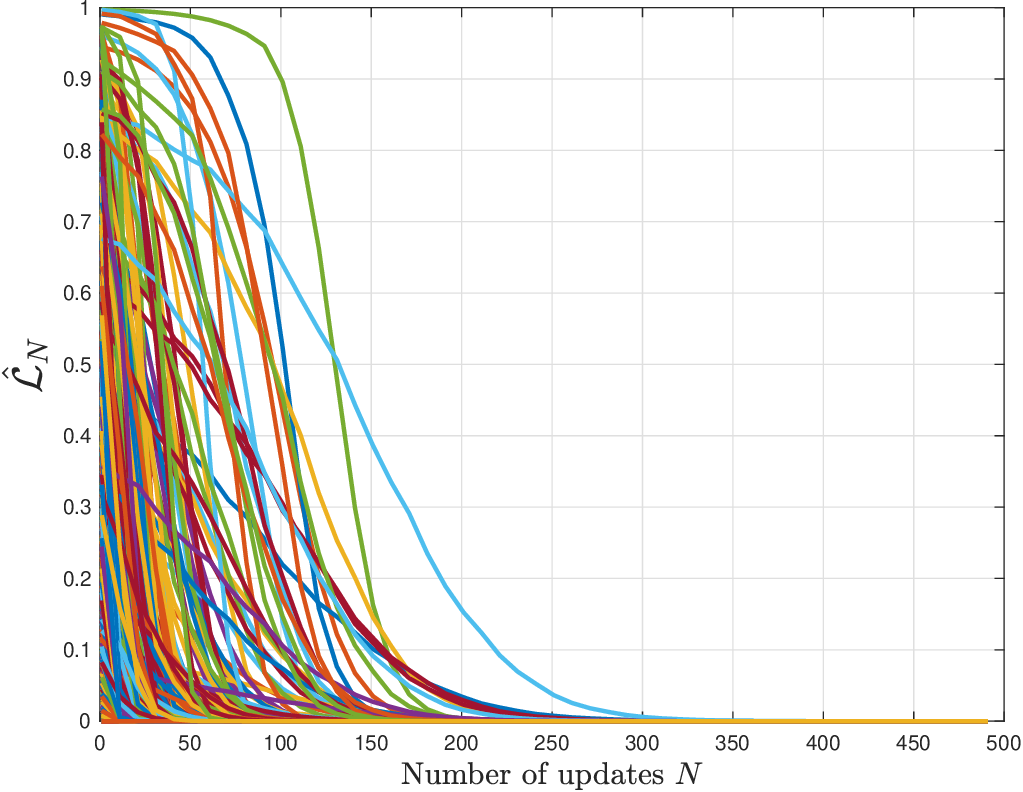}  
  \caption{Evolution of $\hat{\mathcal{L}}_N(s,p)$ when $s\not \models p$, $\mathcal{O}_T < \mathcal{O}_F$}
  \label{fig:L_N_not_s_not_accurate}
\end{subfigure}
\label{fig:L_N}
\caption{Evolution of $\hat{\mathcal{L}}_N(s,p)$ according to the Bayesian updates \eqref{eq:prob updates}. }
\end{figure}


\textcolor{blue}{
\subsection{Extension to continuous spaces}
}

\textcolor{blue}{
In cases where the state $S$ or action space $A$ is continuous, the tabular q-updates of the QRM algorithm (Alg. \ref{algo:QRM_mod}) cannot be implemented. In such cases, one can employ a parametrization $\theta^q \in \mathbb{R}^\ell$, for some $\ell \in \mathbb{N}$, in order to approximate the $q$ function for each state of the hypothesis reward machine. 
In particular, we define $q^v(s,a | \theta^q) :  S \times A \times \bar{\mathbb{V}} \times \mathbb{R}^\ell \to \mathbb{R}$ as a parametric approximation of $q^v(s,a)$ for each state-action pair $(s,a) \in S\times A$ and reward-machine state $v\in \bar{V}$;  $\bar{\mathbb{V}}$ is defined as $\bar{\mathbb{V}}= \{1,\dots,\bar{V}\}$, where $\bar{V}$ is an upper bound on the number of reward-machine states. Such an approximation is typically modelled using deep neural networks. 
The $q$-learning update of Alg. \ref{algo:QRM_mod} (lines 7 and 11) is then performed by minimizing the loss function $\mathbb{E}_{s \sim pr^\beta} [(q^v(s,a | \theta^q) - y)^2]$  \cite{hasanbeig2021deepsynth}, where $pr^\beta$ is the probability distribution of state-visit over $S$ under an arbitrary stochastic policy $\beta$ and $y= R(s_0\dots s')  + \gamma \max_{a' \in A} q^v(s',a' | \theta^q) $.
One could also append an intrinsic reward $r^{in}$ to the aforementioned update; such an intrinsic reward takes positive values when the agent observes a new label during exploration in an episode and zero otherwise and has been shown to aid with learning \cite{hasanbeig2021deepsynth,jin2022creativity}.}

\textcolor{blue}{We note that, although the underlying spaces are continuous, the atomic propositions naturally induce a discrete abstraction $S_A = \{S_{A_1},\dots,S_{A_M}\}$  
of the continuous state space $S$. Such an abstraction is formed by grouping states where the same atomic propositions hold true. That is, given $s,s' \in S$ and $S_{A_i} \in S_A$ a discrete state, it holds $s$, $s'$ $\in$ $S_{A_i}$ if and only if there exists a subset $P \subset \mathcal{AP}$ of atomic propositions satisfying $s \models p$ and $s' \models p$ for all $p \in P$. Therefore, the ground-truth labelling function $L_G:S \to 2^{\mathcal{AP}}$ is naturally extended from $S$ to $S_A$, i.e., $L_G:S_A \to 2^{\mathcal{AP}}$.
Finally, we note that one could use appropriate tools, such as the options framework \cite{jin2022creativity},
to abstract the continuous state and action spaces to discrete ones and apply the proposed algorithm.
}

\section{Experimental Results} \label{sec:exps}

\textcolor{blue}{
We test the proposed algorithm in three different environments, described in the following. 
In all environments, we employ the RC2 SAT solver \cite{morgado2014maxsat} from the PySAT library \cite{imms-sat18} towards reward-machine inference. All experiments were conducted on a Lenovo laptop with 2.70-GHz Intel i7 CPU and 8-GB RAM.}

\subsection{HouseExpo floor plans} \label{sec:exps floorplan} 
We first test the algorithm in floor plan environments based on the indoor layouts collected in the HouseExpo dataset \cite{li2019houseexpo}. HouseExpo consists of indoor layouts, built from 35,126 houses with a total of 252,550 rooms. There are 25 room categories, such as bedroom, garage, office, boiler room, etc. The dataset provides bounding boxes for the rooms with their corresponding categories, along with 2D images of layouts. We pre-process these layouts to generate grid-world-like floor plans, where every grid is labelled with the type of the room it is in (see Fig. \ref{fig:layout}). We use these labels as atomic propositions to define tasks encoded as reward machines. In particular, we use the set of atomic propositions $\mathcal{AP} = \{\tt{kitchen}, \tt{bathroom}, \tt{bedroom}, \tt{office}, \tt{indoor}\}$. 

In the experiments, an episode terminates when the agents receives a non-zero reward or the maximum number of steps allowed in one episode is exceeded. This maximum number of steps per episode is 1000, whereas the maximum number of total training steps is 500,000. 
Furthermore, we set the Jensen-Shannon divergence threshold $\gamma_d$ of Section \ref{sec:divergence test} to $10^{-5}$, and Algorithm \ref{algo:QRM_mod} selects a random action with probability $\varepsilon_a=0.3$.

We compare the following experimental settings regarding the observation model and the belief updates: 
\begin{enumerate}
\item Time-varying Bayesian Observation Model (TvBNN): We use a number of training HouseExpo layouts to train a Bayesian neural network, which the algorithm uses as an observation model in test layouts for the experiments. The algorithm samples observation probabilities of the atomic propositions' truth value from the neural network at the beginning of every training episode. 
\item Fixed Bayesian Observation Model (FiBNN): We use a number of training HouseExpo layouts to train a Bayesian neural network, which the algorithm uses as an observation model in test layouts for the experiments. The algorithm samples observation probabilities of the atomic propositions' truth value from the neural network only once, at the beginning of training.
\item Random Observation Model: The algorithm uses an observation model whose observation probabilities are uniformly sampled from the interval [0.1,0.9].
\item Random Observation Model-2: The algorithm uses an observation model whose observation probabilities are uniformly sampled from the interval [0.4,0.5].
\item No Belief Updates: The algorithm uses a belief function that is randomly initialized, without updating it.
\end{enumerate}

The Bayesian observation model in the first two settings is a neural network of three densely connected layers with flipout estimators \cite{wen2018flipout}. The loss function is the expected lower-bound loss, and a combination of Kullback-Leibler divergence and categorical cross entropy. Given the information about a grid in a floor plan, the Bayesian neural network (BNN) predicts category probabilities of the room where the grid is located.
The input vector consists of the 2D coordinates of the grid, the size and the number of neighbours of
the room. The output is the predicted probabilities of the room type. 
The training set consists of 3,000 pre-processed layouts from the HouseExpo dataset, which we select based on the floor plan size and the number of different room categories. After pre-processing, to eliminate variance
with respect to the initial position of grids, we rotate every training layout 3 times by 90 degrees, and then
add each rotated version to the training set. 
In all aforementioned cases, we simulate the observations to be consistent with the respective observation model, i.e., observe that $\mathcal{Z}(s_1,s_2,p) = \mathsf{True}$ and $\mathcal{Z}(s_1,s_2,p) = \mathsf{False}$ with probabilities $\mathcal{O}(s_1,s_2,p,\mathsf{True})$ and $1-\mathcal{O}(s_1,s_2,p,\mathsf{True})$, respectively, when $s_2 \models p$, and with probabilities $\mathcal{O}(s_1,s_2,p,\mathsf{False})$ and $1-\mathcal{O}(s_1,s_2,p,\mathsf{False})$, respectively, when $s_2 \not \models p$.

For the reward-machine-inference part of the algorithm, 
we limit the maximum allowable number of states for the hypothesis reward machine to 4. Similarly, we limit the maximum allowable size of the trace sets, which is the number of episodes whose traces are recorded, to 20. In case the inference exceeds the allowed number of states, the process is stopped and the training continues with the last valid hypothesis reward machine until the next inference step.

We first specify the following task to the agent: 
$\phi_1 = $ ``Go to $\tt{bedroom}$ or $\tt{office}$, and then to the $\tt{kitchen}$", which is encoded by the reward machine shown in Fig. \ref{fig:automata floorplan simple}. 
\begin{figure}
\centering
\includegraphics[width=\textwidth]{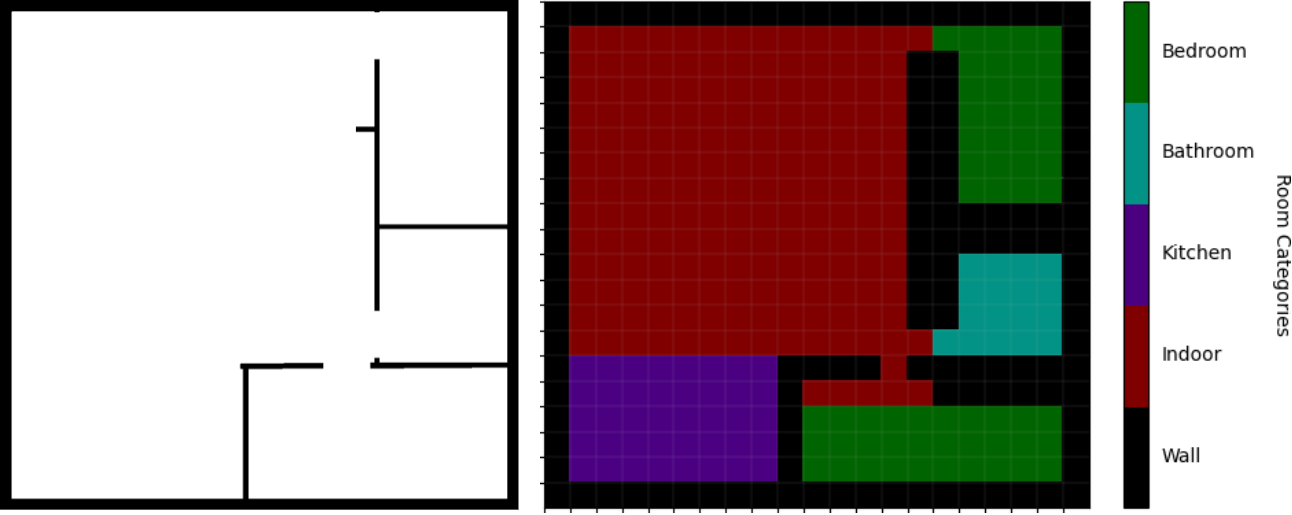}
\caption{Left: Sample indoor layout from HouseExpo dataset; Right: Pre-processed sample layout}
\label{fig:layout}
\end{figure}

\begin{figure}
\centering
\begin{tikzpicture}[scale =0.9]
\node[state, initial] (q0) {$v_0$};
\node[state, below left of=q0] (q1) {$v_1$};
\node[state, accepting, below right of=q1] (q2) {$v_2$};
\draw (q0) edge[loop above] node {\tt $(\neg \textup{bedroom} \land \neg \textup{office},0)$} (q0);
\draw (q0) edge node {\tt $(\textup{bedroom} \lor \textup{office},0)$} (q1);
\draw (q1) edge[loop left] node {\tt $(\neg \textup{kitchen}, 0)$} (q1);
\draw (q1) edge node [left] {\tt $(\textup{kitchen} ,1)$} (q2);
\draw (q2) edge[loop above] node {\tt $(\textup{True}, 0)$} (q2);
\end{tikzpicture}
\caption{Ground truth reward machine of the task $\phi_1$.}
\label{fig:automata floorplan simple}
\end{figure}
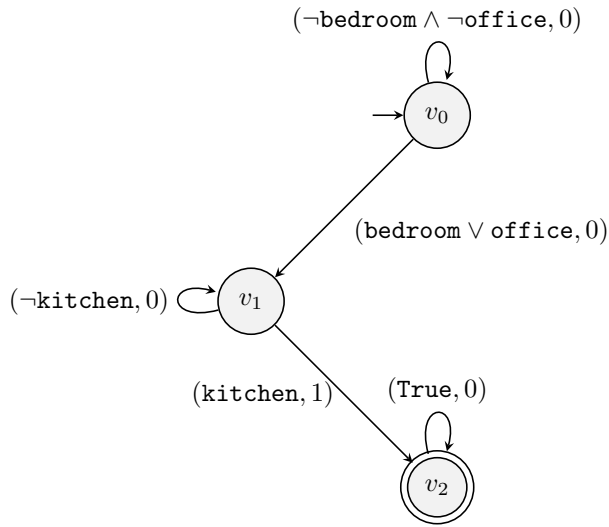

We execute the proposed algorithm in 75 test floor plans. For each one of the first four experimental settings, where updates according to an observation model are performed, 
the learning agent is trained 5 times in each test floor plan, each one corresponding to a different random seed, leading to a total of 375 runs. For the last setting that does not include belief updates, we removed 12 test floor plans where the inference procedure took unreasonably long time to finish, resulting in 63 test floor plans. This is attributed to the fact that the belief functions are randomly initialized and not updated throughout the training. For this last setting, the learning agent is also trained 5 times in each test floor plan, each one corresponding to a different random seed, leading to a  total of 290 runs. 
 In each one of the aforementioned runs, we perform an evaluation episode every 100 training steps and register the obtained rewards. \textcolor{blue}{Figure \ref{fig:rewards and belief floorplan 1} and Table \ref{table:comparison} show the results. In particular, Fig. \ref{fig:reward_time_varying_BNN}  depicts the progression of the rewards during training for the time-varying BNN (TvBNN), separated in the ones belonging to the $25^{th}$ and $75^{th}$ percentiles, and the ones belonging to the  $10^{th}$ and $90^{th}$ percentiles; The figure also depicts the median values with solid line. We omit the plots for the other observation models since they converge after a similar amount of steps, as can be verified by Table \ref{table:comparison}. In particular, Table \ref{table:comparison} shows the training steps required for convergence for the $25^{th}$ (Q1), $50^{th}$ (Q2), and $75^{th}$ (Q3) percentiles of the attained rewards for all cases of observation models. More specifically, we compute the $z^{th}$ percentile as follows: for each training step, we write the obtained rewards for all the 375 runs in increasing order and calculate the $z^{th}$ position; subsequently, the training steps for the $z^{th}$ percentile are the training steps required for this position to become $1$. 
 The table also depicts the ratio of successful inferences of reward machines (RS), and the average number of belief updates (BU) of $\hat{\mathcal{L}}_h$ after the JSD divergence test. 
}

Figures \ref{fig:belief_num_time}-\ref{fig:belief_jsd} demonstrate the progression of belief updates throughout the training; Fig. \ref{fig:belief_num_time} shows the number of training runs (y-axis) in which the belief function $\hat{\mathcal{L}}_h$ is updated for the $k^{th}$ time (x-axis). For instance, in the time-varying BNN (``TvBNN'') setting, the learning agent updates its belief for the third time in around 20 runs out of 375 training runs (5 runs per every floor plan where a reward machine is successfully inferred). 
Fig. \ref{fig:belief_step} shows the number of the step (y-axis) in which the belief function is updated for the $k^{th}$ time on average (x-axis), and Fig. \ref{fig:belief_jsd} displays the average Jensen-Shannon divergence value that led to the $k^{th}$ belief update. The time-varying BNN (``TvBNN'') setting allows the learning agent to obtain a belief function that represents the environment well in a couple of steps. When sampling observation probabilities once, as in the fixed BNN (``FiBNN'') setting, it slows down the process since the sampled observation might not be informative. Regarding the random observation model, due to the wide uniform distribution range, the agent samples probabilities that are close to zero or one, resulting to a more unambiguous observation model; that is, the agent is more ``certain" that an atomic proposition won't or will be observed, respectively. This is the reason why the second random observation model performs more belief updates until the divergence score drops below the threshold. Overall, it is expected to see that, in the BNN settings, the algorithm performs fewer belief updates compared to the random settings.   

\textcolor{blue}{
In view of the aforementioned results, one concludes that the labelling-function belief converges after much fewer steps than the ones needed for the learning of the optimal policy. Therefore, the different observation models used do not affect significantly the inference of the reward machine or the speed of the q-learning's convergence to the optimal policy, as long as such observation models lead to accurate estimation of the ground-truth labelling function. This fact can be verified by the experiments of Sec. \ref{sec:L estimate}, where the labelling-function estimate converges after approximately 500 updates (training steps), which is much fewer than the thousands of steps needed for learning the optimal policy (see Fig. \ref{fig:rewards and belief floorplan 1} and Table \ref{table:comparison}).
}

 \textcolor{blue}{
 \begin{figure}
\begin{subfigure}{.5\textwidth}
  \centering
  \includegraphics[width=\linewidth]{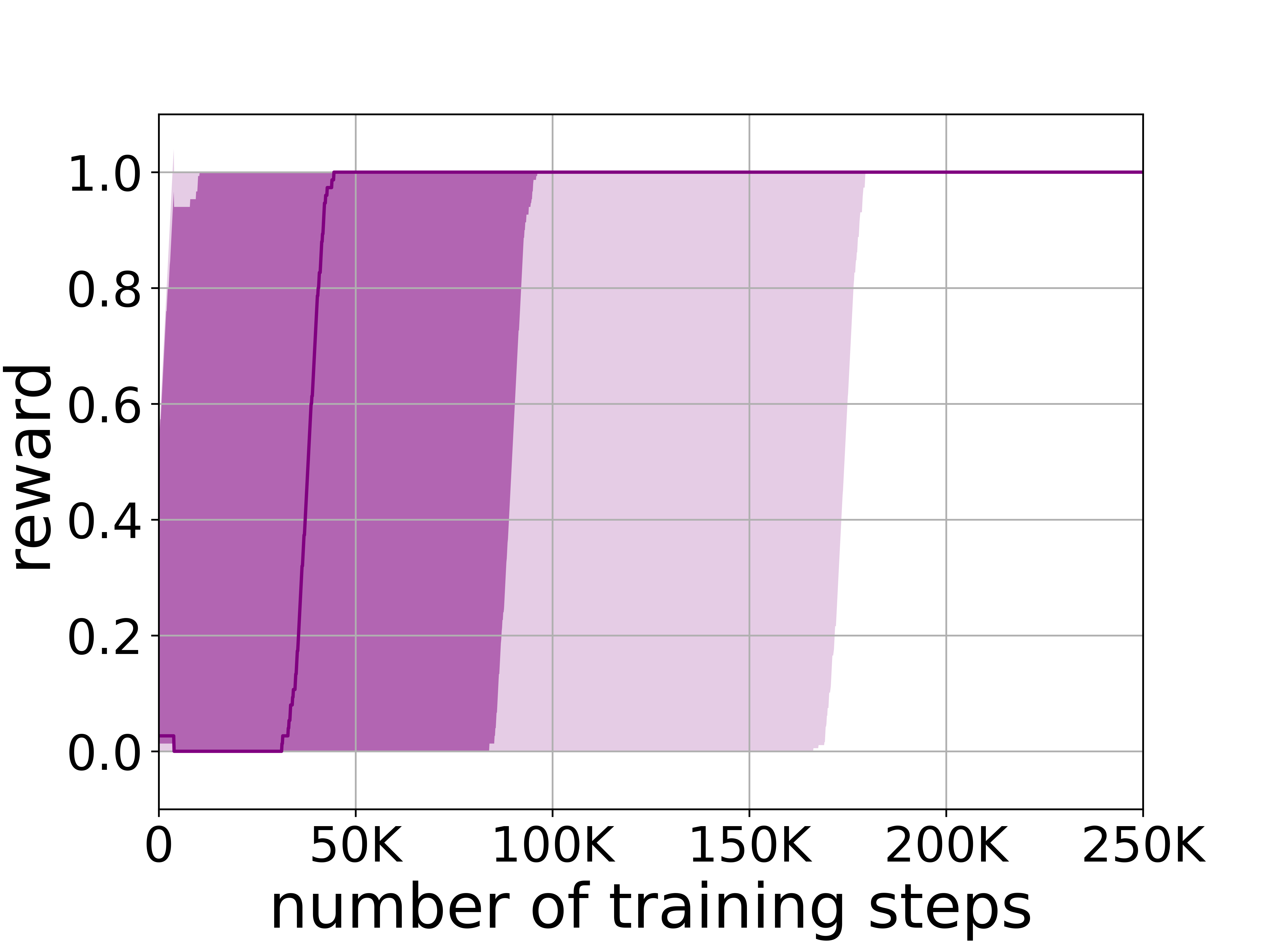}  
  \caption{\textcolor{blue}{
Time-varying BNN}}
  \label{fig:reward_time_varying_BNN}
\end{subfigure}
\begin{subfigure}{.5\textwidth}
  \centering
  \vspace{5mm}
  \includegraphics[width=\linewidth]{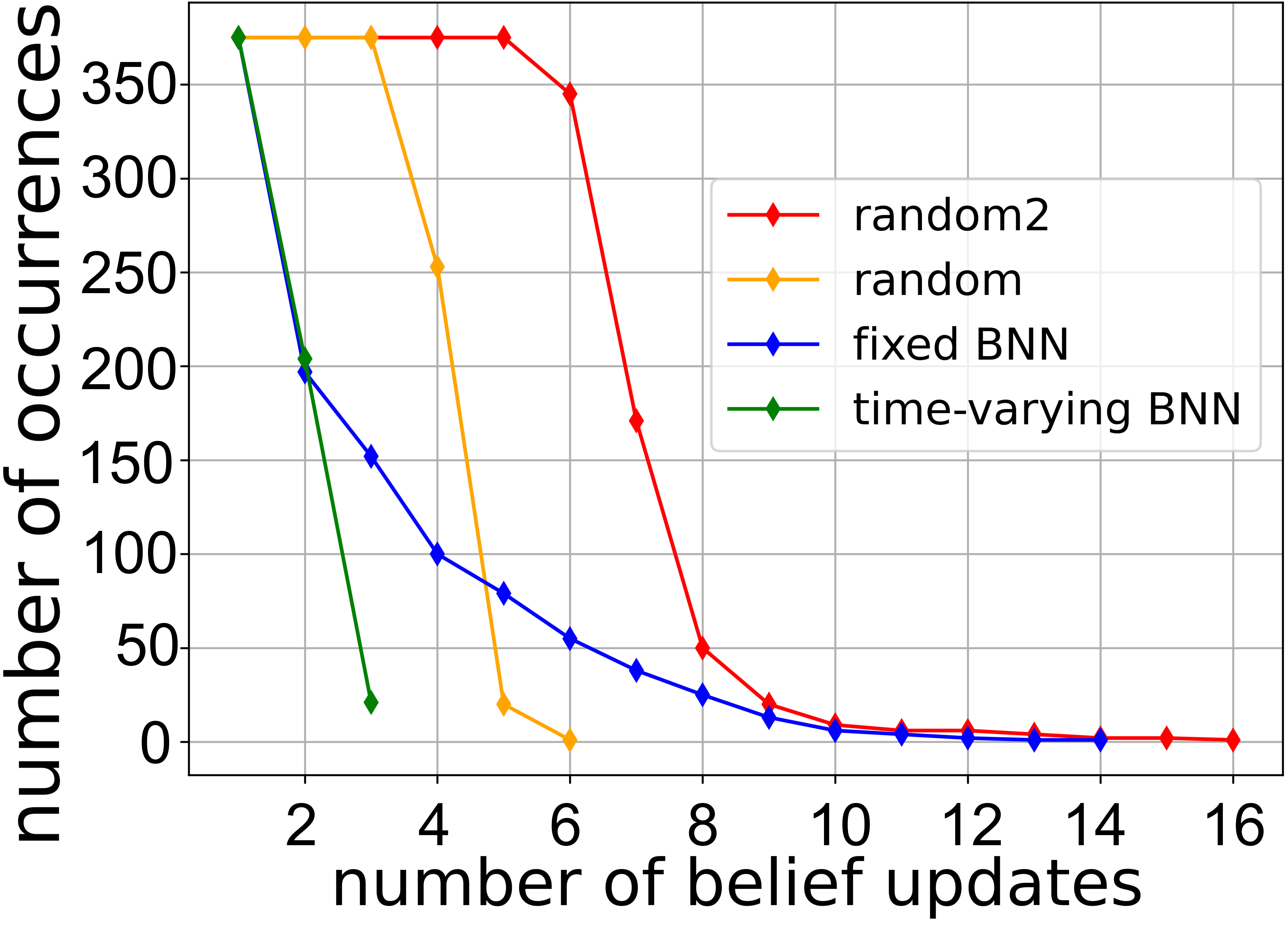}  
  \caption{\textcolor{blue}{
Average number of runs for the $k^{th}$ belief update.}}
  \label{fig:belief_num_time}
\end{subfigure}
\begin{subfigure}{.5\textwidth}
  \centering
  \includegraphics[width=\linewidth]{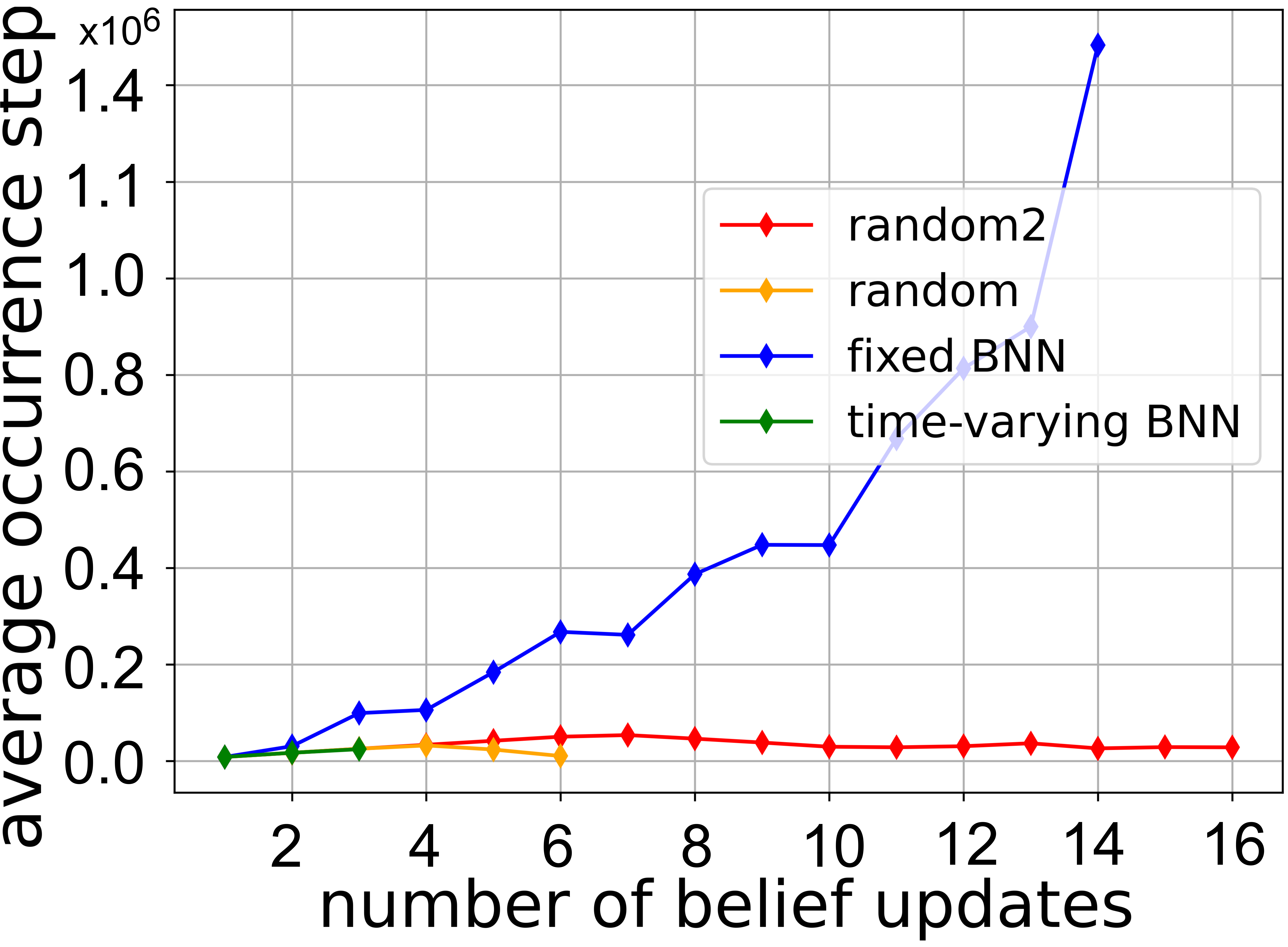}  
  \caption{\textcolor{blue}{Training step for the $k^{th}$ belief update, on average. }}
  \label{fig:belief_step}
\end{subfigure}
\begin{subfigure}{.5\textwidth}
  \centering
  \includegraphics[width=\linewidth]{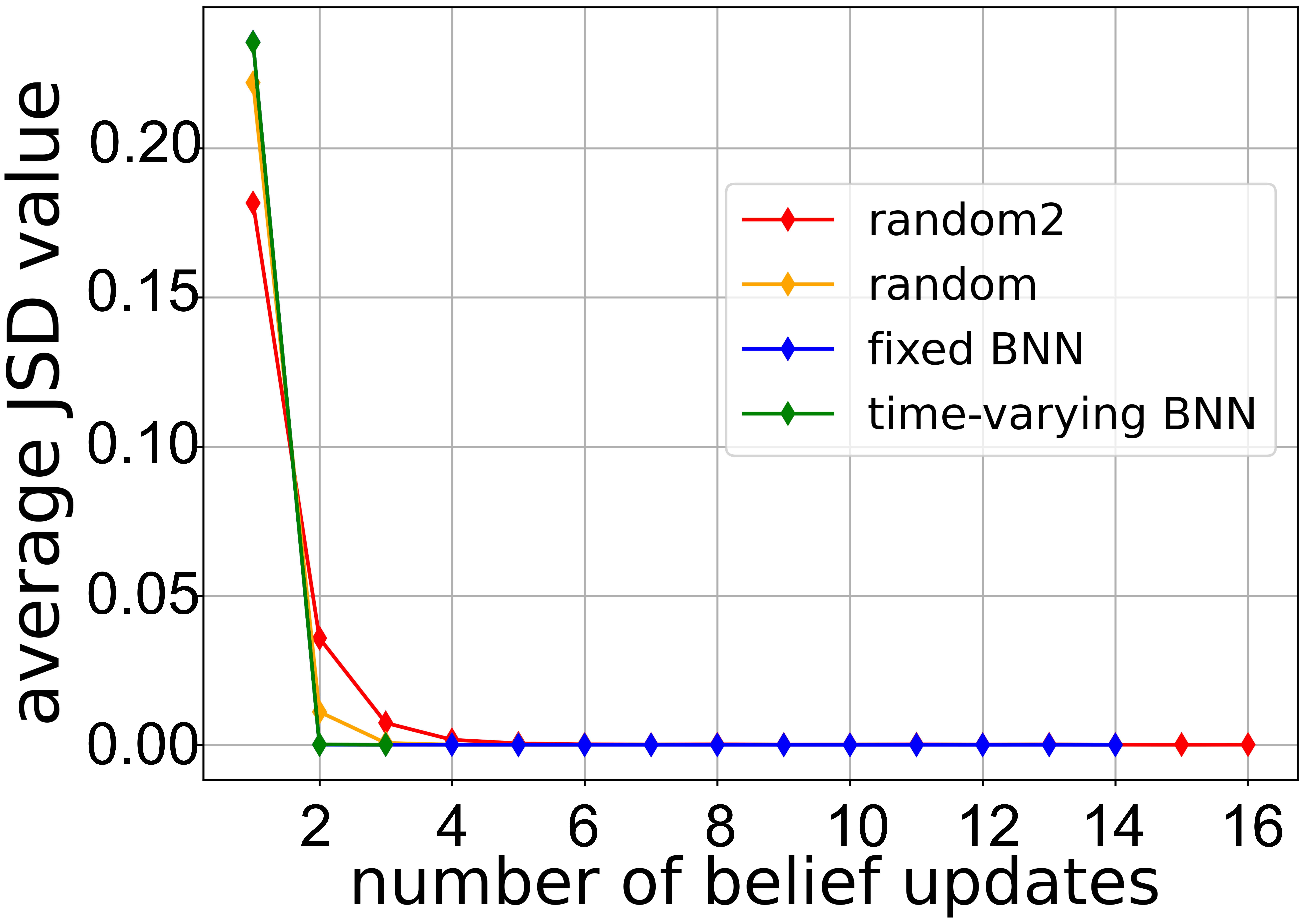}  
  \caption{\textcolor{blue}{Average value of the JSD that led to the $k^{th}$ update.}}
  \label{fig:belief_jsd}
\end{subfigure}
\caption{\textcolor{blue}{
(a): Attained rewards of 5 independent runs for every 100 training steps in 75 test indoor layouts for task $\phi_1$ using TvBNN: The darker region is bounded by $25^{th}$ and $75^{th}$ percentiles, while the lighter region is bounded by $10^{th}$ and $90^{th}$. The solid line represents the median.
(b)-(d): Progression of belief updates during training for task $\phi_1$.}}
\label{fig:rewards and belief floorplan 1}
\end{figure}
 }

\begin{table}[h!]
 \caption{Convergence results of the proposed algorithm under different observation models in 5 independent runs for 75 test indoor layouts in the first HouseExpo floorplan experiment.}
\centering
\begin{tabular}{c c c c c c} 
 \hline 
 Settings & TvBNN & FiBNN & Random & Random2 & No Update \\ 
 \hline\hline
 Q1 & 95,800 & 109,600 & 103,000 & 107,800 & 500,000 \\ 
 \hline
 Q2 & 44,400 & 48,900 & 40,100 & 45,700 & 189,900 \\
 \hline
 Q3 & 9,900 & 22,700 & 11,200 & 17,200 & 35,000\\
 \hline
 RS & 75/75 & 75/75 & 75/75 & 75/75 & 63/75\\
 \hline
 BU & 1.6 & 2.79 & 3.73 & 6.64 & -\\
 \hline
\end{tabular}
 \label{table:comparison}
\end{table}


To test the capabilities of our method in a more complicated scenario, we set a second task as $\phi_2 = $ ``Go to $\tt{kitchen}$ while avoiding the $\tt{bathroom}$, and then  go to the $\tt{bathroom}$ while avoiding the $\tt{bedroom}$", which is encoded by the reward machine shown in Fig. \ref{fig:automata_floorplan2}.

\begin{figure}
    \centering
    \begin{tikzpicture}[scale =0.6]
        \node[state, initial] (q0) {$v_0$};
        \node[state, below left of=q0] (q1) {$v_1$};
        \node[state, below right of=q0] (q2) {$v_2$};
        \node[state, accepting, below right of=q1] (q3) {$v_3$};
        \draw (q0) edge[loop above] node {\tt $(\neg \textup{kitchen} \land \neg \textup{bathroom},0)$} (q0);
        \draw (q0) edge node {\tt $(\textup{bathroom}, 0)$} (q2);
        \draw (q0) edge node [left] {\tt $(\textup{kitchen} \land \neg \textup{bathroom},0)$} (q1);
        \draw (q2) edge[loop right] node {\tt $(\textup{True}, 0)$} (q2);
        \draw (q1) edge[loop left] node {\tt $(\neg \textup{bathroom}, 0)$} (q1);
        \draw (q1) edge node {\tt $(\textup{bedroom}, 0)$} (q2); 
        \draw (q1) edge node [left]{\tt $(\textup{bathroom} \land \neg \textup{bedroom}, 1)$} (q3);
        \draw (q3) edge[loop above] node {\tt $(\textup{True}, 0)$} (q3);
\end{tikzpicture}
        \caption{Ground truth reward machine of the task $\phi_2$.}
        \label{fig:automata_floorplan2}
\end{figure}
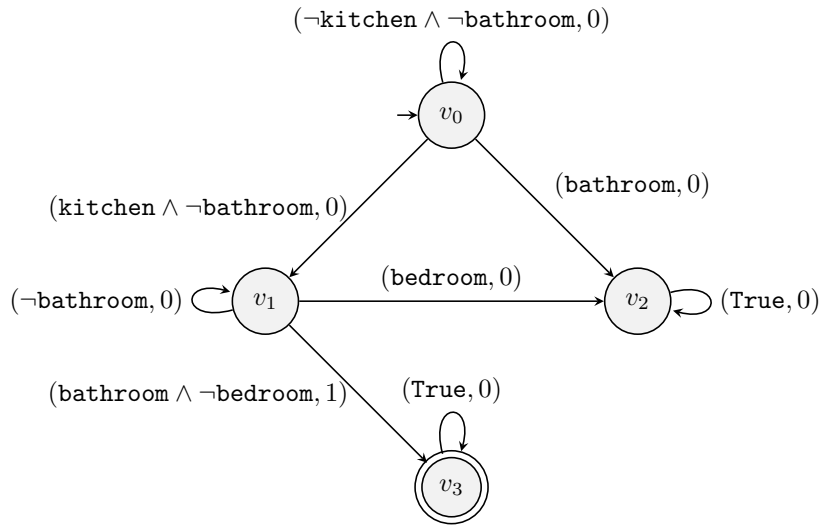


For this second task, we execute the proposed algorithm in 7 test floor plans with 5 different random seeds, resulting to a total of 35 runs.
Figure \ref{fig:rewards and belief floorplan 2} and Table \ref{table:comparison2} show the convergence results and belief updates, similar to the ones for $\phi_1$, for the first four settings; we omit the last one (no belief updates) since training fails to yield an inferred reward machine. We conclude that all observation models lead to accurate estimation of the ground-truth labelling function and hence to successful inference of the reward machine and learning of the optimal policy. 
According to Table \ref{table:comparison2}, the belief update patterns are similar to the ones observed for $\phi_1$. 
The average number of belief updates is smaller for the BNN observation settings compared to the random observation settings. The TvBNN setting requires the fewest number of belief updates on average, whereas the Random-2 setting requires the highest number of updates.   
Finally, one can still conclude that the proposed algorithm outperforms the nominal setting, where no belief updates are performed.   

\begin{figure}
\begin{subfigure}{.55\textwidth}
  \centering
  \includegraphics[width=\linewidth]{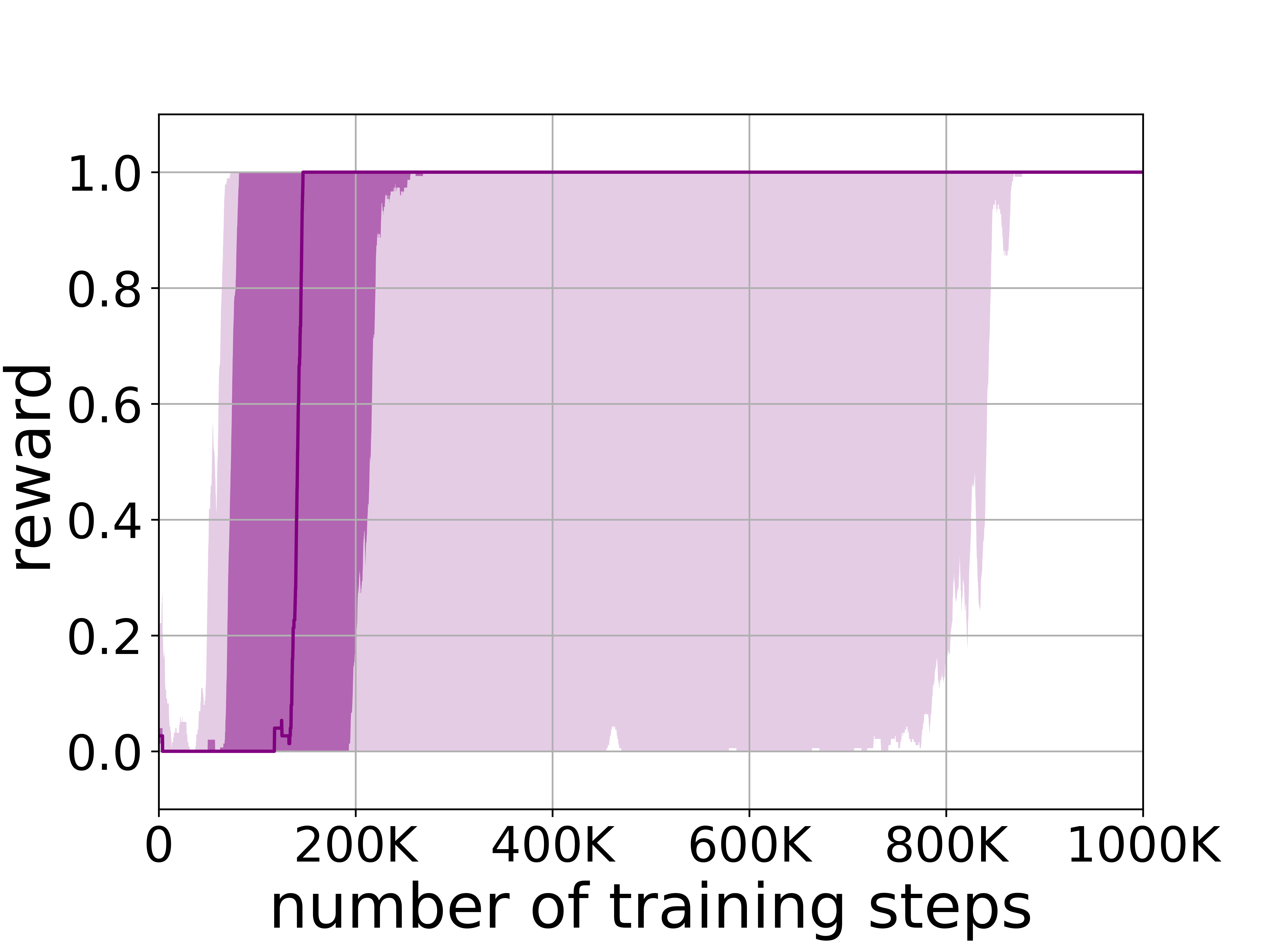}  
  \caption{\textcolor{blue}{
Time-varying BNN}}
  \label{fig:reward2_time_varying_BNN}
\end{subfigure}
\begin{subfigure}{.5\textwidth}
  \centering
  \includegraphics[width=\linewidth]{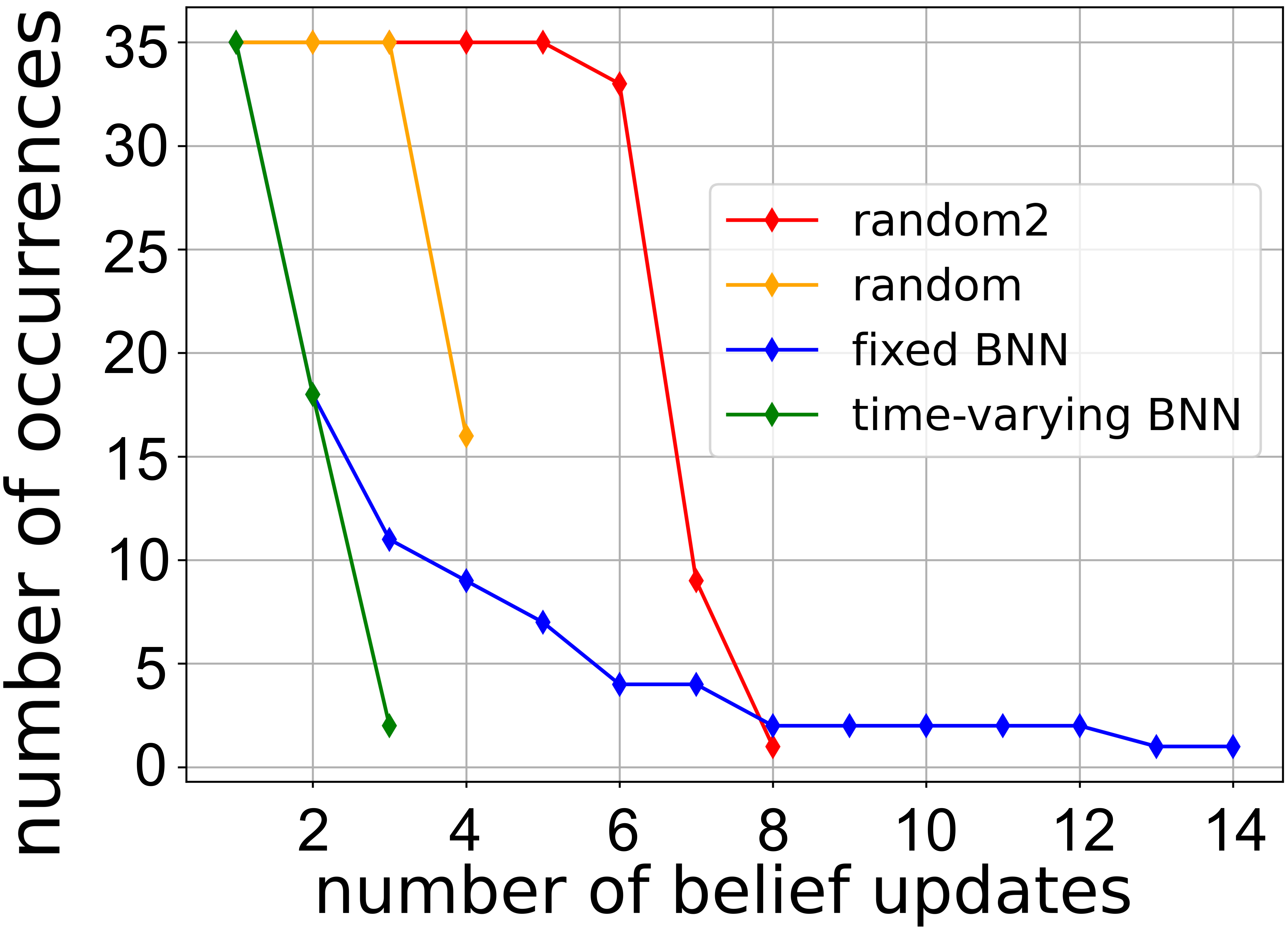}  
  \caption{\textcolor{blue}{
Average number of runs for the $k^{th}$ belief update.}}
  \label{fig:belief2_num_time}
\end{subfigure}
\begin{subfigure}{.5\textwidth}
  \centering
  \includegraphics[width=\linewidth]{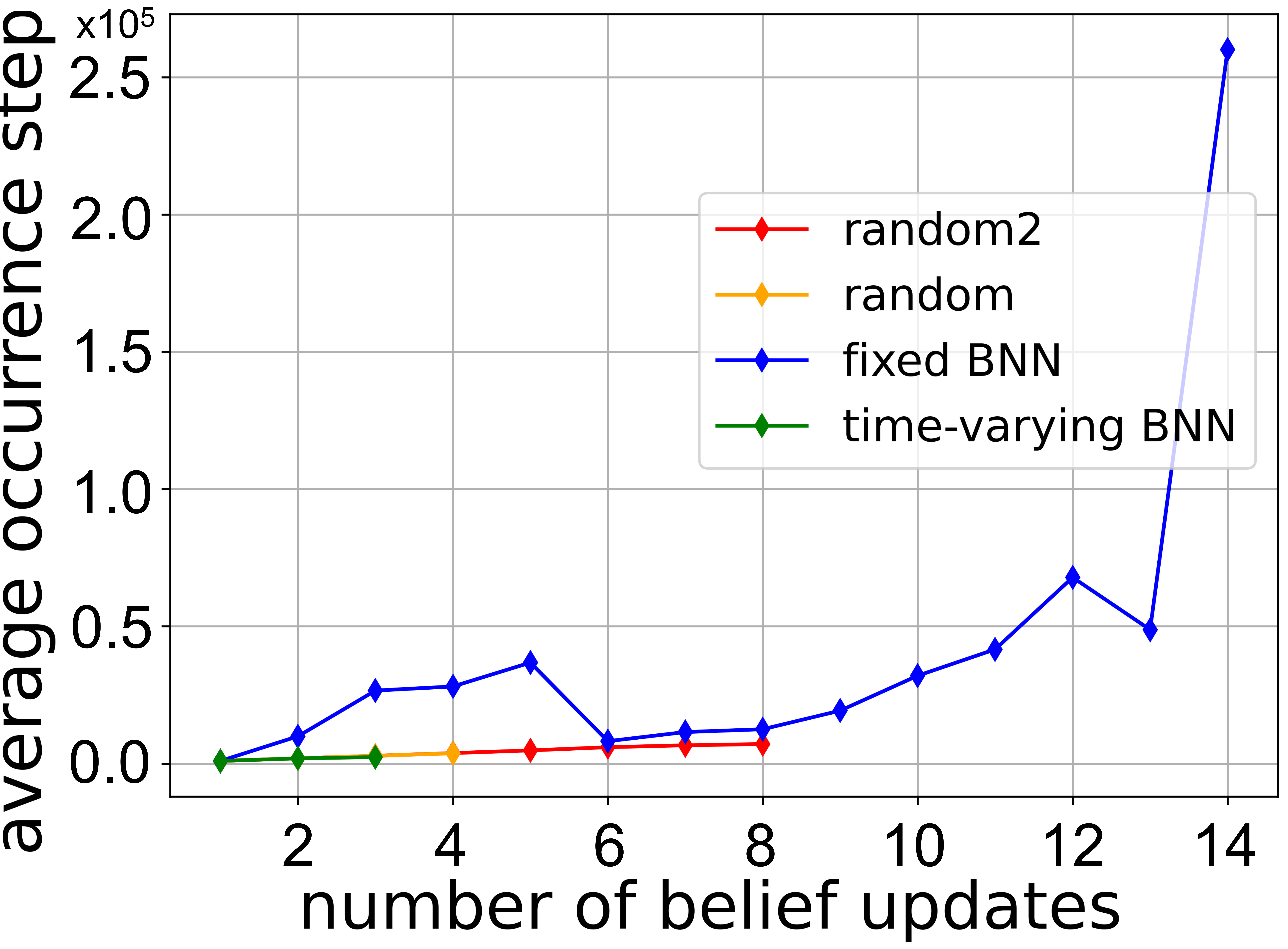}  
  \caption{\textcolor{blue}{
Average number of runs for the $k^{th}$ belief update.}}
  \label{fig:belief2_step}
\end{subfigure}
\begin{subfigure}{.5\textwidth}
  \centering
  \includegraphics[width=\linewidth]{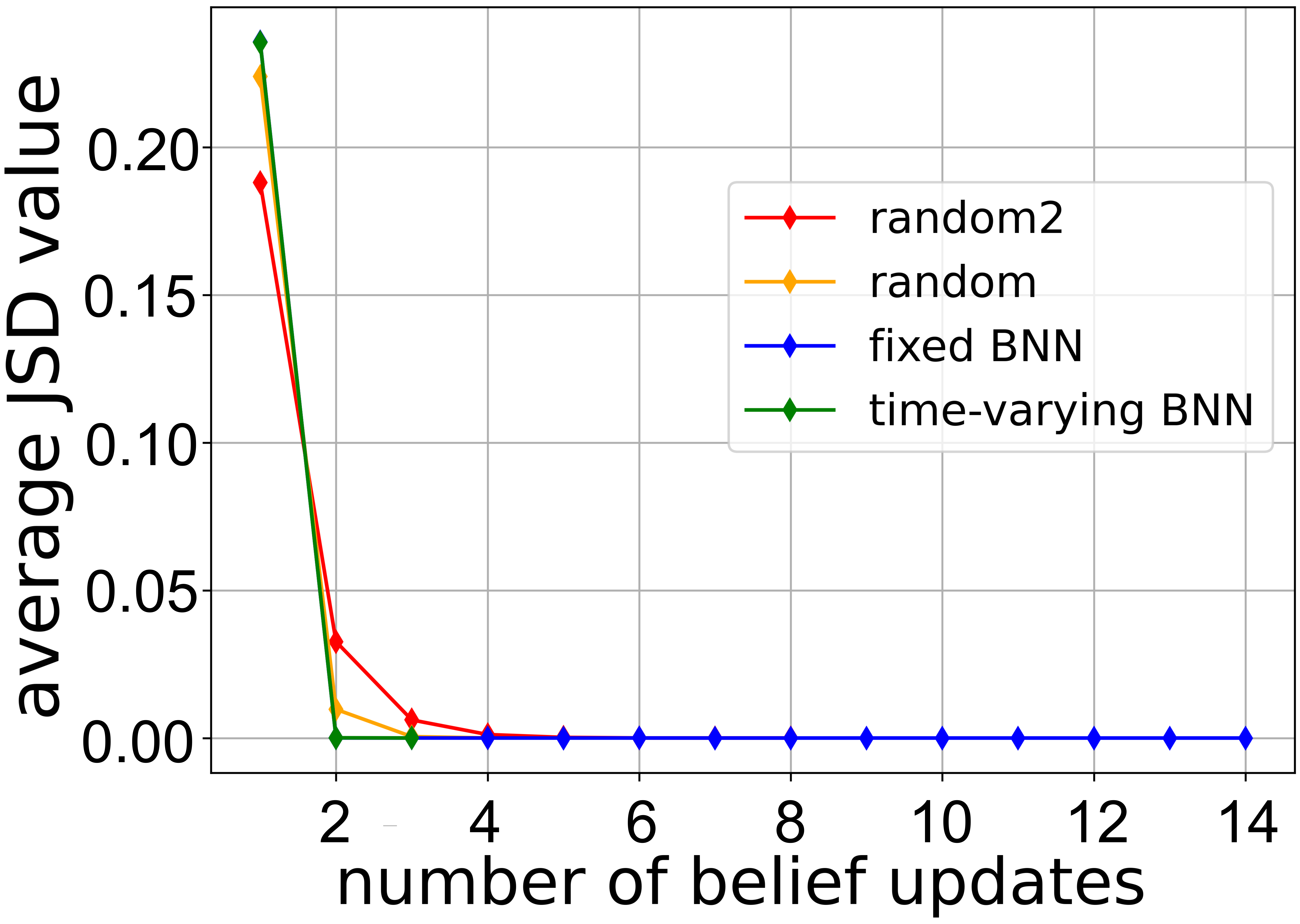}  
  \caption{\textcolor{blue}{
Average value of the JSD that led to the $k^{th}$ update.}}
  \label{fig:belief2_jsd}
\end{subfigure}
\caption{\textcolor{blue}{
(a): Attained rewards of 5 independent runs for every 100 training steps in 7 test indoor layouts for task $\phi_2$ using TvBNN: The darker region is bounded by $25^{th}$ and $75^{th}$ percentiles, while the lighter region is bounded by $10^{th}$ and $90^{th}$. The solid line represents the median.
(b)-(d): Progression of belief update during training for task $\phi_2$.}}
\label{fig:rewards and belief floorplan 2}
\end{figure}

\begin{table}[h!]
 \caption{Convergence results of the proposed algorithm under different observation models in 5 independent runs for 7 test indoor layouts in the second HouseExpo floorplan experiment. }
\centering
\begin{tabular}{r c c c c c} 
 \hline 
 Settings & TvBNN & FiBNN & Random & Random2 \\ 
 \hline\hline
Q1 & 255,500 & 263,100 & 341,100 & 231,100 \\ 
 \hline
Q2 & 146,700 & 141,800 & 137,100 & 122,200 \\
 \hline
Q3 & 81,500 & 80,000 & 74,200 & 91,100 \\
 \hline
 BU & 1.57 & 2.85 & 3.45 & 6.22\\	
 \hline
\end{tabular}
\label{table:comparison2}
\end{table}


 \textcolor{blue}{
\subsection{Office world} \label{sec:exps office world}
}

 \textcolor{blue}{
We next test the proposed algorithm in the office-world environment shown in Fig. \ref{fig:office_env}. 
The environment is a $12\times9$ grid world with four atomic propositions $\mathcal{AP} = \{a,b,c,d\}$ corresponding to coffee $a$, mail $b$, obstacle $c$, and office $d$. The task assigned to the agent is to get coffee, then mail, and eventually arrive at the office without encountering the obstacle. The reward machine that encodes the aforementioned task is given in Figure \ref{fig:office_RM}. In the experiments, an episode terminates when the agent reaches the accepting state or it exceeds the maximum number of 2000 steps per episode. The maximum number of training steps is 1,500,000. We set the Jensen-Shannon divergence threshold $\gamma_d$  to $\gamma_d = 10^{-5}$ and Algorithm \ref{algo:QRM_mod} selects a random action with probability $\varepsilon_a = 0.3$. 
}

\begin{figure}
\begin{subfigure}{.5\textwidth}
  \centering
  \includegraphics[width=\linewidth]{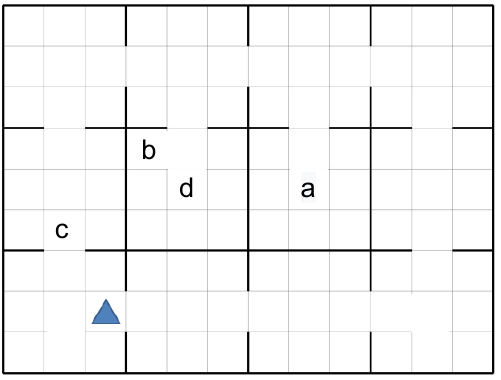}
  \caption{\textcolor{blue}{
Office world with four atomic propositions: coffee $a$, mail $b$, obstacle $c$, and office $d$. The blue triangle indicates the initial state of an episode.}}
  \label{fig:office_env}
\end{subfigure}
\begin{subfigure}{.5\textwidth}
\centering
\begin{tikzpicture}[node distance=1.7cm]
   \node[state,initial] (q_0)   {$v_0$};
   \node[state] (q_1) [right=of q_0] {$v_1$};
   \node[state] (q_2) [below=of q_0] {$v_2$};
   \node[state] (q_3) [right=of q_1] {$v_3$};
   \node[state, accepting] (q_4) [below=of q_3] {$v_4$};
    \path[->]
    (q_0) edge [loop above]       node [swap] {$(\lnot a \land \lnot c,  0)$} (q_0)
          edge              node        {$(a,                      0)$} (q_1)
          edge              node        {$(c,                      0)$} (q_2)
          
    (q_1) edge [loop above]       node [swap] {$(\lnot b \land \lnot c,  0)$} (q_1)
          edge              node  {$(b,                            0)$} (q_3)
          edge              node  {$(c,                            0)$} (q_2)
          
    (q_2) edge [loop below]       node [swap] {$(\top,    0)$} (q_2)

    (q_3) edge [loop above]       node [swap] {$(\lnot c \land \lnot d,  0)$} (q_3)
          edge                    node        {$(d,                      0)$} (q_4)
          edge [bend left]        node [above]       {$(c \land \lnot d,        0)$} (q_2);
\end{tikzpicture}
\caption{
 \textcolor{blue}{
   Reward machine for the office-world environment. The task is to fetch coffee $a$, then mail $b$, and eventually arrive at the office $d$ without encountering the obstacle $c$.}
}
\label{fig:office_RM}
\end{subfigure}
\caption{\textcolor{blue}{
Office-world environment}}
\label{fig:office}
\end{figure}
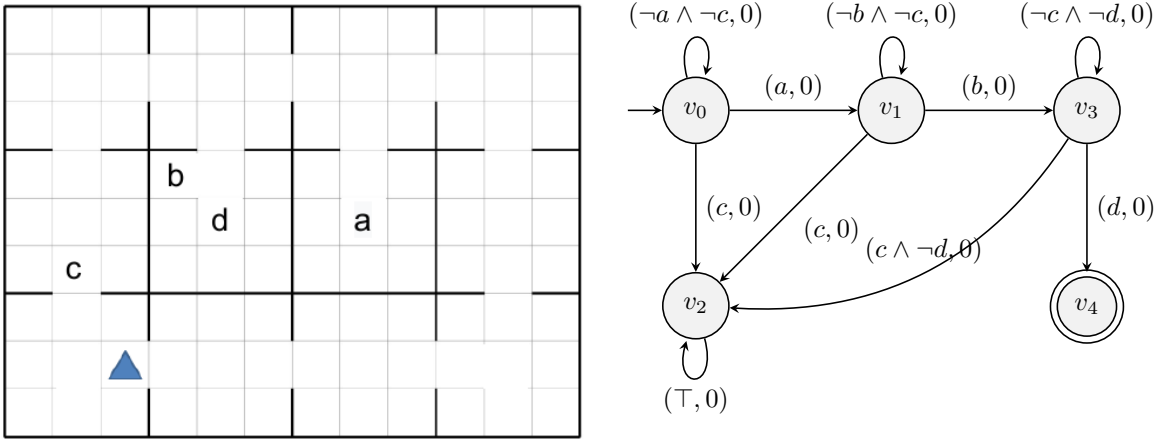

\textcolor{blue}{
We consider the following experimental settings, where we compare different observation models as well as with standard learning-algorithms of the related literature. In what follows, we use the acronym JPL (Joint Perception and Learning) for the proposed algorithm.
}

\textcolor{blue}{
\begin{enumerate}
    \item True Observation model (JPL-true): The proposed algorithm uses an accurate and unambiguous observation model that gives the ground-truth probability of observing an atomic proposition at a state, i.e., $\mathcal{O}(s_1,s_2,p,\mathsf{True})= 1$  and $\mathcal{O}(s_1,s_2,p,\mathsf{False})= 0$. 
        \item False Observation model (JPL-false): The proposed algorithm uses an inaccurate and unambiguous observation model that gives false probabilities of observing an atomic proposition at a state, i.e., $\mathcal{O}(s_1,s_2,p,\mathsf{True})= 0$  and $\mathcal{O}(s_1,s_2,p,\mathsf{False})= 1$. 
    \item Random Observation model (JPL-random): The algorithm uses an observation model whose observation probabilities are uniformly sampled from the interval $[0.1,0.9]$.
    \item Random Observation model-2 (JPL-random2): The algorithm uses an observation model whose observation probabilities are uniformly sampled from the interval $[0.4,0.6]$.
    \item Q-learning: We test the traditional Q-learning algorithm \cite{watkins1992q} with a randomly initialized labelling-function belief, without updating it or inferring a reward machine. Our implementation extends the original algorithm by providing a vector of binary values indicating whether each atomic proposition is observed by the agent (1) or not (0) according to its belief function.
    \item DDQN: We test the double deep Q-learning (DDQN) algorithm \cite{van2016deep}, with a randomly initialized labelling-function belief, without updating it or inferring a reward machine. Our implementation augments the state space with the sequence of the last 200 labels observed according to the agent's belief function. We use a neural network of 6 hidden layers, each consisting of 64 neurons.
\end{enumerate}
}

\textcolor{blue}{
In all aforementioned cases, we simulate the observations to be consistent with the respective observation model, i.e., observe that $\mathcal{Z}(s_1,s_2,p) = \mathsf{True}$ and $\mathcal{Z}(s_1,s_2,p) = \mathsf{False}$ with probabilities $\mathcal{O}(s_1,s_2,p,\mathsf{True})$ and $1-\mathcal{O}(s_1,s_2,p,\mathsf{True})$, respectively, when $s_2 \models p$, and with probabilities $\mathcal{O}(s_1,s_2,p,\mathsf{False})$ and $1-\mathcal{O}(s_1,s_2,p,\mathsf{False})$, respectively, when $s_2 \not \models p$.}

\textcolor{blue}{
We run the aforementioned cases in the office environment for 10 independent runs, each corresponding to a different random seed.
In each one of the aforementioned runs, we perform an evaluation episode every 100 training steps and register the obtained rewards.
The results are shown in Fig. \ref{fig:rewards and belief office world} and Table \ref{table:comparison_office}. In particular, Fig. \ref{fig:office_rewards_random}
shows the progression of rewards collected in test episodes for the random observation model (JPL-random) as well as for the Q-learning and DDQN algorithms. 
One concludes that the proposed algorithm successfully learns the optimal policy, whereas 
Q-learning and DDQN fail to learn a policy that accomplishes the task; such a result demonstrates that simply augmenting the state space with observed labels is not sufficient neither with a Q-table nor with a Q-function approximated by a neural network. 
We omit the reward-progression plots for the other cases of observation models, since they yield similar convergence results with the first random observation model. 
Table \ref{table:comparison_office} shows the training steps required for convergence for the $25^{th}$ (Q1), $50^{th}$ (Q2), and $75^{th}$ (Q3) percentiles of the attained rewards for the aforementioned cases, as in Sec. \ref{sec:exps floorplan}. 
 The table also depicts the average number of belief updates (BU) of $\hat{\mathcal{L}}_h$.}
 
\textcolor{blue}{
Figures \ref{fig:office_belief_num_time}-\ref{fig:office_belief_jsd} further show how the information-processing mechanism of the proposed algorithm handles the belief updates. On average, JPL-true and JPL-false obtain a belief function that satisfies the JSD constraint, $\gamma_d=10^{-5}$, in one episode, hence Figs. \ref{fig:office_belief_num_time}, \ref{fig:office_belief_step}, and \ref{fig:office_belief_jsd} have only one data point for these algorithms; JPL-random2 satisfies the JSD threshold in 6 episodes, 3 more than what JPL-random, since the observation model is more ambiguous; i.e., it carries less information due to a narrow range of values used to draw probabilities ($[0.4,0.6]$ vs $[0.1,0.9]$).
}

\textcolor{blue}{
\begin{figure}
\begin{subfigure}{.55\textwidth}
  \centering
  \includegraphics[width=\linewidth]{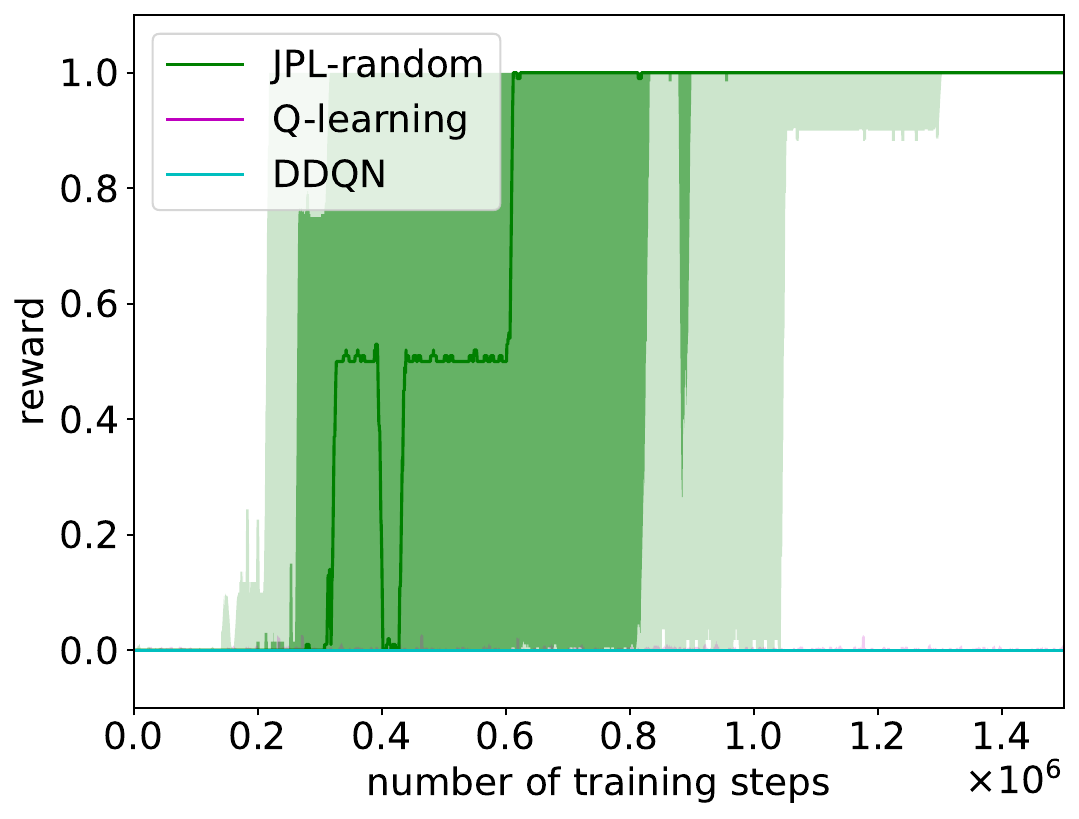}  
  \caption{\textcolor{blue}{
Progression of reward using (i) the proposed algorithm with a random observation model, (ii) a Q-learning algorithm, and a (iii) DDQN algorithm.}}
  \label{fig:office_rewards_random}
\end{subfigure}
\begin{subfigure}{.5\textwidth}
  \centering
  \includegraphics[width=\linewidth]{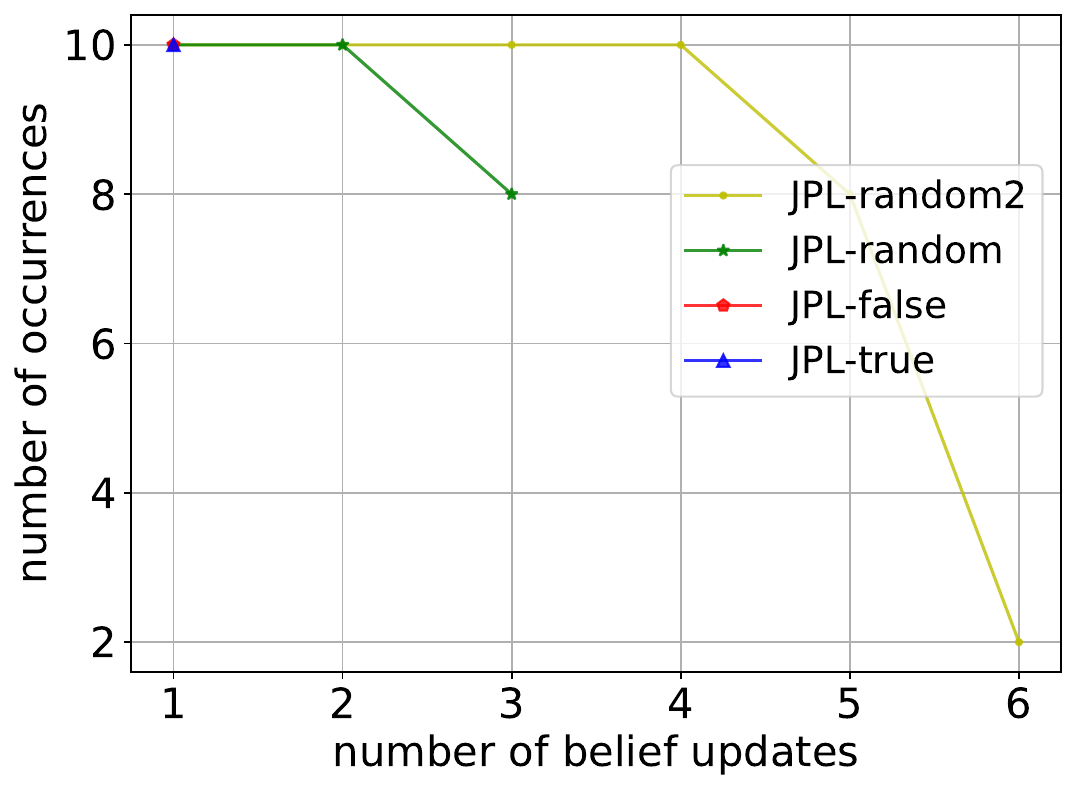}  
  \caption{\textcolor{blue}{
Average number of runs for the $k^{th}$ belief update.}}
  \label{fig:office_belief_num_time}
\end{subfigure}
\begin{subfigure}{.5\textwidth}
  \centering
  \includegraphics[width=\linewidth]{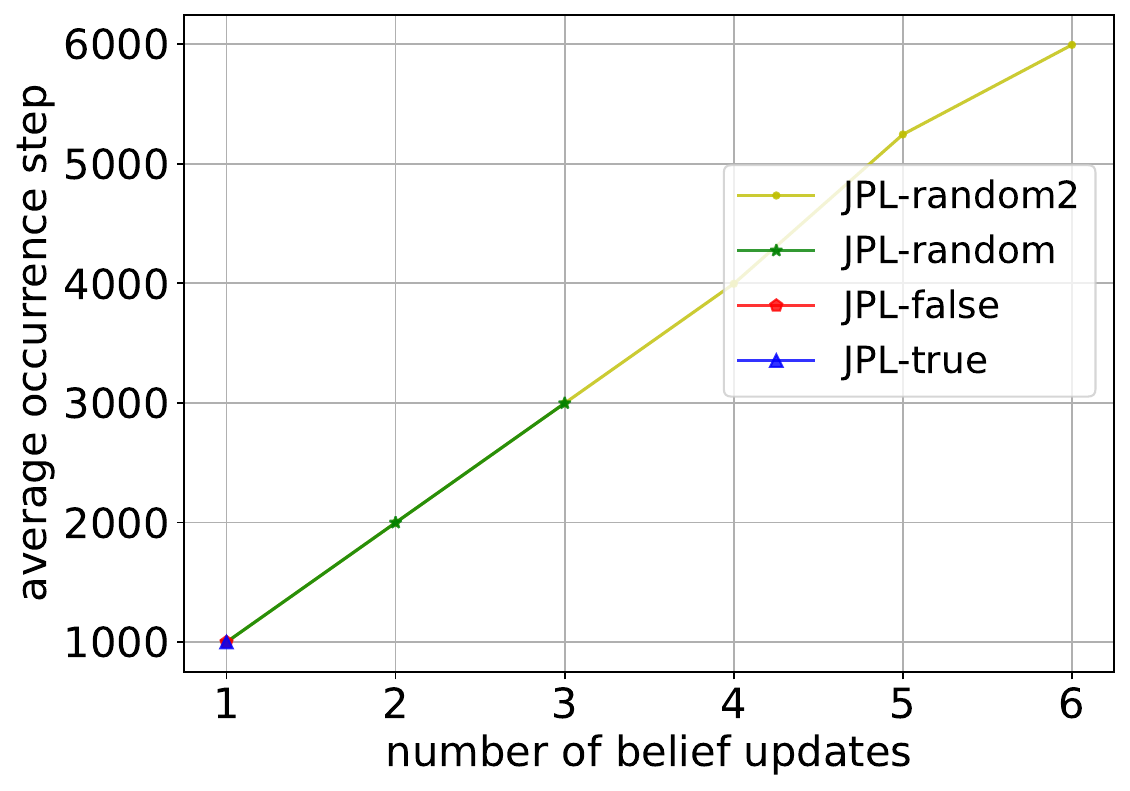}  
  \caption{\textcolor{blue}{
Average number of runs for the $k^{th}$ belief update.}}
  \label{fig:office_belief_step}
\end{subfigure}
\begin{subfigure}{.5\textwidth}
  \centering
  \includegraphics[width=\linewidth]{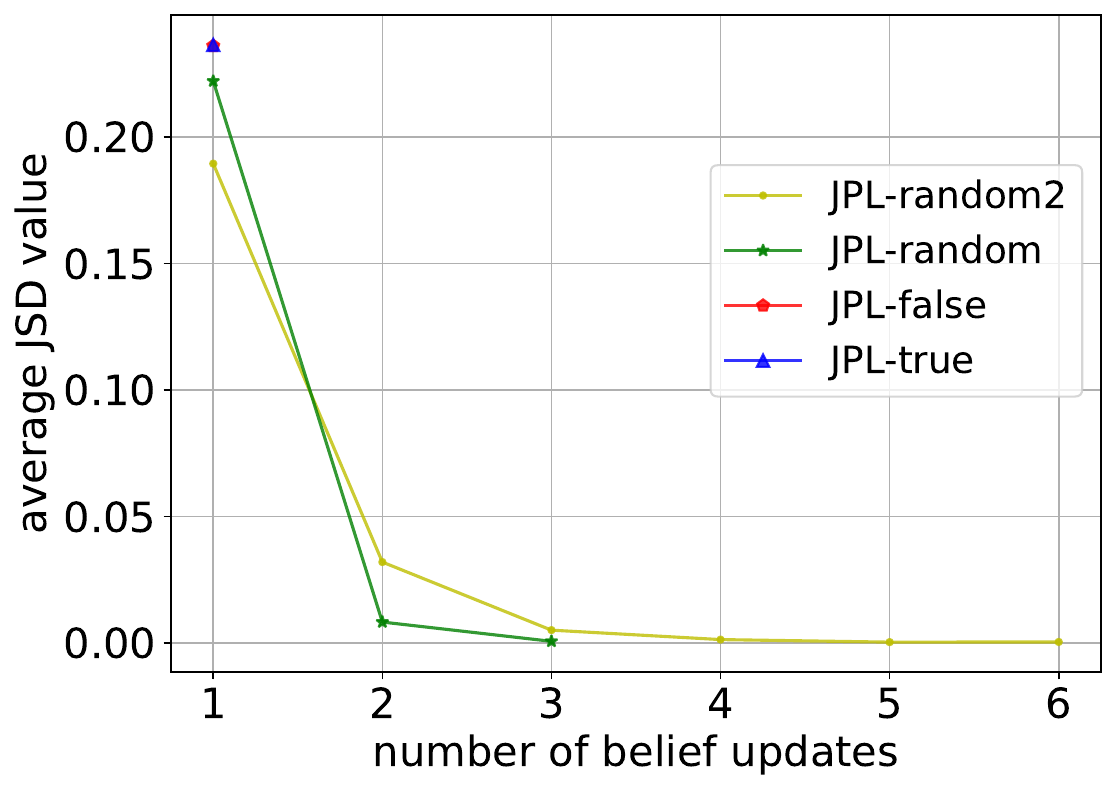}  
  \caption{\textcolor{blue}{
Average value of the JSD that led to the $k^{th}$ update.}}
  \label{fig:office_belief_jsd}
\end{subfigure}
\caption{\textcolor{blue}{
(a): Attained rewards of 10 independent runs for every 100 training steps for the office-world task using the first random observation model: The darker region is bounded by $25^{th}$ and $75^{th}$ percentiles, while the lighter region is bounded by $10^{th}$ and $90^{th}$. The solid line represents the median.
(b)-(d): Progression of belief update during training for the office-world task.}}
\label{fig:rewards and belief office world}
\end{figure}
}

\begin{table}[h!]
\caption{Convergence results of the proposed algorithm under different observation models in 10 independent runs in the office-world environment.}
\centering
\begin{tabular}{r c c c c c} 
 \hline 
 Setting & JPL-true & JPL-false & JPL-random & JPL-random2 \\ 
 \hline\hline
   Q1 & 1,029,000 & 1,019,800 & 958,500 & 1,003,700 \\ 
 \hline
   Q2 & 776,000 & 776,000 & 818,300 & 737,600 \\
 \hline
 Q3 & 412,100 & 423,300 & 315,600 & 412,100 \\
 \hline
 BU & 1 & 1 & 2.8 & 5\\
 \hline
\end{tabular}
 \label{table:comparison_office}
\end{table}

\textcolor{blue}{
\subsection{Craft-world Environment} \label{sec:exps Craftworld}
}

\textcolor{blue}{
Finally, we test the proposed algorithm in the craft-world environment shown in Fig. \ref{fig:craft_env}. 
The environment is a $12\times9$ grid world with five atomic propositions $\mathcal{AP} = \{w,i,t,h,f\}$ corresponding to wood $w$, iron $i$, toolshed $t$, workbench $h$, and factory $f$. The task assigned to the agent is to build stairs by following the instructions below:
\begin{enumerate}
    \item Get wood $w$ (toolshed $t$ cannot be used before getting wood),
    \item Use toolshed $t$ (using workbench $h$ before toolshed resets again),
    \item Use workbench $h$,
    \item Get iron $i$ (using factory $f$ before collecting iron resets task), and
    \item Use factory $f$.
\end{enumerate} 
The reward machine for such a task is shown in Fig. \ref{fig:craft_RM}.
In the experiments, an episode terminates when the agent reaches the accepting state or it exceeds the maximum number of 400 steps per episode. The maximum number of training steps is 2,000,000. We set the Jensen-Shannon divergence threshold $\gamma_d$  to $\gamma_d = 10^{-5}$ and Algorithm \ref{algo:QRM_mod} selects a random action with probability $\varepsilon_a = 0.3$. 
We evaluate the same algorithmic observation-model cases as in the office-world case of Sec. \ref{sec:exps office world}. 
}

\begin{figure}
\begin{subfigure}{.5\textwidth}
  \centering
  \includegraphics[width=\linewidth]{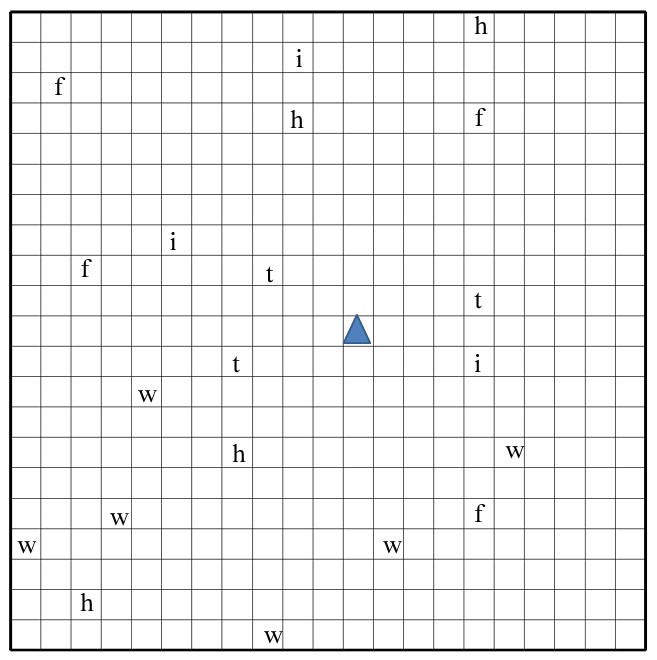}
  \caption{\textcolor{blue}{Craft world with five atomic propositions: wood $w$, iron $i$, toolshed $t$, workbench $h$, and factory $f$. The blue triangle indicates the initial state of an episode.}}
  \label{fig:craft_env}
\end{subfigure}
\begin{subfigure}{.5\textwidth}
  \centering
  \begin{tikzpicture}[node distance=1.7cm]
   \node[state,initial] (q_0)   {$v_0$};
   \node[state] (q_1) [right=of q_0] {$v_1$};
   \node[state] (q_2) [right=of q_1] {$v_2$};
   \node[state] (q_3) [below=of q_2] {$v_3$};
   \node[state] (q_6) [below=of q_0] {$v_6$};
  \node[state,accepting] (q_5) [below=of q_6] {$v_5$};
  \node[state] (q_4) [right=of q_5] {$v_4$};
    \path[->]
    (q_0) edge [loop above]       node [swap] {$(\lnot w \land \lnot t,  0)$} (q_0)
          edge [bend right] node [below]       {$(w \land \lnot t,        0)$} (q_1)
          edge              node [left]       {$(t,                      0)$} (q_6)
          
    (q_1) edge [loop above]       node [swap] {$(\lnot h \land \lnot t,  0)$} (q_1)
          edge              node  {$(t \land \lnot h,              0)$} (q_2)
          edge [bend right] node [above] {$(h,                            0)$} (q_0)
          
    (q_2) edge [loop above]       node [swap] {$(\lnot h, 0)$} (q_2)
          edge              node        {$(t,       0)$} (q_3)
          
    (q_3) edge [loop below]       node [swap]  {$(\lnot i \land \lnot f,  0)$} (q_3)
          edge                    node [left]        {$(I \land \lnot f,        0)$} (q_4)
          edge [bend left]        node [above] {$(f,                      0)$} (q_0)

    (q_4) edge [loop right]       node [swap] {$(\lnot f, 1)$} (q_4)
          edge                    node        {$(f,       0)$} (q_5)

    (q_6) edge [loop right]       node [swap] {$(\top,    0)$} (q_6);
\end{tikzpicture}
\caption{\textcolor{blue}{
    Reward machine for the craft-world environment. The task is to build stairs by getting wood $w$ (toolshed $t$ cannot be used before getting wood), then using toolshed $t$ (using workbench $h$ before toolshed resets) and workbench $h$ consecutively. Finally, the agent should obtain iron $i$ (using factory $f$ before collecting iron resets task) and process all items in the factory $f$.}
}
\label{fig:craft_RM}
\end{subfigure}
\caption{\textcolor{blue}{
Craft-world environment.}}
\label{fig:craft}
\end{figure}

\textcolor{blue}{The results for 10 independent runs, each corresponding to a different random seed, are shown in Fig. \ref{fig:rewards and belief craft world}.
Figure \ref{fig:craft_rewards_random} show the progression of rewards obtained in test episodes for the first random observation model of the proposed algorithm as well as the Q-learning and DDQN algorithms. We omit the reward-progression plots for other observation models since they achieve similar speed of convergence.
We observe that, since the craft-world scenario involves a more complicated task and larger state space than the office-world scenario, the proposed algorithm requires more than 1.2 million training steps, in median values, to converge. Additionally, as in the office-world case, the Q-learning and DDQN algorithms fail to learn the policy that achieves the task. 
Furthermore, Table  \ref{table:comparison_craft} shows the training steps required for convergence for the $25^{th}$ (Q1), $50^{th}$ (Q2), and $75^{th}$ (Q3) percentiles of the attained rewards for the aforementioned cases, as in the previous cases. 
 The table also depicts the average number of belief updates (BU) of $\hat{\mathcal{L}}_h$.
 Finally, Fig. \ref{fig:craft_belief_num_time}-\ref{fig:craft_belief_jsd} show how the information-processing mechanism of the proposed algorithm handles the belief updates, where we observe a behaviour similar to the previous cases in Sec. \ref{sec:exps floorplan} and \ref{sec:exps office world}.
 }
 

\begin{figure}
\begin{subfigure}{.55\textwidth}
  \centering
  \includegraphics[width=\linewidth]{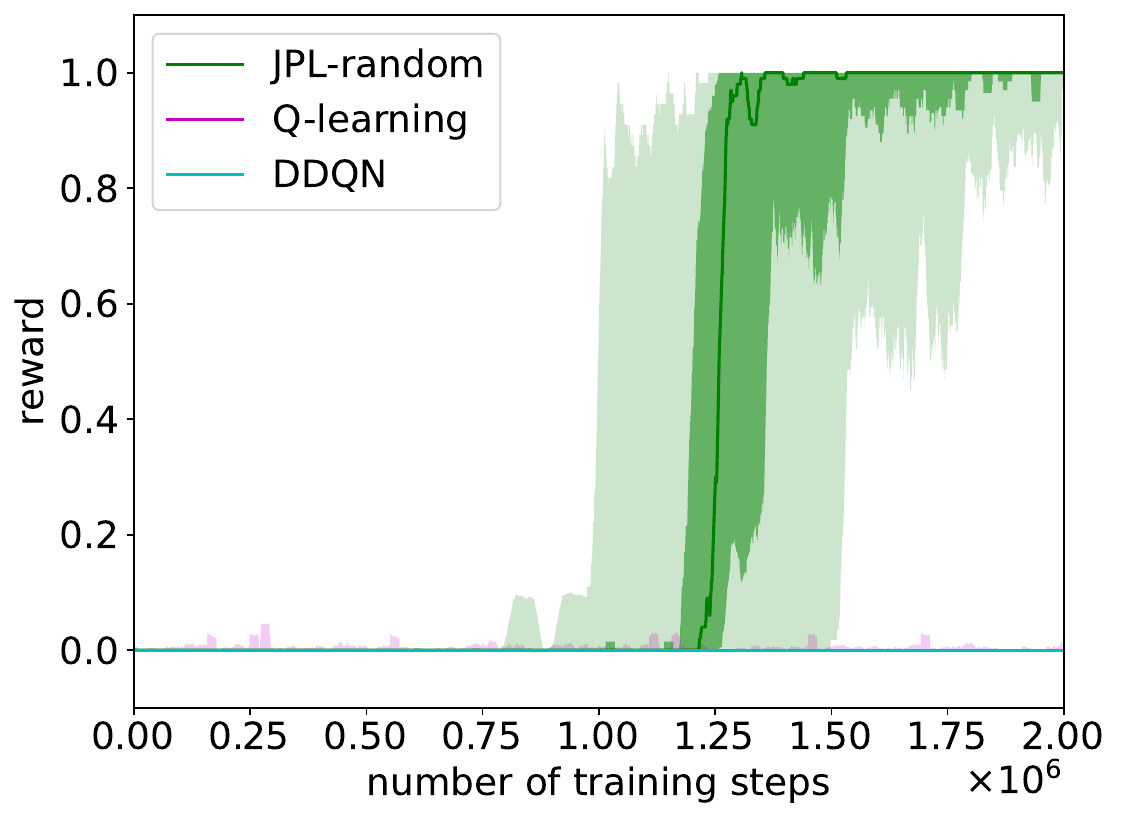}  
  \caption{\textcolor{blue}{Progression of reward using (i) the proposed algorithm with a random observation model, (ii) a Q-learning algorithm, and a (iii) DDQN algorithm.}}
  \label{fig:craft_rewards_random}
\end{subfigure}
\begin{subfigure}{.5\textwidth}
  \centering
  \includegraphics[width=\linewidth]{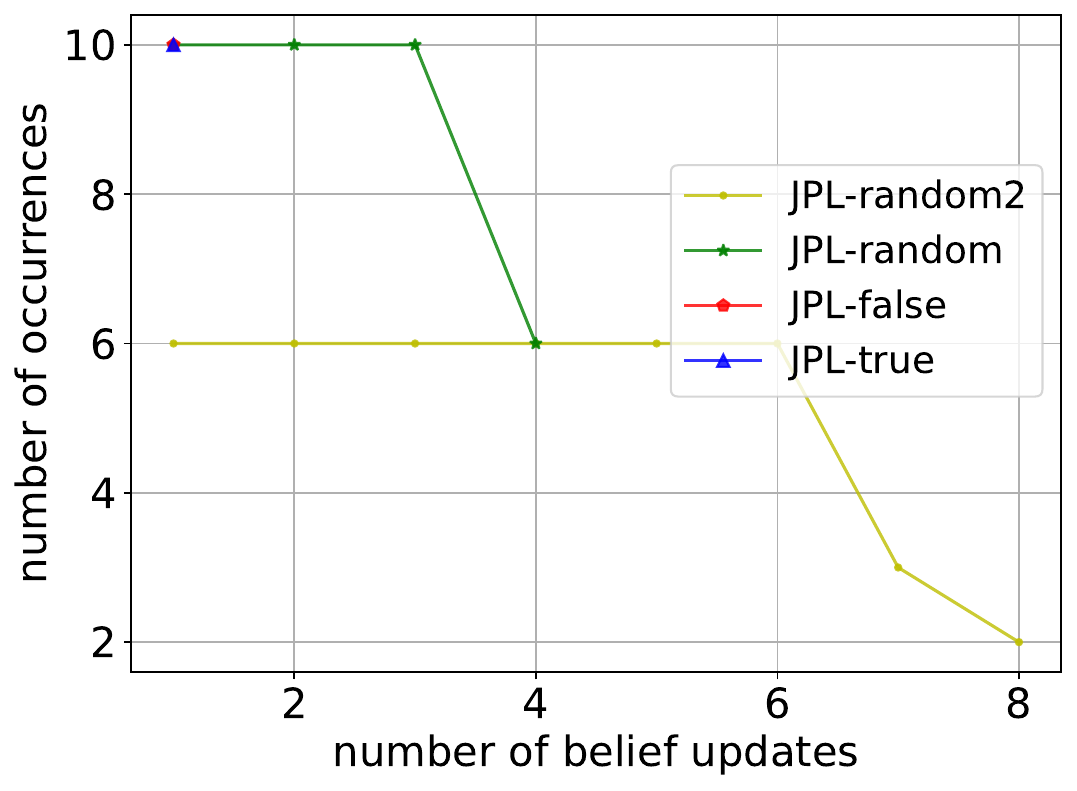}  
  \caption{\textcolor{blue}{
Average number of runs for the $k^{th}$ belief update.}}
  \label{fig:craft_belief_num_time}
\end{subfigure}
\begin{subfigure}{.5\textwidth}
  \centering
  \includegraphics[width=\linewidth]{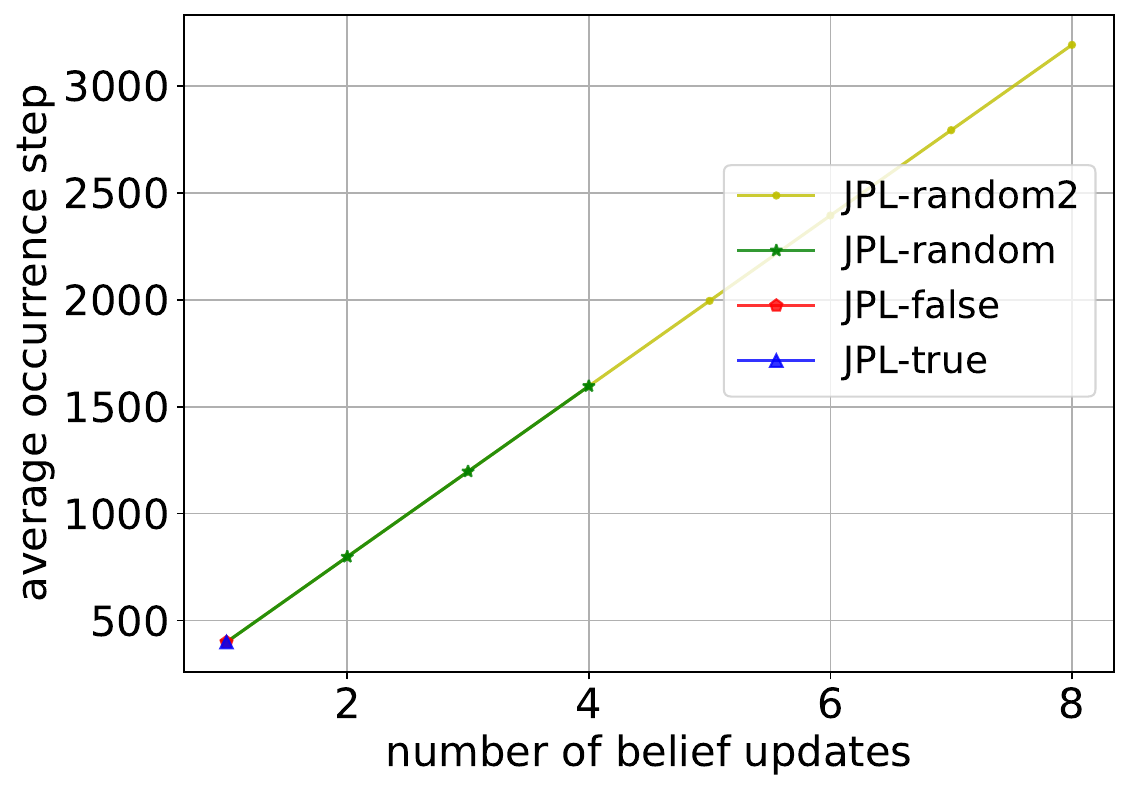}  
  \caption{\textcolor{blue}{
Average number of runs for the $k^{th}$ belief update.}}
  \label{fig:craft_belief_step}
\end{subfigure}
\begin{subfigure}{.5\textwidth}
  \centering
  \includegraphics[width=\linewidth]{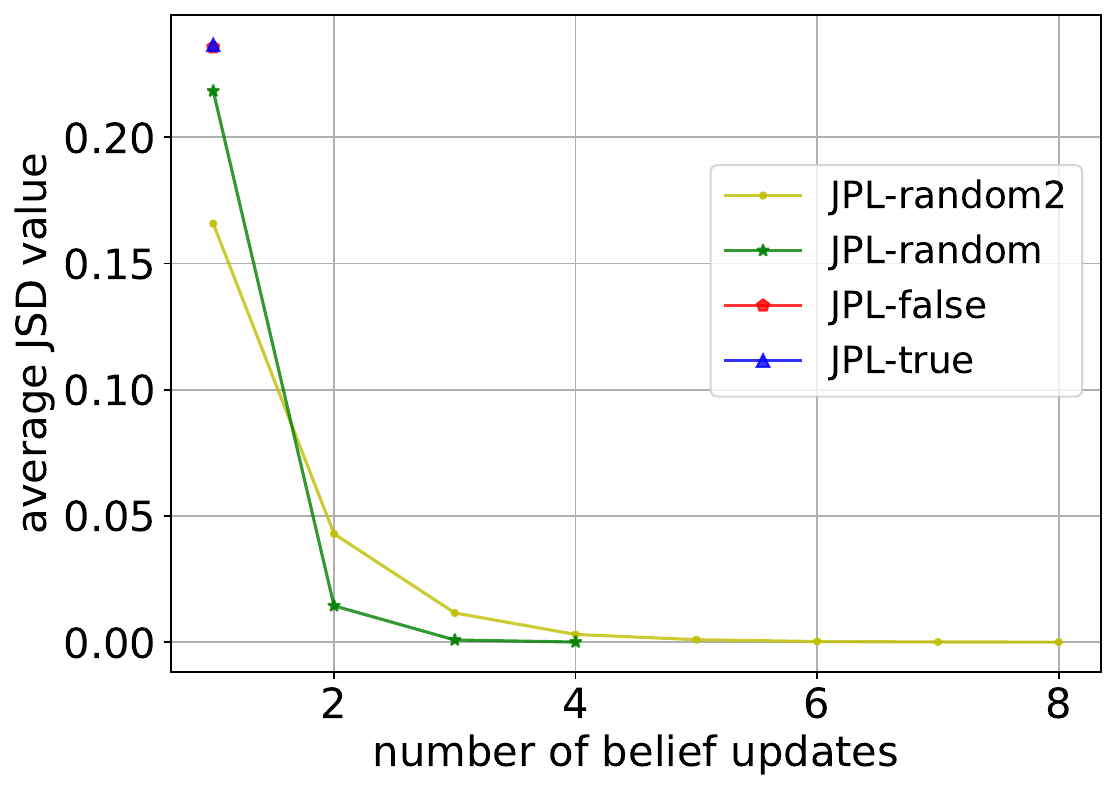}  
  \caption{\textcolor{blue}{
Average value of the JSD that led to the $k^{th}$ update.}}
  \label{fig:craft_belief_jsd}
\end{subfigure}
\caption{\textcolor{blue}{(a): Attained rewards of 10 independent runs for every 100 training steps for the craft-world task using the first random observation model: The darker region is bounded by $25^{th}$ and $75^{th}$ percentiles, while the lighter region is bounded by $10^{th}$ and $90^{th}$. The solid line represents the median.
(b)-(d): Progression of belief update during training for the office-world task.}}
\label{fig:rewards and belief craft world}
\end{figure}

\begin{table}[h!]
\caption{Convergence results of the proposed algorithm under different observation models in 10 independent runs in the craft-world environment.}
\centering
\begin{tabular}{r c c c c c} 
 \hline 
 Setting & JPL-true & JPL-false & JPL-random & JPL-random2 \\ 
 \hline\hline
  Q1 & 1,937,600 & 1,907,600 & 1,950,800 & 1,985,800  \\ 
 \hline
  Q2 & 1,435,600 & 1,242,800 & 1,531,200 & 1,825,600 \\
 \hline
  Q3 & 1,196,400 & 1,129,600 & 1,256,000 & 1,484,400 \\
 \hline
 BU & 1 & 1 & 3.6 & 6.83\\
 \hline
 \label{table:comparison_craft}
\end{tabular}
\end{table}


\section{Conclusion and Discussion}

We develop a reinforcement algorithm subject to reward-machine-encoded tasks and perceptual limitations. The algorithm holds probabilistic beliefs over the truth value of the atomic propositions and uses a hypothesis reward machine to estimate the reward machine that encodes the task to be accomplished. Both the beliefs and the hypothesis reward machine are updated by using the agent's sensor measurements and the obtained rewards from the agent's trajectories. The algorithm uses the aforementioned beliefs and the hypothesis reward machine in a q-learning procedure to 
learn an optimal policy that accomplishes this task.  

Currently, the theoretical guarantees, provided by Theorem \ref{th:main}, presuppose that the algorithm 
obtains a good enough estimate of 
 the atomic propositions' value, as encoded in the ground truth labelling function $L_G$, i.e., $\hat{L}(\hat{\mathcal{L}}_h,\cdot) = L_G(\cdot)$.  The definition of $\hat{L}(\mathcal{\hat{L}}_h,\cdot) $ in \eqref{eq:L_hat}, based on the most probable atomic propositions $p$ for state $s$, i.e., $\hat{\mathcal{L}}_h(s,p) \geq 0.5$, mildens such an assumption, which is further verified to hold by the experimental results. \textcolor{blue}{ Additionally, Sec. \ref{sec:L estimate} illustrates that unambiguous observation models, i.e., when then probability of observing an atomic proposition when it is true is different than then probability of observing it when it is false, lead to accurate estimation of the ground-truth labelling function when the obtained observations are consistent with these observation models. Such an estimation is achieved even by the time-varying observation models used in Sec. \ref{sec:exps}, illustrating the robustness of the proposed approach to the choice of observation model.  
Along the same lines, the experimental results show that the proposed algorithm, used with a large variety of unambiguous observation models, achieves estimation of the ground-truth labelling function in much fewer steps than the ones needed for inferring the reward machine and learning an optimal policy. Therefore, we conclude that the selection of observation model does not practically 
affect the convergence of the algorithm to the optimal policy.  
Larger sets of atomic propositions would naturally lead to more complex reward machines and a proportional increase in convergence is expected both in the labelling-function estimation as well as the reward-machine inference and learning of the optimal policy.}

\textcolor{blue}{
We further note that larger reward machines, possibly describing more complex tasks yield higher computational times. In our implementation, inference of reward machines with more than 7 states resulted in unreasonably high computation times. As mentioned in Sec. \ref{sec:inference}, modification of the inference algorithm would be able to alleviate such issues. We argue, however, that reward machines up to 7 states can describe significantly complicated tasks,  such as Minecraft  \cite{hasanbeig2021deepsynth} or Atari tasks \cite{jin2022creativity}.
 In the future, we will consider uncertain observation models that are not consistent with the obtained observations 
 and cases where the ground-truth values of the atomic propositions change with time.}

\section{Acknowledgements}

This work was supported in part by ONR N00014-22-1-2254 and NSF 1652113.

\bibliography{mybibfile}

\begin{thebibliography}{10}
\expandafter\ifx\csname url\endcsname\relax
  \def\url#1{\texttt{#1}}\fi
\expandafter\ifx\csname urlprefix\endcsname\relax\def\urlprefix{URL }\fi
\expandafter\ifx\csname href\endcsname\relax
  \def\href#1#2{#2} \def\path#1{#1}\fi

\bibitem{taylor2007cross}
M.~E. Taylor, P.~Stone, Cross-domain transfer for reinforcement learning,
  Proceedings of the 24th international conference on Machine learning (2007)
  879--886.

\bibitem{singh1992reinforcement}
S.~P. Singh, Reinforcement learning with a hierarchy of abstract models, in:
  Proceedings of the National Conference on Artificial Intelligence, no.~10,
  Citeseer, 1992, p. 202.

\bibitem{kulkarni2016hierarchical}
T.~D. Kulkarni, K.~Narasimhan, A.~Saeedi, J.~Tenenbaum, Hierarchical deep
  reinforcement learning: Integrating temporal abstraction and intrinsic
  motivation, Advances in neural information processing systems 29 (2016)
  3675--3683.

\bibitem{abel2018state}
D.~Abel, D.~Arumugam, L.~Lehnert, M.~Littman, State abstractions for lifelong
  reinforcement learning, International Conference on Machine Learning (2018)
  10--19.

\bibitem{li2017reinforcement}
X.~Li, C.-I. Vasile, C.~Belta, Reinforcement learning with temporal logic
  rewards, IEEE/RSJ International Conference on Intelligent Robots and Systems
  (IROS) (2017) 3834--3839.

\bibitem{aksaray2016q}
D.~Aksaray, A.~Jones, Z.~Kong, M.~Schwager, C.~Belta, Q-learning for robust
  satisfaction of signal temporal logic specifications, IEEE 55th Conference on
  Decision and Control (CDC) (2016) 6565--6570.

\bibitem{icarte2018using}
R.~T. Icarte, T.~Klassen, R.~Valenzano, S.~McIlraith, Using reward machines for
  high-level task specification and decomposition in reinforcement learning,
  International Conference on Machine Learning (2018) 2107--2116.

\bibitem{xu2020joint}
Z.~Xu, I.~Gavran, Y.~Ahmad, R.~Majumdar, D.~Neider, U.~Topcu, B.~Wu, Joint
  inference of reward machines and policies for reinforcement learning,
  Proceedings of the International Conference on Automated Planning and
  Scheduling 30 (2020) 590--598.

\bibitem{konidaris2019necessity}
G.~Konidaris, On the necessity of abstraction, Current opinion in behavioral
  sciences 29 (2019) 1--7.

\bibitem{andreas2017modular}
J.~Andreas, D.~Klein, S.~Levine, Modular multitask reinforcement learning with
  policy sketches, International Conference on Machine Learning (2017)
  166--175.

\bibitem{wen2015correct}
M.~Wen, R.~Ehlers, U.~Topcu, Correct-by-synthesis reinforcement learning with
  temporal logic constraints, IEEE/RSJ International Conference on Intelligent
  Robots and Systems (IROS) (2015) 4983--4990.

\bibitem{sadigh2014learning}
D.~Sadigh, E.~S. Kim, S.~Coogan, S.~S. Sastry, S.~A. Seshia, A learning based
  approach to control synthesis of markov decision processes for linear
  temporal logic specifications, IEEE Conference on Decision and Control (2014)
  1091--1096.

\bibitem{hasanbeig2018logically}
M.~Hasanbeig, A.~Abate, D.~Kroening, Logically-constrained reinforcement
  learning, arXiv preprint arXiv:1801.08099.

\bibitem{toro2018teaching}
R.~Toro~Icarte, T.~Q. Klassen, R.~Valenzano, S.~A. McIlraith, Teaching multiple
  tasks to an rl agent using ltl, Proceedings of the 17th International
  Conference on Autonomous Agents and MultiAgent Systems (2018) 452--461.

\bibitem{wang2021reinforcement}
Y.~Wang, A.~K. Bozkurt, M.~Pajic, Reinforcement learning with temporal logic
  constraints for partially-observable markov decision processes, arXiv
  preprint arXiv:2104.01612.

\bibitem{bozkurt2020control}
A.~K. Bozkurt, Y.~Wang, M.~M. Zavlanos, M.~Pajic, Control synthesis from linear
  temporal logic specifications using model-free reinforcement learning, IEEE
  International Conference on Robotics and Automation (ICRA) (2020)
  10349--10355.

\bibitem{hahn2019omega}
E.~M. Hahn, M.~Perez, S.~Schewe, F.~Somenzi, A.~Trivedi, D.~Wojtczak,
  Omega-regular objectives in model-free reinforcement learning, International
  Conference on Tools and Algorithms for the Construction and Analysis of
  Systems (2019) 395--412.

\bibitem{carr2020verifiable}
S.~Carr, N.~Jansen, U.~Topcu, Verifiable rnn-based policies for pomdps under
  temporal logic constraints, arXiv preprint arXiv:2002.05615.

\bibitem{ghasemi2020task}
M.~Ghasemi, E.~Bulgur, U.~Topcu, Task-oriented active perception and planning
  in environments with partially known semantics, International Conference on
  Machine Learning (2020) 3484--3493.

\bibitem{da2019active}
R.~R. da~Silva, V.~Kurtz, H.~Lin, Active perception and control from temporal
  logic specifications, IEEE Control Systems Letters 3~(4) (2019) 1068--1073.

\bibitem{agha2014firm}
A.-A. Agha-Mohammadi, S.~Chakravorty, N.~M. Amato, Firm: Sampling-based
  feedback motion-planning under motion uncertainty and imperfect measurements,
  The International Journal of Robotics Research 33~(2) (2014) 268--304.

\bibitem{guo2013revising}
M.~Guo, K.~H. Johansson, D.~V. Dimarogonas, Revising motion planning under
  linear temporal logic specifications in partially known workspaces, 2013 IEEE
  International Conference on Robotics and Automation (2013) 5025--5032.

\bibitem{lahijanian2016iterative}
M.~Lahijanian, M.~R. Maly, D.~Fried, L.~E. Kavraki, H.~Kress-Gazit, M.~Y.
  Vardi, Iterative temporal planning in uncertain environments with partial
  satisfaction guarantees, IEEE Transactions on Robotics 32~(3) (2016)
  583--599.

\bibitem{livingston2012backtracking}
S.~C. Livingston, R.~M. Murray, J.~W. Burdick, Backtracking temporal logic
  synthesis for uncertain environments, IEEE International Conference on
  Robotics and Automation (2012) 5163--5170.

\bibitem{montana2017sampling}
F.~J. Montana, J.~Liu, T.~J. Dodd, Sampling-based reactive motion planning with
  temporal logic constraints and imperfect state information, Critical Systems:
  Formal Methods and Automated Verification (2017) 134--149.

\bibitem{ayala2013temporal}
A.~M. Ayala, S.~B. Andersson, C.~Belta, Temporal logic motion planning in
  unknown environments, IEEE/RSJ International Conference on Intelligent Robots
  and Systems (2013) 5279--5284.

\bibitem{camacho2017non}
A.~Camacho, O.~Chen, S.~Sanner, S.~A. McIlraith, Non-markovian rewards
  expressed in ltl: guiding search via reward shaping, Tenth Annual Symposium
  on Combinatorial Search.

\bibitem{ramasubramanian2019secure}
B.~Ramasubramanian, A.~Clark, L.~Bushnell, R.~Poovendran, Secure control under
  partial observability with temporal logic constraints, American Control
  Conference (ACC) (2019) 1181--1188.

\bibitem{zhang2015learning}
X.~Zhang, B.~Wu, H.~Lin, Learning based supervisor synthesis of pomdp for pctl
  specifications, 54th IEEE Conference on Decision and Control (CDC) (2015)
  7470--7475.

\bibitem{furelos2020induction}
D.~Furelos-Blanco, M.~Law, A.~Russo, K.~Broda, A.~Jonsson, Induction of subgoal
  automata for reinforcement learning, in: Proceedings of the AAAI Conference
  on Artificial Intelligence, Vol.~34, 2020, pp. 3890--3897.

\bibitem{hasanbeig2021deepsynth}
M.~Hasanbeig, N.~Y. Jeppu, A.~Abate, T.~Melham, D.~Kroening, Deepsynth:
  Automata synthesis for automatic task segmentation in deep reinforcement
  learning, in: Proceedings of the AAAI Conference on Artificial Intelligence,
  Vol.~35, 2021, pp. 7647--7656.

\bibitem{abate2022learning}
A.~Abate, Y.~Almulla, J.~Fox, D.~Hyland, M.~Wooldridge, Learning task automata
  for reinforcement learning using hidden markov models, arXiv preprint
  arXiv:2208.11838.

\bibitem{toro2019learning}
R.~Toro~Icarte, E.~Waldie, T.~Klassen, R.~Valenzano, M.~Castro, S.~McIlraith,
  Learning reward machines for partially observable reinforcement learning,
  Advances in Neural Information Processing Systems 32 (2019) 15523--15534.

\bibitem{rens2020online}
G.~Rens, J.-F. Raskin, R.~Reynouad, G.~Marra, Online learning of non-markovian
  reward models, arXiv preprint arXiv:2009.12600.

\bibitem{dohmen2022inferring}
T.~Dohmen, N.~Topper, G.~Atia, A.~Beckus, A.~Trivedi, A.~Velasquez, Inferring
  probabilistic reward machines from non-markovian reward signals for
  reinforcement learning, in: Proceedings of the International Conference on
  Automated Planning and Scheduling, Vol.~32, 2022, pp. 574--582.

\bibitem{bellman1957markovian}
R.~Bellman, A markovian decision process, Journal of mathematics and mechanics
  (1957) 679--684.

\bibitem{puterman2014markov}
M.~L. Puterman, Markov decision processes: discrete stochastic dynamic
  programming, John Wiley \& Sons, 2014.

\bibitem{gold1978complexity}
E.~M. Gold, Complexity of automaton identification from given data, Information
  and control 37~(3) (1978) 302--320.

\bibitem{neider2013regular}
D.~Neider, N.~Jansen, Regular model checking using solver technologies and
  automata learning, NASA Formal Methods Symposium (2013) 16--31.

\bibitem{heule2010exact}
M.~J. Heule, S.~Verwer, Exact dfa identification using sat solvers,
  International Colloquium on Grammatical Inference (2010) 66--79.

\bibitem{neider2014applications}
D.~Neider, Applications of automata learning in verification and synthesis,
  Ph.D. thesis, Hochschulbibliothek der Rheinisch-Westf{\"a}lischen Technischen
  Hochschule Aachen (2014).

\bibitem{jin2022creativity}
M.~Jin, Z.~Ma, K.~Jin, H.~H. Zhuo, C.~Chen, C.~Yu, Creativity of ai: Automatic
  symbolic option discovery for facilitating deep reinforcement learning, in:
  Proceedings of the AAAI Conference on Artificial Intelligence, Vol.~36, 2022,
  pp. 7042--7050.

\bibitem{li2019houseexpo}
L.~Tingguang, H.~Danny, L.~Chenming, Z.~Delong, W.~Chaoqun, M.~Q.-H. Meng,
  Houseexpo: A large-scale 2d indoor layout dataset for learning-based
  algorithms on mobile robots, arXiv preprint arXiv:1903.09845.

\bibitem{morgado2014maxsat}
A.~Morgado, C.~Dodaro, J.~Marques-Silva, Core-guided maxsat with soft
  cardinality constraints (2014) 564--573.

\bibitem{imms-sat18}
A.~Ignatiev, A.~Morgado, J.~Marques{-}Silva,
  \href{https://doi.org/10.1007/978-3-319-94144-8_26}{{PySAT:} {A} {Python}
  toolkit for prototyping with {SAT} oracles} (2018) 428--437\href
  {http://dx.doi.org/10.1007/978-3-319-94144-8_26}
  {\path{doi:10.1007/978-3-319-94144-8_26}}.
\newline\urlprefix\url{https://doi.org/10.1007/978-3-319-94144-8_26}

\bibitem{wen2018flipout}
Y.~Wen, P.~Vicol, J.~Ba, D.~Tran, R.~Grosse,
  \href{https://openreview.net/forum?id=rJNpifWAb}{Flipout: Efficient
  pseudo-independent weight perturbations on mini-batches}.
\newline\urlprefix\url{https://openreview.net/forum?id=rJNpifWAb}

\bibitem{watkins1992q}
C.~J. Watkins, P.~Dayan, Q-learning, Machine learning 8~(3-4) (1992) 279--292.

\end{thebibliography}

\appendix

\textcolor{blue}{
\section{} \label{sec:appendix}
}

We provide here the proof of Theorem \ref{th:main}.
We first need some necessary concepts, starting with the \textit{attainable trajectory}:

\begin{definition} \label{def:m-attainable traj}
Let $\mathcal{M} = (S, s_I,A, \mathcal{T}, R, \mathcal{AP}, \hat{L}, \gamma)$ be an \textcolor{blue}{l-MDP} and $m \in \mathbb{N}$ a natural number. We call a trajectory
$\zeta = s_0 a_0 s_1 \dots s_k a_k s_{k+1} \in (S\times A)^\ast \times S$ $m$-attainable if (i) $k\leq m$ and (ii) $\mathcal{T}(s_i,  a_i, s_{i+1}) > 0$ for each $i \in \{0, \dots, k\}$.
Moreover, we say that a trajectory $\zeta$ is attainable if there exists an $m \in \mathbb{N}$ such that $\zeta$ is $m$-attainable.

\end{definition}

Since the function $\mathsf{GetEpsilonGreedyAction}()$ (line 3 of Algorithm \ref{algo:QRM_mod}) follows an $\varepsilon_a$-greedy policy, 
we can show that the proposed algorithm almost surely explores every attainable trajectory in the limit, i.e., with probability
1 as the number of episodes goes to infinity.

\begin{lemma}\cite[Lemma 1]{xu2020joint} \label{lem:m attainable traj}
Let $m \in \mathbb{N}$ be a natural number. Then, Algorithm \ref{algo:JIRP ext}, with $eplength \geq m$,  almost surely explores every $m$-attainable
trajectory at least once in the limit.
\end{lemma} 


Similar to  Def. \ref{def:m-attainable traj}, we call an observed label sequence  $\hat{\lambda} = \hat{\ell}_0,\dots,\hat{\ell}_k$ ($m$-)attainable on $\hat{\mathcal{L}}_h$ if there exists an ($m$-)attainable
trajectory $s_0 a_0 s_1 \dots s_k a_k s_{k+1}$ such that $\hat{\ell}_j = \hat{L}( \hat{\mathcal{L}}_h, s)$  for each $j \in \{0, \dots, k\}$. 

Consider now Algorithm \ref{algo:JIRP ext} and assume that the condition of $\mathsf{SignifChange}(\hat{\mathcal{L}}_h, \hat{\mathcal{L}}_j)$ (line 13) is not satisfied after a certain number of episodes, i.e., $\hat{\mathcal{L}}_h$ remains fixed. Then, Lemma \ref{lem:m attainable traj} implies that Algorithm \ref{algo:JIRP ext} almost surely explores every $m$-attainable label sequences on $\hat{\mathcal{L}}_h$ in the limit, formalized in the following corollary.

\begin{corollary}
Assume that there exists \textcolor{blue}{an} $n_r > 0$ such that $\mathsf{SignifChange}(\hat{\mathcal{L}}_h, \hat{\mathcal{L}}_j) = \mathsf{False}$  (line 13 of Algorithm \ref{algo:JIRP ext}), for all episodes $n > n_r$. Then Algorithm \ref{algo:JIRP ext}, with $eplength \geq m$, explores almost surely every $m$-attainable label sequence on $\hat{\mathcal{L}}_h$ in the limit.
\end{corollary}

Therefore, if Algorithm \ref{algo:JIRP ext} explores sufficiently many $m$-attainable label sequences on some distribution $\hat{\mathcal{L}}_h$ and for a large enough value of $m$, it is guaranteed to infer a reward machine that is ``good enough" in the sense that it is equivalent to the reward machine encoding the
reward function $R$ on $\hat{\mathcal{L}}_h$ and on all attainable label sequences on $\hat{\mathcal{L}}_h$, assuming that such a reward machine exists (see Example \ref{ex:office}). This is formalized in the next lemma.

\begin{lemma}\cite[Lemma 2]{xu2020joint} \label{lem:RM learn}
Let $\mathcal{M} = (S, s_I,A, \mathcal{T}, R, \mathcal{AP}, \hat{L}, \gamma)$ be an \textcolor{blue}{l-MDP} and assume that there exists a reward machine $\mathcal{A}_h$ that encodes the reward function $R$ on $\hat{L}(\hat{\mathcal{L}}_h,\cdot)$ for some belief $\hat{\mathcal{L}}_h$.  Assume that there exists \textcolor{blue}{an} $n_r > 0$ such that $\mathsf{SignifChange}(\hat{\mathcal{L}}_h, \hat{\mathcal{L}}_j)$ $=$ $\mathsf{False}$  (line 13 of Alg. \ref{algo:JIRP ext}), for all episodes $n>n_r$. Then, Algorithm \ref{algo:JIRP ext}, with $\mathsf{eplength} \geq 2^{| \mathcal{M} |+1}(|\mathcal{A}_h| + 1)-1$, almost surely learns a reward machine in the limit that is equivalent to $\mathcal{A}_h$ on all attainable label sequences on $\hat{\mathcal{L}}_h$. 
\end{lemma}


Following Lemma \ref{lem:RM learn}, Algorithm \ref{algo:JIRP ext} will eventually learn the reward machine encoding the reward function on $\hat{\mathcal{L}}_h$, when $\hat{\mathcal{L}}_h$ is fixed.
Intuitively, Lemma \ref{lem:RM learn} suggests that, eventually, the algorithm is able to learn a reward machine that encodes the reward function on a \textit{fixed} estimate $\hat{L}(\hat{\mathcal{L}}_h,\cdot)$, i.e., without the updates from the new observations.  
Nevertheless, the potential perception updates (line 14 of Algorithm \ref{algo:QRM_mod}) might prevent the algorithm \textcolor{blue}{from learning a reward machine}, since $\hat{\mathcal{L}}_h$ might be changing after a finite number of episodes. However, if the observation model is accurate (or inaccurate) enough and the agent explores sufficiently many label sequences, intuition suggests that $\hat{\mathcal{L}}_j$ will be constantly improving, leading to less frequent updates with respect to the divergence test (line 13 of Algorithm \ref{algo:JIRP ext}). Therefore, there exists a finite number of episodes, after which $\hat{\mathcal{L}}_h$ will be fixed to some belief $\hat{\mathcal{L}}_f$ and $\hat{L}(\hat{\mathcal{L}}_f,\cdot)$ will be ``close'' to the ground-truth function $L_G$. By applying then Lemma \ref{lem:RM learn}, we conclude that Algorithm \ref{algo:JIRP ext} will learn the reward machine that encodes the reward function on $\hat{\mathcal{L}}_f$, and consequently converge to the $q$-function that defines an optimal policy. This is used in the proof of Theorem \ref{th:main} that follows.

\begin{proof}[Proof of Theorem \ref{th:main}]
Note first that the $\varepsilon_a$-greedy action policy, imposed by the function  $\mathsf{GetEpsilonGreedyAction}()$, and $eplength \geq |\mathcal{M}|$ imply that every action-pair of $\mathcal{M}$ will be visited infinitely often. 
Moreover, $\mathsf{SignifChange}(\hat{\mathcal{L}}_h, \hat{\mathcal{L}}_j)$ will be always false for all episodes $n  > n_f$, since $\hat{\mathcal{L}}_h$ will be an accurate enough estimate.  By definition, the reward machine $\mathcal{A}$ that encodes the reward function on the ground truth encodes  the reward function on $\hat{\mathcal{L}}_h$ as well.  Therefore, Lemma \ref{lem:RM learn} guarantees that Algorithm \ref{algo:JIRP ext} eventually learns a reward machine $\mathcal{H}$ that is equivalent to $\mathcal{A}$ on all attainable label sequences.   

According to Observation 1 of \cite{icarte2018using}, given the \textcolor{blue}{l-MDP} $\mathcal{M} = (S, s_I,A, \mathcal{T}, R, \mathcal{AP}, \hat{L}, \gamma)$ and the reward machine $\mathcal{A} = (V, v_I, 2^\mathcal{AP}, \mathcal{R}, \delta, \sigma)$, one can construct an \textcolor{blue}{l-MDP} $\mathcal{M}_\mathcal{H} = ( S\times V, (s_I,v_I), A, \mathcal{T}', R', \mathcal{AP}, \hat{L}, \gamma)$ with 
\begin{align*}
&\mathcal{T}'((s,v),a,(s',v')) = \left\{ 
\begin{matrix} 
\mathcal{T}(s,a,s'), & \text{if } v' = \delta(v,L(s')) \\
0 & \text{otherwise}
\end{matrix} 
\right. \\ 
&R'((s,v),a,(s',v')) = \sigma(v, L(s'))
\end{align*}
and a 
Markovian reward function such that every attainable label sequence of $\mathcal{M}_\mathcal{H}$ receives the same reward as in $\mathcal{M}$. Thus, an optimal policy for $\mathcal{M}_\mathcal{H}$ will be also optimal for $\mathcal{M}$. Due to the $\varepsilon_a$-greedy action policy, the episode length being $eplength \geq |\mathcal{M}|$, and the fact that the updates are done in parallel for all states of the reward machine $\mathcal{H}$,
every state-action pair of the \textcolor{blue}{l-MDP} $\mathcal{M}_H$ will be visited infinitely often. Therefore, convergence of q-learning for $\mathcal{M}_\mathcal{H}$ is guaranteed \cite{watkins1992q}. Finally, since an optimal policy for $\mathcal{M}_\mathcal{H}$ is optimal for $\mathcal{M}$, Algorithm \ref{algo:JIRP ext} converges to an optimal policy too. 

\end{proof}

	\end{document}